\documentclass[journal,12p]{IEEEtran}
\usepackage{makecell}
\usepackage[colorlinks,linkcolor=blue]{hyperref}
\usepackage{threeparttable}
\IEEEoverridecommandlockouts

\usepackage{graphicx,inputenc,amsmath,amssymb,amsfonts,color,amsthm,subfigure,enumerate,framed,booktabs,multirow,paralist}
\usepackage{tablefootnote}
\usepackage{nomencl}
\makenomenclature
\usepackage[inline]{enumitem}
\usepackage{epstopdf}
\usepackage{epsfig}
\usepackage[ruled,vlined]{algorithm2e}
\usepackage{bm}
\newcommand{\ad}{\mathrm{ad}}
\newcommand{\Ad}{\mathrm{Ad}}
\newcommand{\cons}{\operatorname{cons}}
\newcommand{\GN}{\operatorname{GN}}
\newcommand{\Gf}{\operatorname{Gf}}
\newcommand{\MLE}{\operatorname{MLE}}

\newcommand{\MAP}{\operatorname{MAP}}
\newcommand{\op}{\operatorname{op}}
\newcommand{\EIKF}{\operatorname{EIKF}}

\newtheorem{problem}{Problem}
\renewcommand{\theproblem}{\arabic{problem}} % Number format: just the theorem number
\newtheorem{theorem}{Theorem}
\newtheorem{lemma}{Lemma}
\newtheorem{proposition}{Proposition}

\newtheorem{definition}{Definition}

\newtheorem{remark}{Remark}

\renewcommand{\d}[1]{\ensuremath{\operatorname{d}\!{#1}}}
\newcommand{\dexp}{{\mathrm{dexp}}}
\newcommand{\bs}[1]{\boldsymbol{#1}}

\newcommand{\rvline}{\hspace*{-\arraycolsep}\vline\hspace*{-\arraycolsep}}

\begin{document}
	\title{Efficient Invariant Kalman Filter for Inertial-based Odometry with Large-sample Environmental Measurements}
	\author{Xinghan Li$^{\rm a,b}$, Haoying Li$^{\rm b}$, Guangyang Zeng$^{\rm b}$, Qingcheng Zeng$^{\rm c,b}$, Xiaoqiang Ren$^{\rm d}$, Chao Yang$^{\rm e}$, Junfeng Wu$^{\rm b}$
		 \thanks{ 
 $^{\rm a}$:~College of Control Science and Engineering, Zhejiang University, Hangzhou, P. R. China.
 $^{\rm b}$:~School of Data Science,  The Chinese University of Hong Kong, Shenzhen, Shenzhen, P. R. China.
 $^{\rm c}$:~The Hong Kong University of Science and Technology (Guangzhou), Guangzhou, P. R. China.  
 $^{\rm d}$: School of Mechatronic Engineering and Automation, Shanghai University, Shanghai, P. R. China. 
 $^{\rm e}$: Department of Automation, East China University of Science and Technology, Shanghai, P. R. China.
 Emails: xinghanli@zju.edu.cn(X. Li), haoyingli@link.cuhk.edu.cn(H. Li), 
zengguangyang@cuhk.edu.cn (G. Zeng), 
qzeng450@connect.hkust-gz.edu.cn(Q. Zeng)
xqren@shu.edu.cn (X. Ren), yangchao@ecust.edu.cn (C. Yang), junfengwu@cuhk.edu.cn (J. Wu).
} 			
			% 	This work was supported in part by	the university development fund of the Chinese University of Hong Kong, Shenzhen under grant No. 01002232.
   }
	
	\date{September 2022}

	\maketitle	
	\begin{abstract}
		A filter for inertial-based odometry is a recursive method used to estimate the pose from measurements of ego-motion and relative pose. Currently, there is no known filter that guarantees the computation of a globally optimal solution for the non-linear measurement model.  In this paper, we demonstrate that an innovative filter, with the state being $SE_2(3)$ and the $\sqrt{n}$-\textit{consistent} pose as the initialization, efficiently achieves \textit{asymptotic optimality} in terms of minimum mean square error. This approach is tailored for real-time SLAM and inertial-based odometry applications.
		Our first contribution is that we propose an iterative filtering method based on the Gauss-Newton method on Lie groups which is numerically to solve the estimation of states from a priori and non-linear measurements. The filtering stands out due to its iterative mechanism and adaptive initialization. Second, when dealing with environmental measurements of the surroundings, we utilize a $\sqrt{n}$-consistent pose as the initial value for the update step in a single iteration. The solution is closed in form and has computational complexity $O(n)$. Third, we theoretically show that the approach can achieve asymptotic optimality in the sense of minimum mean square error from the a priori and virtual relative pose measurements (see Problem~\ref{prob:new update problem}). Finally, to validate our method, we carry out extensive numerical and experimental evaluations. Our results consistently demonstrate that our approach outperforms other state-of-the-art filter-based methods, including the iterated extended Kalman filter and the invariant extended Kalman filter, in terms of accuracy and running time.  
	\end{abstract}	
 
 \section{Introduction}
	\subsection{Background and Contributions}
	Simultaneous localization and mapping (SLAM) is essential for estimating a robot's position while mapping its surroundings. Odometries, like LiDAR inertial odometry (LIO) and visual inertial odometry (VIO), provide precise robot pose information. With the increasing availability of sensor data, the demand of accurate and efficient odometry algorithms has grown. These algorithms find applications in various domains, including unmanned aerial vehicles~\cite{nguyen2021viral, zhao2021super}, unmanned ground vehicles~\cite{Mourikis2007AMC,xu2022fast}, and legged robots~\cite{wisth2022vilens, wisth2021unified}. This paper introduces an innovative filter-based method custom-tailored for VIO and LIO applications, with a particular emphasis on achieving superior performance in scenarios involving large-samples environmental measurements~(LEM).
	
	% The platform of SLAM has witnessed significant advancements, enabling real-time state estimation and mapping using single perceptual sensors like LiDAR~\cite{Zhang2014LOAMLO} or camera~\cite{MurArtal2015ORBSLAMAV}. LiDAR-based methods are renowned for their ability to capture fine details over long distances.  Vision-based methods, on the other hand, excel in texture-rich environments but can be susceptible to issues like changes in illumination, rapid motion, and unknown scale~\cite{Qin2017VINSMonoAR, Qin2019LINSAL}. By incorporating data from LiDAR, camera, and IMU sensors, we aim to implement the LVI-navigation platform which leverages their respective strengths to overcome degeneracy of the single sensor~\cite{Zhang2016OnDO} and improve the robustness. 
	
	\begin{figure*}[htbp]
		\centering
		\subfigure[]{
			\includegraphics[width=0.45\textwidth]{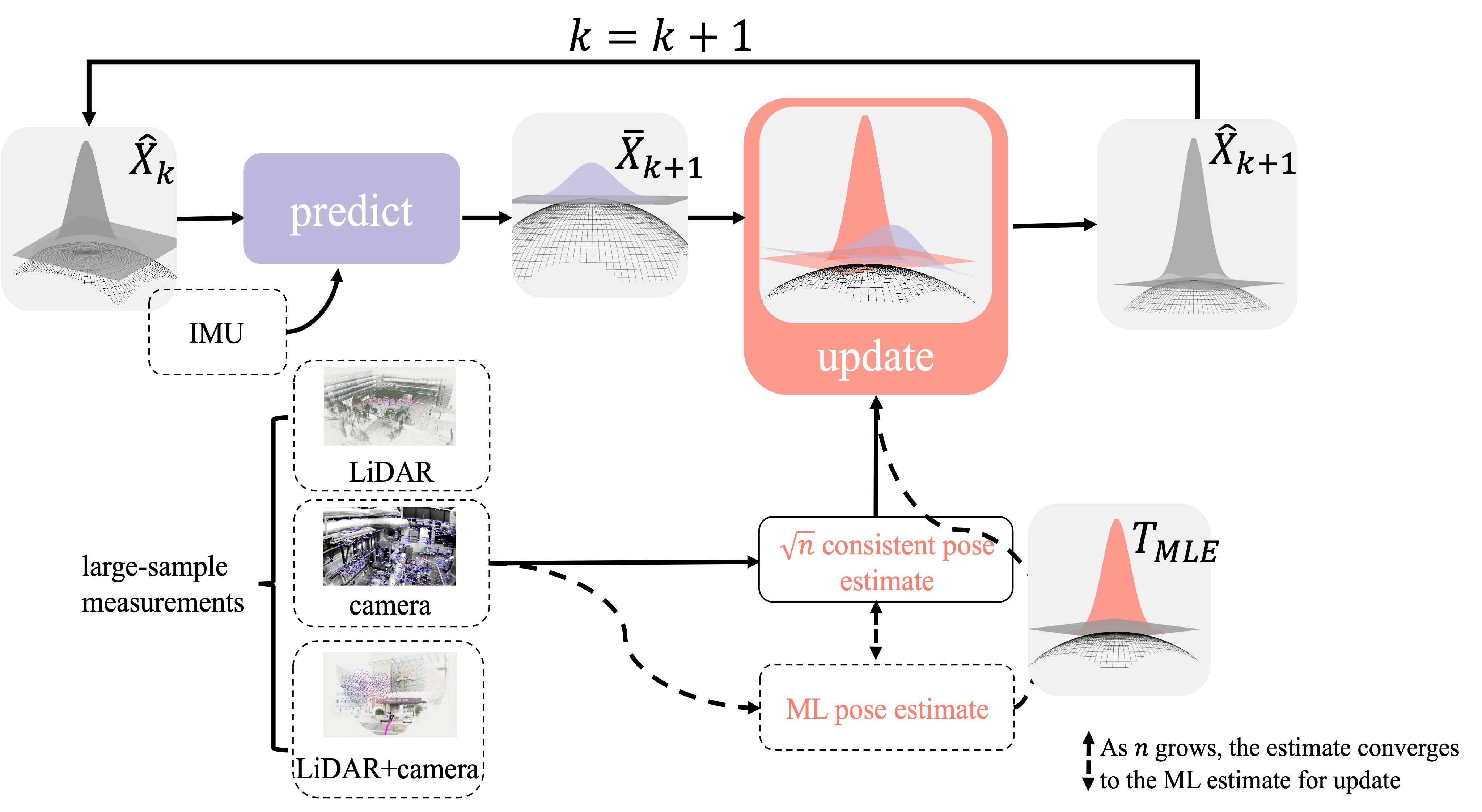}\label{fig:architecture}
		}
           \subfigure[]{
			\includegraphics[width=0.45\textwidth]{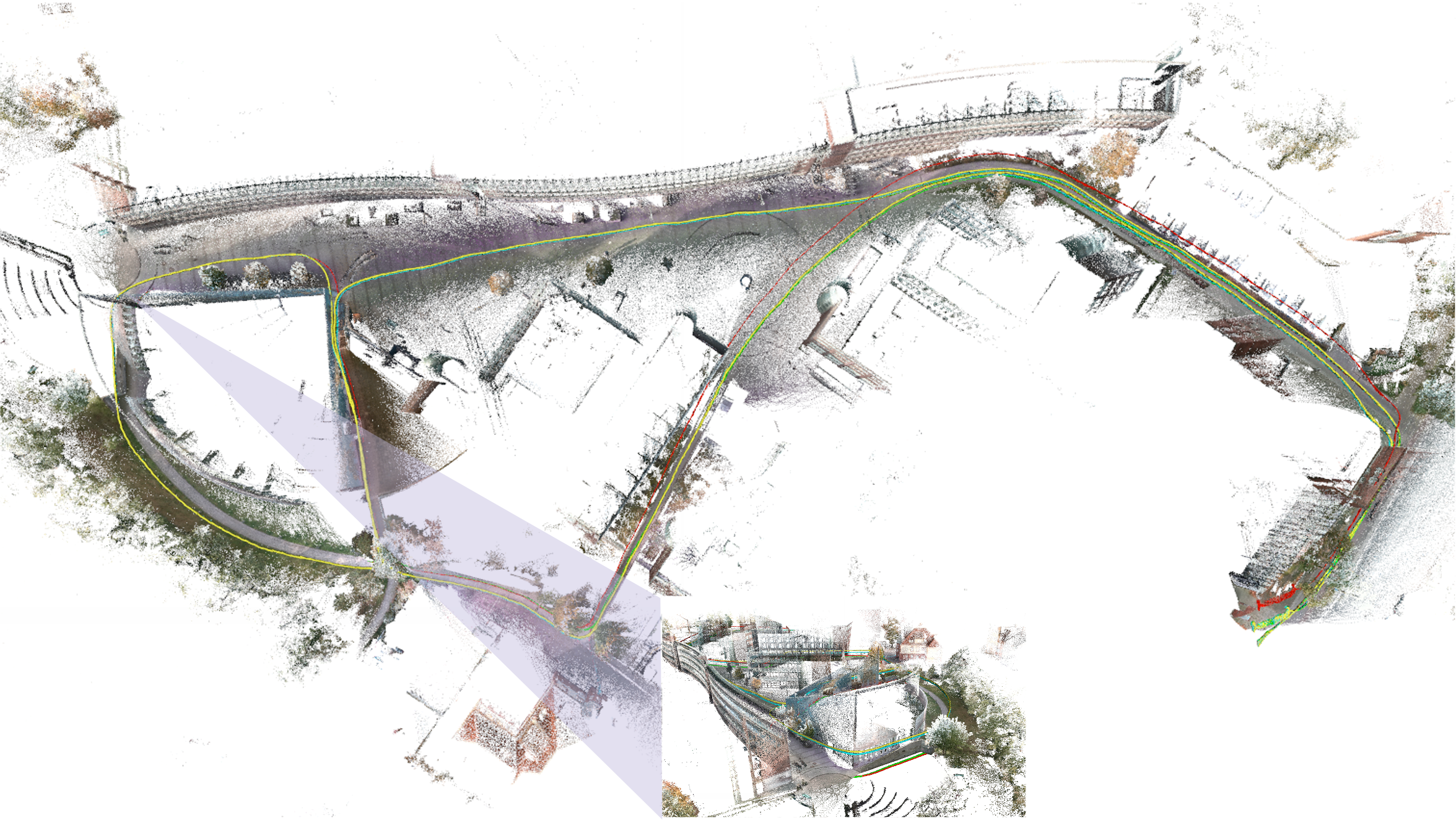}\label{fig:efficiency}
		}
\caption{Fig.~\ref{fig:architecture} depicts the architecture of efficient invariant Kalman filter (EIKF). In LEM scenarios, EIKF solves Problems~\ref{prob:predict problem} and~\ref{prob:new update problem}, which are 
approximate to exact predict and update (Problems~\ref{prob:update problem} and~\ref{prob:original predict problem}) for 
sequential Bayesian filtering. Fig.~\ref{fig:efficiency} is a snapshot of the LiDAR-Visual-Inertial odometry with EIKF in the zzz\_day\_02 sequence of the public dataset~\cite{mcdviral2023}. 
EIKF represented by the yellow line outperforms other filtering methods in terms of accuracy~(see specific performance metrics in Table~\ref{table:exp_LVIO}). }\label{fig:illustrate}
	\end{figure*}
	
	Filtering methods in odometry and SLAM operate recursively, involving the prediction and update steps, such as extended Kalman filter~(EKF)~\cite{Huang2008AnalysisAI}. Unlike optimization-based approaches, these methods excel at handling a large volume of real-time measurements with less computational cost, making them ideal for time-critical applications. To address the non-linearity of pose estimation in the update step, iterative techniques via filtering methods strike a balance between computational efficiency and estimation precision. For instance, methods like the iterated extended Kalman filter~(IEKF) utilize the Gauss-Newton~(GN) method in multiple iterations to enhance accuracy~\cite{Bell93}. 
	
	When filtering methods are considered to be used in SLAM, researchers have found that the EKF produces inconsistent estimates, the covariance calculated by the filter violating the true covariance of the estimator \cite{Huang2007ConvergenceAC}. To address this concern, researchers have explored observability analysis, leading to the introduction of the first estimation Jacobian~\cite{Huang2008AnalysisAI}. An alternative approach, the invariant extended Kalman filter (InEKF) proposed by Barrau et al.~\cite{Barrau2015AnEA}, has also been shown promising performance in terms of observability and local stability. In light of this, our work capitalizes on the benefits of the InEKF.
	
	In recent times, an abundance of data has become available from sensors at regular intervals, such as dense image features and point clouds. Taking advantage of the principles of asymptotic optimality as outlined in statistical theory by~\cite{Jennrich1969AsymptoticPO,zeng2022cpnp}, we can harness these vast samples to improve the accuracy of odometry estimatess. Asymptotic efficiency is used to describe how well an estimator can approach optimality in minimizing mean squared error as the environmental measurements of samples grow. Consequently, it is imperative to place significant emphasis on analyzing the asymptotic efficiency of estimators in the current technology and context. Our specific focus revolves around the adoption of the invariant error approach to obtain estimates that not only demonstrate asymptotic efficiency in the context of extensive data samples but also offer the flexibility to iterate multiple times for various scenarios.
	
There are limited studies investigating effective methods to exactly compute the probability of the current true state given the historical measurements
due to the non-linearity of the robotic kinematics and the measurement models. In this paper, we 
 turn to investigate 
approximate prediction and update, namely Problem \ref{prob:predict problem} and Problem \ref{prob:new update problem}, with the help of so-called virtual measurements constructed from the raw measurements. 
To solve the new problems, we introduce a new filtering approach that features an iterative mechanism on the Lie group with an initial consistent pose constructed from LEMs, as illustrated in Figure~\ref{fig:illustrate}. Its validity is supported by theoretical analysis and multi-sensor experimental evaluations. The main contributions of the paper are summarized as follows:
	\begin{enumerate}[label=(\roman*)]
		\item\label{contri:Iterated InEKF} We develop a GN method, equivalent to the update of iterated filtering, to solve the estimates of states from the a priori on the Lie group and relative pose measurements. This filtering method is characterized by a flexible number of iterations compared to InEKF and a more adaptive initialization compared to other conventional GN methods. The derivation is applied specifically in the context of VIO and LIO.
  
		\item We introduce a novel initialization of the above filtering acquired by large samples of surroundings, referred to as a \textit{$\sqrt{n}$-consistent} pose. We derive the $\sqrt{n}$-consistent solutions for perspective-n-points (PnP) in VIO and point-to-plane iterative closest point (ICP) in LIO. This $\sqrt{n}$-consistent pose is efficiently expressed in a closed form and demonstrates asymptotic convergence to the true pose. Additionally, one iteration is sufficient to provide an accurate estimate. These characterizations make it applicable in LEM scenarios.
		\item We theoretically show that our proposed filtering method initialized with a $\sqrt{n}$-consistent pose, namely the efficient invariant Kalman filter (EIKF), asymptotically approaches to the solution to Problems~\ref{prob:predict problem} and~\ref{prob:new update problem} in LEM scenarios. In other words,  our EIKF is asymptotically optimal in terms of minimum mean square error (MMSE) to Problems~\ref{prob:predict problem} and~\ref{prob:new update problem}.

		\item We conduct simulations and experimental demonstrations of both the VIO and LIO methods individually. Furthermore, we implement the filtering on a LiDAR-Visual-IMU platform. Through a series of simulations and dataset experiments, we evaluate the root mean square error (RMSE) and time cost of our method with other state-of-the-art filtering techniques. The results are evident, highlighting that our filter approach surpasses existing methods in terms of both the accuracy of the estimates and the complexity of time. The testing code can be downloaded from \href{https://github.com/LIAS-CUHKSZ/EIKF-VIO-LIO}{https://github.com/LIAS-CUHKSZ/EIKF-VIO-LIO}. 
		
		% The implementation for the multi-sensor platform in C++ code is available at \href{https://github.com/LIAS-CUHKSZ/ILIVE}{https://github.com/LIAS-CUHKSZ/ILIVE} and \href{https://github.com/LIAS-CUHKSZ/IEKF-OpenVINS}{https://github.com/LIAS-CUHKSZ/IEKF-OpenVINS}.
	\end{enumerate}

	\subsection{Links and differences with existing literature}
	In this subsection, we provide an overview of filter-based estimators for pose estimates and their applications, including VIO, LIO, and multi-sensors navigation system.
	\subsubsection*{Kalman filter}
	filtering methods are commonly employed to estimate the state of robots using a combination of control inputs and measurements up to the present time. The filter process typically consists of two steps: prediction and update. In situations where the process model and measurements exhibit linearity, and the noises follow the Gaussian distribution, Kalman filter~(KF) offers an optimal solution in terms of MMSE~\cite{Klmn1960ANA}.
	
	However, practical mobile robots are of non-linear system dynamics, and their measurement models may also be non-linear. In such cases, an EKF based on error states can be employed for pose estimates and SLAM~\cite{Bailey2006ConsistencyOT,Huang2010ObservabilitybasedRF}. The EKF linearizes the non-linear dynamics around the current state estimate, allowing for the use of a linear KF. It is important to note that while the EKF provides a locally stable observer, it does not preserve the optimality guaranteed by KF in general cases. Here, ``locally'' means that if the initial error is significant, the EKF diverges and yields unreliable estimates. Additionally, the EKF fails to maintain the observability space of the original system in SLAM, leading to an inconsistent\footnote{A state estimator for SLAM is consistent
		defined in~\cite{Huang2008AnalysisAI} if the estimation
		errors (i) are zero-mean, and (ii) have covariance matrix smaller or equal to the one calculated by the filter.} estimate of the robot's states~\cite{Huang2010ObservabilitybasedRF}.
	
	There is a significant body of literature addressing the issue of ``inconsistency" in SLAM estimators. One pioneering work in this field is by Barrau et al.~\cite{Barrau2015AnEA}, who introduced the concept of Kalman Filter with error defined on  Lie groups, known as InEKF. This approach provides a more accurate and consistent estimation framework for SLAM. The substantial study~\cite{Barrau} also included an analysis of convergence and highlighted an affine property of $SE_2(3)$~(the extension of $SE(3)$), further improving the performance of EKF. However, it is important to acknowledge that due to the non-linearity of the measurement model, achieving optimality in the MMSE sense remains a persisting issue in pose estimates, despite the advancements made by InEKF and related techniques.
	
	\subsubsection*{Bayes filter}
	The Bayes filter is a probabilistic approach used to recursively estimate the probability density function of state using measurement and a mathematical  model. It is a more general framework that can be applied to a broader class of problems, including non-linear and non-Gaussian systems and widely recognized that the KF is a specific implementation of Bayes filter designed for linear systems with Gaussian noises.
	
	When considering uncertainty analysis on Lie groups for Bayes filter, pioneering work has demonstrated that the dispersion of mobile robots, under the influence of sensor noises, exhibits a shape resembling a ``banana'' rather than a standard ellipse\cite{Long2012TheBD}. This phenomenon can be accurately approximated using Gaussian distributions on the Lie exponential coordinates of the special Euclidean group $SE(2)$. In a related study by Wang and Chirikjian ~\cite{YunfengWang2006}, they introduced the concept of the concentrated Gaussian distribution, which is defined on $SE(3)$.
	
	Furthermore, the application of Bayes fusion on the pose $SE(3)$ has been explored in the work of Barfoot et al.\cite{Barfoot2014AssociatingUW}. Additionally, Brossard et al.\cite{Brossard2020AssociatingUT} investigated the application of Bayes fusion specifically on the extended pose $SE_2(3)$. These studies highlighted the uncertainty propagation on different Lie groups. While they have focused on uncertainty analysis in the propagation step of estimation, further research is needed to explore the optimality of the update step in estimation algorithms.
	
	\subsubsection*{ A maximum likelihood estimation approach}
	The Bayes filter update step can be formulated as a maximum likelihood estimation (MLE) problem. It is well-known that the MLE formulation, assuming i.i.d. Gaussian noises, can be equivalently represented as a LS problem. The pioneering work by Bell et al.\cite{Bell93} introduced the concept of IEKF, which utilizes the non-linear least squares~(NLS) formulation for the update step and proposes the GN method as a numerical solution. The IEKF has found extensive applications in various domains, including pose estimates~\cite{Lin2021R3LIVEAR} and calibration~\cite{mirzaei2008kalman}.
	
	However, optimizing the NLS formulation on Lie groups poses several challenges. Firstly, while the GN method has been extended to manifolds such as $SO(3)$ and $SE(3)$~\cite{Forster2015OnManifoldPF}, concerns regarding computational cost in the optimization framework persist~\cite{Leutenegger2015KeyframebasedVO,Qin2017VINSMonoAR}. Secondly, ensuring the convergence of the GN method to the NLS problem remains uncertain~\cite{dennis1996numerical,carlone2013convergence}, particularly with regards to the choice of the initial value and the number of iterations. 
	
	\subsubsection*{Inertial-based odometries platform}
	In the field of VIO, OpenVINS\cite{Geneva2020OpenVINSAR} provides an on-manifold sliding window KF, showing competitive estimation performance.
	For LIO, Fast-LIO2~\cite{xu2022fast} offers a fast, robust, and versatile navigation framework based on IEKF.
	The benefits of fusing multiple measurements to enhance pose estimates accuracy have been discussed extensively by researchers~\cite{Yang2022OnlineSF,Zhang2016OnDO}. 
	R3LIVE, a LiDAR-IMU-Visual platform designed for real-time localization and radiance map reconstruction, is based on iterated  EKF~\cite{Lin2021R3LIVEAR}.
	While these works explore the utilization of filters for navigation, they are primarily based on the EKF or IEKF, providing limited consideration for the iterative scheme on Lie group~$SE_2(3)$. Furthermore, there is a noticeable absence of statistical performance assessment with a substantial number of samples. In conclusion, there exists ample room for enhancing the estimates performance of filter-based platforms, especially in terms of asymptotic optimality.
	
	\subsection{Organization of the paper}
	The remainder of this paper is organized as follows. Section~\ref{sec:math_preliminary} revisits selected preliminaries, and Section~\ref{sec:model} describes our problem and models. Prediction of our algorithm is described in Section~\ref{sec:filtering_design_predict} and a novel update of our algorithm in Section~\ref{sec:update} is introduced. Section~\ref{sec:theoretical analysis} provides the analysis of the efficiency and optimality of our algorithm. Section~\ref{sec:experiment} reports simulation results and experimental evaluation on real-word dataset and our collected hardware. Finally, Section~\ref{sec:conclusion} concludes the paper and envisions future work.
	
	\textbf{Notation}: The bracket $[\cdot,\cdot]$ denotes the Lie bracket. For $\bs{x}\triangleq\begin{bmatrix}
		x_1&\cdots&x_m
	\end{bmatrix}^\top\in\mathbb{R}^m$, $\bs{X}\in\mathbb{R}^{m\times n}$ and $f:\mathbb{R}^m\rightarrow\mathbb{R}^n$, $\frac{\partial f}{\partial \bs{x}}\triangleq\begin{bmatrix}
		\frac{\partial f}{\partial {x_1}}&\cdots&\frac{\partial f}{\partial {x}_m}
	\end{bmatrix}$. The following two linear operators:~$\langle\langle\bs{A}\rangle\rangle \triangleq -\operatorname{tr}(\bs{A}) \bs{I}+\bs{A}, 
	\langle\langle\bs{A}, \bs{B}\rangle\rangle \triangleq\langle\langle\bs{A}\rangle\rangle\langle\langle\bs{B}\rangle\rangle+\langle\langle\bs{B} \bs{A}\rangle\rangle$ and the vectorization operator $vec([\bs{x}_1\cdots \bs{x}_n])\triangleq[\bs{x}^\top_1,\cdots,\bs{x}^\top_n]^\top$ will be used. 
	
	\section{ PRELIMINARIES}\label{sec:math_preliminary}
	\subsection{Lie groups in robotics:$SO(3)$, $SE(3)$ and $SE_2(3)$}
	We review  Lie groups for robotics.  
	The special Euclidean group $SE(3)$ can describe relative pose between different frames, including a rotation matrix $\bs{R}$ on the special orthogonal group $SO(3)$ and a translation vector $\bs{p}\in \mathbb{R}^3$, as follows:
	$$SE(3)\triangleq
	\left\lbrace\begin{bmatrix} \bs{R} &\bs{p}\\\bs{0}_{1\times3}&1\end{bmatrix}\vline \bs{R}\in SO(3),\bs{p}\in\mathbb{R}^{3} \right\rbrace.$$
	The Lie algebra of $SE(3)$ is denoted by $\mathfrak{se}(3)$ writes$$\mathfrak{se}(3)\triangleq\left\lbrace \begin{bmatrix}\bs{\omega}^{\wedge} & \bs{t} \\ \bs{0}_{1\times3} & 0\end{bmatrix}\vline \bs{\omega}^{\wedge}\in\mathfrak{so}(3),\bs{t}\in\mathbb{R}^3\right\rbrace, $$
	where $(\cdot)^\wedge$ is used to represent the mapping from the Lie exponential coordinates to the corresponding Lie algebra. For instance, in $\mathfrak{so}(3)$, $(\cdot)^\wedge:\mathbb R^3\rightarrow \mathfrak{so}(3)$ denotes the skew matrix. 	
	The matrix Lie group $SE_2(3)$, known as an \textit{extended pose} in the area of pose estimation and proposed in~\cite{Barrau}, includes a rotation matrix $\bs{R}$, position $\bs{p}$ and velocity $\bs{v}$. The matrix group $SE_2(3)$ writes:
	$$SE_2(3)\triangleq\left\lbrace\begin{bmatrix}
		\bs{R}&\vline &\bs{p}&\bs{v}\\
		\hline
		\bs{0}_{2\times3}&\vline&\multicolumn{2}{c}{\begin{matrix}
				\bs{I}_2
		\end{matrix}}
	\end{bmatrix}\vline \bs{R}\in SO(3),\bs{p},\bs{v}\in\mathbb{R}^{3} \right\rbrace. $$
	The corresponding Lie algebra $\mathfrak{se}_2(3)$ writes:	
	$$\mathfrak{se}_2(3)\triangleq\left\lbrace\begin{bmatrix}
		\bs{\omega}^\wedge&\vline &\bs{t}_1&\bs{t}_2\\
		\hline
		\bs{0}_{2\times3}&\vline&\multicolumn{2}{c}{\begin{matrix}
				\bs{0}_2
		\end{matrix}}
	\end{bmatrix}\vline \bs{\omega}^\wedge\in \mathfrak{so}(3),\bs{t}_1,\bs{t}_2\in\mathbb{R}^{3} \right\rbrace. $$
	% In addition, for the sake of brevity, we use an operator $\mathcal{L}$ to denote the extraction mapping from $SE_2(3)$ to $SE(3)$:
	% \begin{equation*}
		% 	\mathcal{L}(\begin{bmatrix}
			% 		\bs{R}&\vline &\bs{p}&\bs{v}\\
			% 		\hline
			% 		\bs{0}_{2\times3}&\vline&\bs{I}_{2}
			% 	\end{bmatrix})=\begin{bmatrix} \bs{R} &\bs{p}\\\bs{0}_{1\times3}&1\end{bmatrix}.
		% \end{equation*}
	\subsubsection{Exponential, logarithm and adjoint operators}
	Lie groups and Lie algebras are linked by the exponential and logarithmic operations. We define the exponential map $\exp:\mathbb{R}^9 \rightarrow SE_2(3)$ for Lie groups as:~$\exp(\xi)=\exp_m(\xi^\wedge)$ where $\exp_m$ is the exponential of the matrix. Locally, it is a bijection if restricted to the open ball $\mathcal B_{\pi} \triangleq \{\xi\in\mathbb R^9\vline  \|\xi_{1:3}\|<\pi\}$, and the Lie logarithm, as an exponential inverse,  is defined as a map $\log: SE_2(3) \rightarrow \mathcal B_{\pi}$ by which  $\log{\left(\exp(\xi)\right) }=\xi$. 
	
	We conveniently define the adjoint operator for $X\in SE_2(3)$:
	\begin{equation}\label{eqn:definition_of_adjoint}
		\Ad_X\triangleq\begin{bmatrix}
			\bs{R}&\bs{0}_3&\bs{0}_3\\
			\bs{p}^\wedge\bs{R}&\bs{R}&\bs{0}_3\\
			\bs{v}^\wedge\bs{R}&\bs{0}_3&\bs{R}
		\end{bmatrix}
	\end{equation}
	as an operator acting directly on $y\in\mathbb R^9$. The operator $\Ad_Ty$ is also equivalent to $Ty^\wedge T^{-1}$. We also define the the adjoint operator for $x\in \mathfrak{se}_2(3)$:
	\begin{equation}\label{eqn:definition_of_adjoint_differential}
		\ad_x\triangleq\begin{bmatrix}
			\bs{\omega}^\wedge&\bs{0}_3&\bs{0}_3\\
			\bs{t}_1^\wedge&\bs{\omega}^\wedge&\bs{0}_3\\
			\bs{t}_2^\wedge&\bs{0}_3&\bs{\omega}^\wedge
		\end{bmatrix}
	\end{equation}
	as an operator acting directly on $y\in\mathbb R^9$. The operator $\ad_xy$ is also equivalent to $[x^\wedge,y^\wedge]$.
	%	The adjoint action of $\bs{X}\in SE_n(3)$ on $\bs{y}^\wedge \in \mathfrak{se}_n(3)$ is defined as $\operatorname{Ad}_{\bs{X}}: \mathfrak{se}_n(3) \rightarrow \mathfrak{se}_n(3), \bs{y}^\wedge \mapsto \bs{X}\bs{y}^\wedge \bs{X}^{-1}$. The differential of the adjoint action $\operatorname{Ad}_{\bs{X}}$ at the identity element of $SE_n(3)$, denoted as $\operatorname{ad}_{\bs{x}}: \mathfrak{se}_n(3) \rightarrow \mathfrak{se}_n(3)$, is a linear mapping from $\mathfrak{se}_n(3)$ to itself. The matrix representations of the adjoint action of $SE_n(3)$ and the adjoint action of $\mathfrak{se}_n(3)$ on itself are:
	%	\begin{equation*}
		%		\begin{split}
			%			\begin{bmatrix}\bs{R} &\bs{0} &\bs{0} &\bs{0}\\
				%				\bs{t}_1^\wedge\bs{R}&\bs{R}&\bs{0} &\bs{0}\\
				%				\vdots &\bs{0} &\ddots &\bs{0}\\
				%				\bs{t}_n^\wedge\bs{R} &\bs{0} &\bs{0} &\bs{R}
				%			\end{bmatrix},
			%			\begin{bmatrix}\omega^\wedge &\bs{0} &\bs{0} &\bs{0}\\
				%				\bs{v}_1^\wedge &\omega^\wedge &\bs{0} &\bs{0}\\
				%				\vdots &\bs{0} &\ddots &\bs{0}\\
				%				\bs{v}_n^\wedge &\bs{0} &\bs{0} &\omega^\wedge
				%			\end{bmatrix}.
			%		\end{split}
		%	\end{equation*}
	\subsubsection{{Baker-Campbell-Hausdorff} (BCH) formula:}
	To compound two matrix exponentials $\exp_m(x^\wedge)$ and $\exp_m(y^\wedge)$, we need to use the \textit{Baker-Campbell-Hausdorff} (BCH) formula:
	\begin{equation*}
		\exp_m(x^\wedge)\exp_m(y^\wedge)=\exp_m(x^\wedge+y^\wedge+\frac{1}{2}[x^\wedge,y^\wedge]+\cdots).
	\end{equation*}
	If we assume that $\left\| x\right\| $ or $\left\| y\right\| $ is close to zero, we provide the following Lemma to approximate BCH formulas.
	\begin{lemma}[Theorem 5.3 in \cite{Hall2004LieGL}]\label{lemma: compound of two matrix exponentials}
		For any $x^\wedge,y^\wedge\in\mathfrak{se}_2(3)$, their matrix exponentials is compounded by the following approximation:
		\begin{enumerate}[label=(\roman*)]
			\item if $\left\| x\right\|$ is sufficiently small, we have
			\begin{equation*}
				\exp(x)\exp(y)\approx\exp(\dexp^{-1}_yx+y),
			\end{equation*}
			\item  if $\left\| y\right\|$ is sufficiently small, we have
			\begin{equation*}
				\exp(x)\exp(y)\approx\exp(x+\dexp^{-1}_{-x}y),
			\end{equation*}
			\item 	if $\left\| x\right\|$ and $\left\| y\right\|$ are both sufficiently small quantities, we recover the first-order approximation, 
			\begin{equation*}
				\exp{(x)}\exp{(y)}\approx\exp{(x+y)}.
			\end{equation*}
		\end{enumerate}
		Here, $\dexp_x$ is known as the left Jacobian of $x^\wedge\in\mathfrak{se}_2(3)$, and its inverse has the form:
		\begin{equation}\label{eqn:left_jacobian}
			\dexp^{-1}_x=\sum_{i=0}^{\infty}\frac{(-1)^{i}}{(i+1)!}(\ad_x)^{i},
		\end{equation} 
		where the operator $\ad_x(y^\wedge)$ is defined in~\eqref{eqn:definition_of_adjoint_differential}.
	\end{lemma}
	
	\subsection{Gauss-Newton method on Lie group~(LGN)}\label{sec:description of the LGN}
	The Gauss-Newton method on a Lie group~(LGN) is a generalization of the GN method for non-linear least-squares problems in a vector space.  One of the main advantages of the LGN is its ability to efficiently optimize on a Lie group. Here, we will review how to extend the classic Gauss-Newton method to LGN on $SE_2(3)$.
	
	Let us consider the following optimization problem:
	\begin{equation}\label{eqn:GN_orignal_problem}
		\bs{X}^* = \mathop{\arg\min}_{\bs{X}\in SE_2(3)} \left\|  r(\bs{X})\right\|^2.
	\end{equation}
	Different from the GN in a Euclidean space, the LGN method needs a mapping called \textit{retraction} $R_{\bs{X}}$, which is defined as the mapping from an increment $\delta x$ in the tangent space of $\bs{X}$ to the neighbourhood of $\bs{X}$. Specifically, in the paper,  we choose the left multiplication as the operation of the retraction, which is defined as:
	\begin{equation*}
		R_{\bs{X}}(\delta x)\triangleq\exp{(\delta x)}\bs{X}.
	\end{equation*}
	We linearize the problem by approximating $r(\bs{X})$ with its first-order Taylor series expansion around the current estimate of the minimizer. The resulting linearized problem is solved by performing a sequence of group operations, which include finding the tangent space and {retraction} to the  Lie group. Each iteration is written in the following form:
	\begin{align}\label{eqn:GN_transformed_problem}
		&\delta x^*=\mathop{\arg\min}_{\delta x\in \mathbb{R}^9}\left\| r(\hat{\bs{X}}^{(l)})+J_{\hat{\bs{X}}^{(l)}}\delta x\right\|^2,\\
  &\hat{\bs{X}}^{(l+1)}= R_{\hat{\bs{X}}^{(l)}}(\delta x^*),\nonumber
	\end{align}
 where $J_{\hat{\bs{X}}^{(l)}}\triangleq\frac{\partial r(R_{\hat{\bs{X}}^{(l)}}(\delta x))}{\partial \delta x}|_{\delta x=\bs{0}}$.
	For the sake
	of notation convenience, we introduce the compact form,
	\begin{equation}\label{eqn:matrix_differential}
		\frac{\partial r(\bs{X})}{\partial \bs{X}}|_{\bs{X}=\hat{\bs{X}}^{(l)}}\triangleq\frac{\partial r(R_{\hat{\bs{X}}^{(l)}}(\delta x))}{\partial \delta x}|_{\delta x=\bs{0}}.
	\end{equation}
	We remark that~\eqref{eqn:matrix_differential} is just a notation convenience for the differential of a matrix Lie group, since division by a vector or a matrix is undefined.  By the solver to the LS~\eqref{eqn:GN_transformed_problem}, we have each iteration time solved in the following form:
	\begin{align*}
		\delta x^* = -\left( J^\top_{\hat{\bs{X}}^{(l)}}J_{\hat{\bs{X}}^{(l)}} \right)^{-1} J^\top_{\hat{\bs{X}}^{(l)}} r(\hat{\bs{X}}^{(l)}). 
	\end{align*}
	
	In~\cite{Tunel2009OptimizationAO}, the reparametrization process is referred to as ``\textit{lift-solve-retract}", as the method reparametrizes~\eqref{eqn:GN_orignal_problem} to~\eqref{eqn:GN_transformed_problem} and then retracts the parametrization to the original coordinates of the manifold. In comparison to computing the solution of Lie algebra directly, the retraction provides a framework for analyzing less expensive alternatives in a Lie group without approximation because it is not necessary to compute the left Jaoobian or the right Jacobian.
	
	\subsection{Uncertainty description on a Lie group}\label{sec:intro_CGN}
	We introduce the concentrated Gaussian distribution as a generalized version of a Gaussian distribution, which will be applied on a Lie group. 
	 Let $\bs{X}\in SE_2(3)$ be a random variable concentrated at the group identity, whose covariance is sufficiently small, and its probability density function~(PDF) is given by
	\begin{equation*}
		f(\bs{X};I,\Sigma) =\alpha(\Sigma)\exp{\left( -\frac{1}{2}\|\log{(\bs{X})}\|^2 _\Sigma\right) }.
	\end{equation*}
	where $\alpha(\Sigma)$ is a normalization factor\footnote{In general, the normalization factor is a function of $\bs{X}$. However, when the covariance is small, it can be assumed to be well approximated by a factor independent of $\bs{X}$. The detailed derivation can be seen in~\cite[Sec.III-B]{Forster2015OnManifoldPF}. } to ensure $f$ to be a PDF. The definition also appears in~\cite{Brossard2020AssociatingUT}. 
	
	By applying a right action $\mu\in SE_2(3)$ on $\bs{X}$, we sent $\bs{X}$ to $\bs{X}'$ centered at $\mu$, denoted by
	$
	\bs{X}'  =  \bs{X}\mu.
	$
	This is called a \textit{right concentrated Gaussian distribution}, denoted as $\mathcal{N}_{RG}(\mu,\Sigma)$ and the PDF of $\bs{X}'$ is written as
	\begin{equation*}
		\alpha(\Sigma)\exp{\left( -\frac{1}{2}\|\log{(\bs{X}'\mu^{-1})}\|^2_ \Sigma\right) }.
	\end{equation*}
	If a left action is applied, i.e., $\bs{X}'=\mu\bs{X}$, we obtain a \textit{left concentrated Gaussian distribution}, denoted as $\mathcal{N}_{LG}(\mu,\Sigma)$.
	
	To summarize this part, when we work with a concentrated Gaussian distribution to compute Lebesgue integral with respect to it on $SE_2(3)$, it is equivalent to compute with respect to  a Gaussian distribution in the Lie algebra.
When the covariance is small, the tails of the distribution in question where bijection fails are marginally necessary to take into account. As shown in~\cite{YunfengWang2006}, the convolution 
of highly concentrated
at the identity is identical to the convolution in the Euclidean space $\mathbb R^9$. The conditioning of $\mathcal{N}_{RG}$ at the identity follows in the same way that it does in $\mathbb{R}^9$. 
	  In this paper, we focus on $\mathcal{N}_{RG}$, and a similar result can be derived for $\mathcal{N}_{LG}$ by utilizing $\mathcal{N}_{RG}$ and $\Ad_{\mu}$, i.e., $\mathcal{N}_{LG}(\mu,\Sigma)$ is structurally identical to $\mathcal{N}_{RG}(\mu,\Ad_{\mu}\Sigma\Ad^\top_{\mu})$, which will not be elaborated on.

	\subsection{Estimation theory}	
	We introduce several key definitions borrowed from statistical theory. First, we need to introduce the conventions of $O_p$ and $o_p$ for our paper. The notion ${X}_n=o_p(a_n)$ means that the set of values $X_n/a_n$ converges to zero in probability, i.e., $\lim_{n\rightarrow\infty}P(|{X}_n/a_n|\geq\epsilon)=0$ for every positive $\epsilon$. The notion $X_n=O_p(a_n)$ means that the set of values $X_n/a_n$ is bounded in probability, i.e., for any positive $\epsilon$, there exists a finite $N$ and $M$ such that $P(\left|X_n/a_n \right| >M)<\epsilon$ for any $n>N$. Similarly to deterministic $O$ and $o$ notation, the probabilistic $O_p$ and $o_p$ notation follows similar rules of calculus.
	\begin{lemma}\label{lemma:calculus_pf_op}
		The following equations hold true for $o_p$ and $O_p$.
		\begin{enumerate}[label=(\roman*)]
			\item Let $r_n$ be a sequence of real numbers. $O_p(r_n)=r_nO_p(1)$ and $o_p(r_n)=r_no_p(1)$.
			\item $o_p(1) + o_p(1) = o_p(1)$ and $O_p(1) + O_p(1) = O_p(1)$ .
			\item $o_p(1) + O_p(1) = O_p(1)$.
			\item if $r(\cdot)$ is a continuous mapping and $\bs{x}_n=\bs{x}_o+O_p(1/\sqrt{n})$ holds, then $r(\bs{x}_n)=r(\bs{x}_o)+O_p(1/\sqrt{n})$ holds\cite{Oehlert1992ANO}. 
            \item if $r(\cdot)$ is a continuous mapping and $\bs{x}_n=\bs{x}_o+o_p(1/\sqrt{n})$ holds, then $r(\bs{x}_n)=r(\bs{x}_o)+o_p(1/\sqrt{n})$ holds\cite{Oehlert1992ANO}. 
		\end{enumerate}    
	\end{lemma}
	
 Now suppose $\bs x \in \mathbb{R}^d$ is a parameter to be estimated, and ${\bs Y}_{1:n} \triangleq \{Y_i\}_{i=1}^n$ is a collection of $n$ independent measurements, where $Y_i$ is generated from a PDF $p(Y | \bs{x})$. Let $\hat{\bs{x}}_n$ be an estimator of $\bs{x}$ derived from the $n$ measurements. 
	\begin{definition}[Consistency in statistical theory]\label{def:consistency}
		The estimator $\hat{\bs{x}}_n$ is said to be  consistent in statistical theory if it converges in probability to the true value of the state $\bs{x}$. This is defined as:
		\begin{equation*}
			\lim_{n\rightarrow\infty} {P} \left( \| \hat{\bs{x}}_n-\bs{x} \| >\epsilon \right)=0,\quad\text{for any} \quad\epsilon>0.
		\end{equation*}
	\end{definition}
	
	Consistency implies that as the quantity of measurements used increases, the sequence of estimates converges to the true value. It is important to note that this definition of consistency differs from the one widely used in filter, where the latter characterizes the relationship between the estimated and actual covariance matrices\cite{Huang2008AnalysisAI}. In the rest of this paper, which  definition we are referring to  will be made clear from the context.
	
	\begin{definition}[$\sqrt{n}$-consistency]\label{def:n-consistent}
		We say that $\hat{\bs{x}}_n$ is a $\sqrt{n}$-consistent estimator of $\bs{x}$ if 
		$\hat{\bs{x}}_n-\bs{x} = O_p(\frac{1}{\sqrt{n}})$.
	\end{definition}
	It is noteworthy that $\sqrt{n}$-consistency implies consistency. In addition, it also characterizes the rate ($1/\sqrt{n}$) at which an estimator converges to the true value in probability concerning the quantity of samples.
	
	\begin{definition}[Efficiency]\label{def:efficiency}
		An unbiased estimator $\hat{\bs{x}}_n$ is said to be efficient if its mean squared error (MSE), given by
		$$
		\operatorname{MSE}(\hat{\bs{x}}_n)= \mathbb E \left[ \|\hat{\bs{x}}_n-\bs{x}\|^2 \right],
		$$
		achieves the theoretical lower bound --- Cram\`er-Rao  lower bound~(CRLB)\footnote{There is no unbiased estimator which is able to estimate the parameter with less variance than the Cram\`er-Rao lower bound.}. In particular, $\hat{\bs{x}}_n$ is considered to be asymptotically efficient if it reaches the CRLB when $n \rightarrow \infty$. 
	\end{definition}
	The CRLB is calculated by:
	\begin{equation*}
		\operatorname{CRLB}=\operatorname{trace}(\bs{F}^{-1}),
	\end{equation*} 
	where $\bs{F}$ is the Fisher information matrix, which is defined as 
	$$
	\bs{F}\triangleq\mathbb{E} \left[ \frac{\partial \ell({\bs x};{\bs Y}_{1:n}) }{\partial{\bs x}}\frac{\partial \ell({\bs x};{\bs Y}_{1:n}) }{\partial {\bs x}^\top }\right],
	$$ where $\ell({\bs x};{\bs Y}_{1:n})\triangleq\log{p({\bs Y}_{1:n}|{\bs x})}$, which is called the log-likelihood function.  Efficiency serves as a crucial measure of the quality of an estimator. An efficient estimator achieves the smallest MSE, that is, CRLB, among all unbiased estimators.

	\section{Sensor Models and Problem Formulation}\label{sec:model}
	In this section, we introduce our IMU model and outline the observation models for the camera and LiDAR sensors. Furthermore, we define two problems that are the focus of the investigation.  
	
	To enhance the readability of the paper,  we list symbols used frequently in the Nomenclature table.
	\begin{framed}
		\nomenclature[01]{\textbf{List of Symbols}}{\textbf{Descriptions}}
		\nomenclature[02]{\(\{\mathcal{W},\mathcal{I}, 
			\mathcal{C}, \mathcal{L}\}\)}{ Frames of global, IMU, camera, and LiDAR }
		\nomenclature[03]{\(\bs{X}\in SE_2(3)\)}{State of IMU in $\left\lbrace \mathcal{W}\right\rbrace $ }
		\nomenclature[04]{\(\bs{T}_C,\bs{T}_{L}\in SE(3)\)}{State of camera and LiDAR in $\left\lbrace \mathcal{W}\right\rbrace $ }
		\nomenclature[05]{\(\bs{R}_{I},\bs{R}_{C},\bs{R}_{L}\in SO(3)\)}{Rotations of IMU, camera, and LiDAR in $\left\lbrace \mathcal{W}\right\rbrace $}
		\nomenclature[06]{\(\bs{p}_{I},\bs{p}_{C},\bs{p}_{L}\in \mathbb{R}^3\)}{Positions of IMU, camera, and LiDAR in $\left\lbrace \mathcal{W}\right\rbrace $ }
		\nomenclature[07]{\(\bs{\omega}_I\)}{Rotational velocity of IMU}
		\nomenclature[08]{\(\bs{\omega}_m\)}{Measurement of rotational velocity}
		\nomenclature[09]{\(\bs{a}_I\)}{Acceleration of IMU in $\left\lbrace \mathcal{W}\right\rbrace $}
		\nomenclature[10]{\(\bs{a}_m\)}{Measurement of acceleration of IMU}
		\nomenclature[11]{\(\bs{b}_g,\bs{b}_a\)}{Bias of the IMU gyroscope and accelerometer }
		\nomenclature[12]{\(\bs{z}_{C},\bs{z}_{L}\)}{Measurements of camera and LiDAR}
		\nomenclature[13]{\(\eta\)}{Right-invariant error of the true state and estimator of IMU}
		\nomenclature[14]{\(\xi\)}{Exponential coordinates of $\eta$ }
		\nomenclature[15]{\(\bs{P}\)}{Covariance of $\xi$}
		\nomenclature[16]{\(\bs{K}\)}{Kalman gain}
		\nomenclature[17]{\(\bar{(\cdot)}\)}{Predicted variable of $(\cdot)$}
		\nomenclature[18]{\(\hat{(\cdot)}\)}{Updated variable of $(\cdot)$}
		\nomenclature[19]{\(\mu^{(l)}\)}{Value at the $l$th step in the iterative update step}
		\nomenclature[20]{\(\delta^{(l)}\)}{Lie exponential coordinates of $\mu^{(l)}\bar{\bs{X}}^{-1}$}
		\printnomenclature[1in]
	\end{framed}
	
	\subsection{State representation in a multisensor system}
	
	An inertial state $\bs{X}$ to be estimated includes the orientation, position, and velocity of the IMU in $\{\mathcal{W}\}$,  which are represented by $\bs{R}_{I}$,  $\bs{p}_{I}$ and $\bs{v}_{I}$, respectively. The state writes by the following matrix:
	\begin{equation*}
		\begin{split}
			\bs{X}&\triangleq\begin{bmatrix}
				\bs{R}_{I}& \bs{p}_{I} &  \bs{v}_{I}\\
				\bs{0}&\multicolumn{2}{c}{\begin{matrix}
						\bs{I}_2
			\end{matrix}}\end{bmatrix}\in SE_2(3).
		\end{split}
	\end{equation*} 
	The orientation and position of the monocular camera and LiDAR in $\{\mathcal{W}\}$  can be modeled in $SE(3)$. These states are represented by $\bs{R}_{C}$, $\bs{p}_{C}$, $\bs{R}_{L}$ and $\bs{p}_{L}$, respectively, and write in the following matrices:
	\begin{equation*}
		\bs{T}_{C}\triangleq\begin{bmatrix}
			\bs{R}_{C}& \bs{p}_{C}\\
			\bs{0}_{1\times3}&{1}\end{bmatrix},\quad			\bs{T}_{L}\triangleq\begin{bmatrix}
			\bs{R}_{L}& \bs{p}_{L}\\
			\bs{0}_{1\times3}&{1}\end{bmatrix}\in SE(3).
	\end{equation*}
	The sensors are assumed to be fixed on the robot's body, and the states of the sensors can be expressed in terms of a relative transformation between them. The relative transformations of the camera and LiDAR in $\{\mathcal{I}\}$ are denoted by $^{I}\bs{T}_{C}$, $^{I}\bs{T}_{L}\in SE(3)$:
	\begin{equation}\label{eqn:relative transformation}
		^{I}\bs{T}_{C}\triangleq\begin{bmatrix}
			^{I}\bs{R}_C&^{I}\bs{p}_C\\
			\bs{0}_{1\times3}&1
		\end{bmatrix},\quad^{I}\bs{T}_{L}\triangleq\begin{bmatrix}
			^{I}\bs{R}_L&^{I}\bs{p}_L\\
			\bs{0}_{1\times3}&1
		\end{bmatrix}.
	\end{equation}

	\subsection{Kinematics of the IMU}
	{The IMU measures the ego motion of the robot.} Its measurements of the gyroscope $	\bs{\omega}_{m}$ and accelerometer $\bs{a}_{m}$ are modeled as
	\begin{equation}\label{eqn:IMU measurement}
		\bs{\omega}_{m} =\bs{\omega}_I +\bs{b}_{{g}} +\bs{n}_{{g}} ,\quad
		\bs{a}_{m} =\bs{R}_{I}^\top \left(\bs{a}_{I} - \bs{g}\right)+\bs{b}_{{a}} +\bs{n}_{{a}} ,    
	\end{equation}
	%where $\bs{\omega}_{I} $ and $\bs{a}_{I} $ denote IMU's rotational velocity and acceleration in the world frame,
	where the measurement noises $\bs{n}_{g}$ and $\bs{n}_{a}$ follow independent zero-mean white Gaussian process with covariance matrices  $\sigma_{g}^2\bs{I}_{3}$ and $\sigma_{a}^2\bs{I}_{3}$, $\bs{g}=\begin{bmatrix}0&0&-9.81\end{bmatrix}^\top \in\mathbb{R}^3$ is the gravitational acceleration and $\bs{b}_g, \bs{b}_a$ denotes the bias vectors affecting the IMU gyroscope and accelerometer.  The biases are modeled using the typical ``Brownian motion'' model to capture the slowly time-varying nature of these parameters:
	\begin{equation}\label{eqn:bias}
		\dot{\bs{b}}_{{g}} =\bs{n}_{{b g}},\quad \dot{\bs{b}}_{{a}} =\bs{n}_{{b a}}, 
	\end{equation}
	where the measurement noises $\bs{n}_{{bg}}$ and $\bs{n}_{{ba}}$ are typically modeled as zero-mean white Gaussian processes with covariance matrix $\sigma_{bg}^2\bs{I}_{ 3}$ and $\sigma_{ba}^2\bs{I}_{3}$.  We note that $\sigma_g,\sigma_a,\sigma_{bg}$ and $\sigma_{ba}$ can be calibrated offline.
	
	The kinematics of state $\bs{X}$ is given by the following equations:
	\begin{align}
		\dot{\bs{R}}_{I} &=\bs{R}_{I} 	\bs{\omega}_I ^{\wedge},\quad
		\dot{\bs{p}}_{I} =\bs{v}_{I} ,\quad
		\dot{\bs{v}}_{I} =\bs{a}_{I} .\label{eqn:IMU_state_dynamics}
	\end{align} Since the increment of states $\bs{\omega}_I, \bs{a}_I$ is measured by the IMU, the whole dynamic can be expressed as 
	\begin{align*}
		\dot{\bs{R}}_{I} &=\bs{R}_{I}(\bs{\omega}_{m} -\bs{b}_{{g}} -\bs{n}_{{g}})^{\wedge},\quad\dot{\bs{p}}_{I} =\bs{v}_{I},\\
		\dot{\bs{v}}_{I} &=\bs{R}_{I}(\bs{a}_m -  \bs{b}_a - \bs{n}_a)+\bs{g} ,
	\end{align*}
	which writes in the matrix differential equation as
	\begin{align}\label{eqn:imu_true_dynamics}
		\dot{\bs{X}} =\bs{X} \bs{v}_b+\bs{v}_g\bs{X} +\bs{M}\bs{X} \bs{N},
	\end{align}
	where $\bs{M}\triangleq\begin{bmatrix}
		\bs{I}_{3}
		&\bs{0}_{3\times2}\\
		\bs{0}_{2\times3}&
		\begin{matrix}
			\bs{0}_{2}
		\end{matrix}
	\end{bmatrix}$ 
	and $\bs{N}\triangleq\begin{bmatrix}
		\bs{0}_{4\times3}
		&\rvline&\bs{0}_{4\times2}\\
		\hline
		\bs{0}_{1\times3}&\rvline&
		1\quad 0
	\end{bmatrix},$ and 	
	$\bs{v}_b\triangleq\begin{bmatrix}
		(\bs{\omega}_m -  \bs{b}_{g} - \bs{n}_g ) ^\wedge
		&\rvline&\bs{0}_{3\times1}\quad\bs{a}_m -  \bs{b}_a - \bs{n}_a\\
		\hline
		\bs{0}_{2\times3}&\rvline&
		\bs{0}_{2}
	\end{bmatrix}\in\mathfrak{se}_2(3)$ and 
	$\bs{v}_g\triangleq\begin{bmatrix}
		\bs{0}_{3}
		&\rvline&\bs{0}_{3\times1}\quad\bs{g}\\
		\hline
		\bs{0}_{2\times3}&\rvline&
		\begin{matrix}
			\bs{0}_{2}
		\end{matrix}
	\end{bmatrix}\in\mathfrak{se}_2(3)$
	are inputs relative to $\{\mathcal{I}\}$ and $\{\mathcal{W}\}$ respectively.	
	%	The temporal offsets of IMU-camera and IMU-LiDAR are modeled as a random walk with white noise $\bs{n}_{IC}$ and $\bs{n}_{IL}$ respectively {for compensating the 
		%		communication and computational latency. }
	
	\subsection{Measurement models of camera and LiDAR}
	Raw LiDAR and camera measurements are initially processed in the front-end to detect and track features that are assumed to be known planes and points in $\left\lbrace \mathcal{W} \right\rbrace$. Once the front-end processing is complete, the positions of point features in the local coordinates are provided to the back-end for updating state estimates.
	
	The position of the $i$th visual feature in $\left\lbrace \mathcal{W} \right\rbrace$ is denoted as $\bs{p}_{f_{i}}\in\mathbb{R}^3$. The measurement $\bs{z}_{C_i}$ of $\bs{p}_{f_{i}}$ in the camera image is described as follows:
	\begin{align}\label{eqn:camera_model}
		\bs{z}_{C_i} &={h}_{p}\left(\bs{R}^\top_{C}\bs{p}_{f_i}-\bs{R}_C^\top\bs{p}_C\right)+\bs{n}_{C_i},
	\end{align}
	where ${h}_p(\begin{bmatrix}
		p_{x} &p_{y} &p_{z}
	\end{bmatrix}^\top) \triangleq K_I \begin{bmatrix}p_{x} / p_{z} \quad p_{y} / p_{z}\end{bmatrix}^\top$ is the projection model where $ p_{z}>0$ and $K_I\triangleq\begin{bmatrix}
		f_x&0&u_0\\
		0&f_y&v_0\\
		0&0&1
	\end{bmatrix}$ is the intrinsic parameter matrix of the camera. Measurement noise $\bs{n}_{C_i}$ is modeled as zero-mean Gaussian with covariance $\Sigma_{C_i}\triangleq\sigma_{C}^{2} \bs{I}_{2}$. 
	
	The LiDAR measurement of the $j$th point feature $\bs{p}_{f_j}$ is denoted by $\bs{z}_{L_j}$, which is expressed as: \begin{equation}\label{eqn:LiDAR_model} \bs{z}_{L_j} =\bs{R}^\top_{L}\bs{p}_{f_j}-\bs{R}^\top_{L}\bs{p}_{L}+\bs{n}_{L_j}, \end{equation} where $\bs{n}_{L_j}$ is a zero mean Gaussian noise with a covariance of $\Sigma_{L_j}\triangleq\sigma_{L}^{2} \bs{I}_{3}$. An online calibration procedure can be used to determine suitable values for $\sigma_{C}$ and $\sigma_{L}$, in order to improve the accuracy of estimates. Details of the online calibration can be found in Section~\ref{sec:initial condition of the GN}.
	To improve the robustness of the feature association and level the computational load for LIO, we use the \textit{point-to-plane} method i.e., $\bs{p}_{f_j}$ is assumed to be lying on a small plane formed by its nearby map points.  Let $\bs{u}_j$ be the normal vector of the plane and $\bs{q}_j$ be some point lying on that plane. We have\begin{equation}\label{eqn: Lidar_icp_problem}
		0=\bs{u}_{j}^\top (\bs{p}_{f_j}-\bs{q}_j).
	\end{equation}
	This point-to-plane model has been studied in the existing literature~\cite{lin2021r}\cite{xu2022fast}.
	
	For simplicity, the residuals derived from~\eqref{eqn:camera_model} and \eqref{eqn: Lidar_icp_problem} for $n$ features are expressed with:
	\begin{align}\label{eqn:residuals}
		&r_{C}\triangleq\begin{bmatrix}
			r^\top_{C_1}&\cdots&r^\top_{C_n}
		\end{bmatrix}^\top,r_{L}\triangleq\begin{bmatrix}
			r_{L_1}&\cdots&r_{L_n}
		\end{bmatrix}^\top,\\
		&r_{{C_i}}(\bs{T})\triangleq\bs{z}_{C_i}-{h}_{p}\left(\bs{R}^\top_{C}\bs{p}_{f_i}-\bs{R}_C^\top\bs{p}_C\right),\nonumber\\
		&r_{{L_i}}(\bs{T})\triangleq\bs{u}_{j}^\top (\bs{R}_{L}\bs{z}_{L_j}+\bs{p}_{L}-\bs{q}_j),\nonumber
	\end{align}
	where $\bs{T}_C$ and $\bs{T}_L$ can be substituted with the pose of $\bs{X}$ denoted by $\bs{T}$ due to~\eqref{eqn:relative transformation}. The covariance of $r_C$ and $r_L$ are:
	\begin{equation}\label{eqn:covariance_residual}
		\Sigma_C\triangleq\operatorname{diag}\left\lbrace \Sigma_{C_1},\cdots,\Sigma_{C_n} \right\rbrace ,\Sigma_L\triangleq\sigma^2_L\bs{I}_n. 
	\end{equation}
	
	%It is noted that the corresponding weight of the LiDAR residual~(i.e., the inverse of the covariance) will be changed $1/\sigma^2_L$ due to~\eqref{eqn: Lidar_icp_problem}.
	\subsection{Problem formulation}\label{sec:problem formulation}
	%	A filtering-based method is a popular approach for SLAM due to its computational efficiency, ability to handle non-linear systems, robustness to sensor stochastic noise, and flexibility to adopt different sensor modalities. Further, linearization errors  in the conventional Extended Kalman Filtering~(EKF) introduce drift in the estimate and render the filter for SLAM inconsistent~\cite{Barrau2015AnEA}, which has an effect leading the estimator become over-confident, resulting in non-optimal information fusion. Compared with EKF, Right-invariant Extended Kalman Filtering~(RInEKF) algorithm in InEKF is able to avoid the inconsistency problem due to maintaining an estimate of the unobservable directions. Motivated by these, this paper aims to propose the RInEKF method  for handling the real-time navigation and mapping of the robot using multi-sensors. Different from these works~\cite{Barrau2015AnEA,Hartley2020,Xinghan2022}, the Gaussian Newton method is also adopted in the Lie group for accuracy in the update step and the consistency and efficiency can also be considered.
	We focus on the development of an efficient and optimal filter, called  EIKF, for pose estimation in the context of navigation problems. We have previously defined the continuous dynamics of $\bs{X}$ in~\eqref{eqn:imu_true_dynamics}. We now move on to describing the estimates at their discrete sensor sampling instants. To do this, we use the symbol ``$k$'' to denote the time of the $k$th sensor sampling instant. We denote the measurements of visual features or cloud point features at instant $k$ as $\bs{z}_k$, which accordingly is equal to $[\bs{z}_{k,C_1}^\top,\ldots,\bs{z}_{k,C_{n_k}}^\top]^\top$ for the camera or $[\bs{z}_{k,L_1}^\top,\ldots,\bs{z}_{k,L_{n_k}}^\top]^\top$ for the LiDAR, and then denote 
 the collective measurements up to $k$  as $\mathcal{Z}_{k}\triangleq\{\bs{z}_{1},\ldots,\bs{z}_{k}\}$.  The conditional PDF of state $\bs{X}_{k}$ is denoted as $p(\bs{X}_{k} |\mathcal{Z}_{k})$.
	
	The Bayes filtering frame can be executed recursively to (approximately) compute $p(\bs{X}_{k+1} |\mathcal{Z}_{k+1})$ by first predicting $p(\bs{X}_{k+1}| \mathcal{Z}_{k})$ with the Markov property of the discretized dynamics, which is 
	\begin{equation}\label{eqn:predict_imu} p(\bs{X}_{k+1} |\mathcal{Z}_{k})=\int p(\bs{X}_{k+1} |\bs{X}_{k})p(\bs{X}_{k} |\mathcal{Z}_{k})\d{\bs{X}_{k}}. \end{equation}
	Upon the arrival of either \eqref{eqn:camera_model} or \eqref{eqn:LiDAR_model} providing $\bs{X}_{k+1}$, the next step is to update $p(\bs{X}_{k+1}|\mathcal{Z}_{k+1})$, which is determined by \begin{equation}\label{eqn:update} p(\bs{X}_{k+1} |\mathcal{Z}_{k+1})=\frac{p(\bs{X}_{k+1}|\mathcal{Z}_{k})p(\bs{z}_{k+1}|\bs{X}_{k+1})}{p(\bs{z}_{k+1} |\mathcal{Z}_k)}. \end{equation} 
	The equality is a result of the Bayes rule, with $\bs{z}_{k+1}$ only depending on $\bs{X}_{k+1}$. The conditional probability of $\bs{X}_{k+1}$ before and after the arrival of $\bs{z}_{k+1}$ is called the \textit{a priori} and the \textit{a posteriori}, respectively. We conclude the problems to be solved:
  \setcounter{problem}{0}
	\renewcommand{\theproblem}{1.\arabic{problem}} 
\begin{problem}[Exact Prediction]\label{prob:original predict problem}
		Given  $p(\bs{X}_{k}|\mathcal{Z}_{k})$ and~\eqref{eqn:imu_true_dynamics}, our objective is to predict the state $p(\bs{X}_{k+1}|\mathcal{Z}_{k})$ in~\eqref{eqn:predict_imu}
	\end{problem}

	\begin{problem}[Exact Update]\label{prob:update problem}
		Given the prior probability $p(\bs{X}_{k+1}|\mathcal{Z}_{k})$ with measurement $\bs{z}_{k+1}$, our objective is to compute $p(\bs{X}_{k+1}|\mathcal{Z}_{k+1})$ in~\eqref{eqn:update}.
	\end{problem} 
	
	limited studies have investigated effective method to exactly compute~\eqref{eqn:update} due to the non-linearity in the kinematics of $\bs{X}$ in~\eqref{eqn:imu_true_dynamics} and in the measurement models, like~\eqref{eqn:camera_model}. Most filtering techniques~\cite{Mourikis2007AMC,xu2022fast,Lin2022R3LIVEAR} resort to linearization at the points of estimates or iterative schemes without convergence analysis.
	To tackle the issue, we first need to set up a prerequisite that we could construct a ``good'' virtual measurement, denoted by $\bs{T}_{\bs{z}_{k+1}}\in SE(3)$, from $\bs{z}_{k+1}$ through a measurable mapping, and given a known uncertainty level~$\Sigma_{\bs{z}_{k+1}}$ it can be statistically expressed as: 
	\begin{equation} \bs{T}_{\bs{z}_{k+1}}|\bs{X}_{k+1}\sim\mathcal{N}_{RG}(\bs{X}_{k+1},\Sigma_{\bs{z}_{k+1}}). \end{equation}
	With this, \eqref{eqn:update} can be approximately solved by considering the following update:
	\begin{equation}\label{eqn:update_2} p(\bs{X}_{k+1} |\mathcal{T}_{k+1})=\frac{p(\bs{X}_{k+1}|\mathcal{T}_{k})p(\bs{T}_{\bs{z}_{k+1}}|\bs{X}_{k+1})}{p(\bs{z}_{k+1} |\mathcal{T}_k)}, \end{equation} 
 where $\mathcal{T}_{k}\triangleq\{\bs{T}_{\bs{z}_{1}},\ldots,\bs{T}_{\bs{z}_{k}}\}$.
	
	With the presence of $\mathcal{T}_{k}$ and $\bs{T}_{\bs{z}_{k+1}}$, we will
 turn to investigate 
 Problem \ref{prob:original predict problem} and Problem \ref{prob:update problem} as approximates to Problem \ref{prob:predict problem} and Problem \ref{prob:new update problem}, respectively.
 \setcounter{problem}{0}
	\renewcommand{\theproblem}{2.\arabic{problem}} 
	\begin{problem}[Prediction with Virtual Measurements]\label{prob:predict problem}
		Given  $p(\bs{X}_{k}|\mathcal{T}_{k})$ and~\eqref{eqn:imu_true_dynamics}, our objective is to predict the state $p(\bs{X}_{k+1}|\mathcal{T}_{k})$ in~\eqref{eqn:predict_imu}
	\end{problem} 
Section~\ref{sec:filtering_design_predict} provides the prediction step of EIKF for Problem~\ref{prob:predict problem}.

	\begin{problem}[ Update with Virtual Measurements]\label{prob:new update problem}
		Given the probability of \textit{a priori} $p(\bs{X}_{k+1} |\mathcal{T}_{k})$ with the vitual measurement $\bs{T}_{\bs{z}_{k+1}}$, our objective is to compute $p(\bs{X}_{k+1} |\mathcal{T}_{k+1})$ in~\eqref{eqn:update_2}.
	\end{problem}
 Section~\ref{sec:update} and Section~\ref{sec:theoretical analysis} jointly 
	provide an solution to Problems~\ref{prob:new update problem} under the condition of LEMs jointly in  .
 Let us be precise. 
 In Section~\ref{sec:update}, 
 we develop the algorithm of update step with LEMs; in Section~\ref{sec:theoretical analysis} we look into the techniques for constructing $\bs{T}_{\bs{z}_{k+1}}$ and justify its asymptotical optimallity in the MMSE sense, and then derive the \textit{a posteriori} of \eqref{eqn:update_2}.
	
	We end this section by showing that the statement that $\bs{X}_{k+1}|\mathcal{T}_k$ and $\bs{X}_{k+1}| \mathcal{T}_{k+1}$ in any $k$ can be obtained in $\mathcal{N}_{RG}$ is valid. At $k$ with $\bs{X}_{k}| \mathcal{T}_k$ being $\mathcal{N}_{RG}$, we can accurately approximate $\bs{X}_{k+1} |\mathcal{T}_k$ as a $\mathcal{N}_{RG}$, discussed in Section~\ref{sec:propagation}. When \textit{a priori} $\bs{X}_{k+1} |\mathcal{T}_k$ and the virtual measurement $\bs{T}_{\bs{z}_{k+1}}$ are available, we can also express the probability $\bs{X}_{k+1} |\mathcal{T}_{k+1}$  as a $\mathcal{N}_{RG}$, explained in Section~\ref{sec:Optimality_fusion}. These two steps demonstrate that the statement holds for each $k$, which keeps $\bs{X}_{k} |\mathcal{T}_{k}$ in $\mathcal{N}_{RG}$.

	\section{Prediction in the Filter}\label{sec:filtering_design_predict}
	In the section, the probability $p( \bs{X}_{k}|\mathcal{T}_{k})$ is preconditioned to be  $\mathcal{N}_{RG}(\hat{\bs{X}}_{k},\hat{\bs{P}}_{k})$.
	To solve Problem~\ref{prob:predict problem}, we first introduce the uncertainty to states in Section~\ref{sec:invariant error}. Moreover, we aim to derive the uncertainty propagation  in Section~\ref{sec:propagation}.
	\subsection{Associating uncertainty to states }\label{sec:invariant error}
	Let us first introduce the concept of invariant error following the work of \cite{Barrau}. Consider two trajectories $\bs{X} $ and $\bar{\bs{X}} $. The left- and right-invariant errors of two trajectories are defined by:
	\begin{align*}
		\textbf{(left-invariant error):}\quad&\bar{\bs{X}} ^{-1}\bs{X} \quad\text{or}\quad \bs{X} ^{-1}\bar{\bs{X}} ,\\
		\textbf{(right-invariant error):}\quad&\bs{X} \bar{\bs{X}} ^{-1}\quad\text{or}\quad \bar{\bs{X}} \bs{X} ^{-1}.
	\end{align*}
	% In our paper, we will further build upon the concept of invariant error, starting from the uncertainty of estimates and aiming to derive more accurate covariance, different from the InEKF approach proposed by~\cite{Barrau}.
	
	Consider the ``dummy'' system based on the IMU measurements~\eqref{eqn:IMU measurement} with the same mode of ~\eqref{eqn:IMU_state_dynamics}. The continuous-time dynamics of the estimated states $\bar{\bs{X}} $ writes:
	\begin{align*}
		\dot{\bar{\bs{R}}}_{I}&=\bar{\bs{R}}_{I}\left(\bs{\omega}_{m}-\bar{\bs{b}}_g\right)^\wedge,
		\dot{\bar{\bs{v}}}_{I}=\bar{\bs{R}}_{I}\left(\bs{a}_m-\bar{\bs{b}}_a\right)+\bs{g},\nonumber\\
		\dot{\bar{\bs{p}}}_{I}&=\bar{\bs{v}}_{I},
	\end{align*}
	where $\bar{\bs{b}}_g$ and $\bar{\bs{b}}_a$ represent the estimated biases for the IMU measurements and are exogenous to the ``dummy'' system. We briefly discuss the estimation of biases in~\eqref{eqn: IMU_augment_state_propagation}. The dynamics of $\bar{\bs{X}} $  is also expressed as a matrix differential equation:
	\begin{align}
		\dot{\bar{\bs{X}}}  =\bar{\bs{X}}  \bs{v}_b+\bs{v}_g\bar{\bs{X}}  +\bs{M}\bar{\bs{X}}  \bs{N}+\bar{\bs{X}}  \bs{n} ^\wedge,\label{eqn:IMU_estimator}
	\end{align}
	with $\bs{n} \triangleq\begin{bmatrix}
		(\bs{n}_{g}-\tilde{\bs{b}}_{g})^\top &\bs{0}_{1\times 3}&(\bs{n}_a-\tilde{\bs{b}}_a)^\top
	\end{bmatrix}^\top$ and $\tilde{\bs{b}}_g\triangleq\bar{\bs{b}}_g-\bs{b}_g$, $\tilde{\bs{b}}_a\triangleq\bar{\bs{b}}_a-\bs{b}_a$.
	
	Let $\bar{\eta}$ be the right-invariant error from $\bar{\bs{X}}$ to $\bs{X}$, which is expressed as $ \bar{\eta} \triangleq\bs{X} \bar{\bs{X}} ^{-1}. $ The Lie exponential coordinate of $\bar{\eta}$, denoted as $\bar{\xi}$, is the logarithm of $\bs{X} \bar{\bs{X}} ^{-1}$: \begin{equation}\label{eqn:prior definition}
		\bar{\xi} \triangleq \log{(\bs{X} \bar{\bs{X}} ^{-1})}. 
	\end{equation}
	
	We focus on the right-invariant error, since $\bs{X}\sim\mathcal{N}_{RG}$ is equivalent to  $\bar{\xi}\sim\mathcal{N}$, when the covariance is small as explained in Section~\ref{sec:intro_CGN} . We next look into the evolution of $\bar{\xi}$ during the prediction step.

	%	In addition, we also estimate the biases of the IMU and the time offset, which cannot be combined into the Lie group directly. We define the augmented nonlinear errors that stack up estimated the time-offset error $\tilde{t}_{IC}$ and the estimated bias error $\tilde{\bs{b}}_{\bs{g}}, \tilde{\bs{b}}_{\bs{a}}$ in the nonlinear error vector $\tilde{\xi}\in\mathbb{R}^{22}$
	%	\begin{equation}
		%		\tilde{\xi}=\begin{bmatrix}
			%			\xi_I^\top&\tilde{\bs{b}}_{g}&\tilde{\bs{b}}_{a}&\xi_{IC}^\top&\tilde{t}_{IC}
			%		\end{bmatrix}^\top.
		%	\end{equation}
	%	Although we use the augmented nonlinear errors, the ``imperfect'' InEKF maintain the symmetry of the Lie group, which maintains the observability of the original system. 
	\subsection{ Statistics of associating uncertainty}\label{sec:propagation}
	Considering the state $\bs{X} $ with kinematics~\eqref{eqn:imu_true_dynamics} and an estimator $\bar{\bs{X}} $ with kinematics~\eqref{eqn:IMU_estimator}, we obtain the dynamics of their right-invariant errors as follows:
	\begin{align*}
		\dot{\bar{\eta}} &=\bs{v}_g\bar{\eta} -\bar{\eta} \bs{v}_g+\bs{M}\bar{\eta} \bs{N}-\bar{\eta}\bar{\bs{X}}\bs{n}^\wedge\bar{\bs{X}}^{-1}.
	\end{align*}
	The dynamics of its exponential coordinate has been studied in~\cite{Xinghan2022,li2023errors} and is covered in the following.	\begin{lemma}[Dynamics of uncertainty]\label{lemma:dynamic_xi}
		The continuous-time dynamics of the Lie exponential coordinate $\bar{\xi} $ of $\bar{\eta} $ can be written in the form of a non-linear differential equation with linear drift:
		\begin{equation}\label{eqn:extended_pose_error_dynamics}
			\dot{\bar{\xi} }=\bs{A} \bar{\xi} -\dexp_{\bar{\xi} }^{-1}\Ad_{\bar{\bs{X}} }\bs{n},
			\bs{A} \triangleq\begin{bmatrix}\bs{0}_3&\bs{0}_3&\bs{0}_3\\
				\bs{0}_3&\bs{0}_3&\bs{I}_3\\
				\bs{g}^{\wedge}&\bs{0}_3&\bs{0}_3
			\end{bmatrix} ,
		\end{equation}
		with $\Ad_{\bar{\bs{X}}}$ defined in~\eqref{eqn:definition_of_adjoint}.
	\end{lemma}
	This result is of great significance for the prediction step, where in the deterministic case, the dynamics of $\bar{\xi}$ is linear. In the stochastic case with $\bs{n}$, the term including $\bs{n}$ appears to be uncorrelated and independent  with $\bar{\xi}$ since $\bar{\xi}$ is small and $\dexp_{\bar{\xi} }^{-1}\approx\bs{I}$. Therefore, $\bar{\xi}$ in~\eqref{eqn:extended_pose_error_dynamics} with the initial condition being a Gaussian can remain in a Gaussian state.
	
	However, the term involving $\dexp_{\bar{\xi} }^{-1}$ presents a challenge in analyzing the exact propagation of the covariance. Next, we continue to obtain a more accurate covariance of the error by approximating the higher-order terms of $\dexp_{\bar{\xi} }^{-1}$ in~\eqref{eqn:extended_pose_error_dynamics}. This approximation also allows us to continue treating~\eqref{eqn:extended_pose_error_dynamics} as a linear ordinary differential equation for practical purposes. 
	
	Let $\bs{n}'\triangleq\Ad_{\bar{\bs{X}}} \bs{n} $. Since the noises $\bs{n}_g,\bs{n}_a,\tilde{\bs{b}}_g,\tilde{\bs{b}}_a$ are independent, the covariance can be readily computed and written into:
	\begin{align}
		\bs{n}'&\sim\mathcal{N}(\bs{0}_{9\times1},\Sigma_{\bs{n}'}), \label{eqn:n'_I}\\
		\Sigma_{\bs{n}'}&\triangleq \Ad_{\bar{\bs{X}} }\text{diag}\left\lbrace \sigma^2_{g}\bs{I}_3+\bs{P}_{\tilde{b}_g},\bs{0}_{3},\sigma^2_{a}\bs{I}_3+\bs{P}_{\tilde{b}_a}\right\rbrace\Ad_{\bar{\bs{X}} }^\top\nonumber,
	\end{align}
	where $\bs{P}_{\tilde{\bs{b}}_g}\triangleq\mathbb{E}\left[\tilde{\bs{b}}_g\tilde{\bs{b}}^\top_g\right]$ and $\bs{P}_{\tilde{\bs{b}}_a}\triangleq\mathbb{E}\left[\tilde{\bs{b}}_a\tilde{\bs{b}}^\top_a\right]$ here are assumed to be known. Because $\dexp_{\bar{\xi}}^{-1}\approx\bs{I}_9-\frac{1}{2}\ad_{\bar{\xi} }+\frac{1}{6}\ad^2_{\bar{\xi} }$ as given in~\eqref{eqn:left_jacobian},
	let $\bs{n}_{\operatorname{4th}}$ be a new variable, which is equal to $(\bs{I}_9-\frac{1}{2}\ad_{\bar{\xi} }+\frac{1}{6}\ad^2_{\bar{\xi} })\bs{n}'$ and is assumed to be independent of ${\bar{\xi}}$ for simplicity in calculating the covariance. We can approximate~\eqref{eqn:extended_pose_error_dynamics}, which writes:
	\begin{align}\label{eqn:linearity_of_xi}
		\dot{{\bar{\xi}} }=\bs{A} \bar{\xi} -\bs{n}_{\operatorname{4th}},
	\end{align}
	with $\bs{n}_{\operatorname{4th} }\sim\mathcal{N}(\bs{0}_{9\times1},\Sigma_{\operatorname{4th}})$.  Here, $\Sigma_{\operatorname{4th}}$ can be calculated in the following procedure:
	\begin{align*}
		&\Sigma_{4th}=\Sigma_{\bs{n}' }+\frac{1}{6}\left(\bs{D}\Sigma_{\bs{n}' }+\Sigma_{\bs{n}' }\bs{D}^\top \right)+\frac{1}{4}\bs{B}, \\
		&\bs{D}\triangleq\mathbb{E}\left[\ad_{\bar{\xi} }^2\right]=\begin{bmatrix}
			\langle\langle \bar{\bs{P}}_{\theta,\theta}\rangle\rangle&\bs{0}_{3}&\bs{0}_{3}\\
			\langle\langle \bar{\bs{P}}_{\theta, p}+\bar{\bs{P}}_{\theta, p}^\top\rangle\rangle&\langle\langle \bar{\bs{P}}_{\theta,\theta}\rangle\rangle&\bs{0}_{3}\\
			\langle\langle \bar{\bs{P}}_{\theta, v}+\bar{\bs{P}}_{\theta, v}^\top\rangle\rangle&\bs{0}_3&\langle\langle \bar{\bs{P}}_{\theta,\theta}\rangle\rangle
		\end{bmatrix},\\
		&\bs{B}\triangleq\mathbb{E}\left[\ad_{\bar{\xi} }\bs{n} '\bs{n} '^\top\ad_{\bar{\xi} }^\top\right]=\begin{bmatrix}
			\bs{B}_{\theta,\theta}&\bs{B}^\top_{p,\theta}&\bs{B}^\top_{v,\theta}\\
			\bs{B}_{p,\theta}&\bs{B}_{p,p}&\bs{0}_3\\
			\bs{B}_{v,\theta}&\bs{0}_3&\bs{B}_{v,v}
		\end{bmatrix},\\
		&\bs{B}_{\theta,\theta}=	\langle\langle \bar{\bs{P}}_{\theta,\theta},\Sigma_{\theta,\theta}\rangle\rangle,\quad \bs{B}_{i,\theta}=\langle\langle\bar{\bs{P}}_{\theta,\theta},\Sigma^\top_{i,\theta}\rangle\rangle+\langle\langle\bar{\bs{P}}^\top_{\theta, i},\Sigma_{\theta,\theta}\rangle\rangle,\\
		&\bs{B}_{i,i}=\langle\langle\bar{\bs{P}}_{\theta,\theta},\Sigma_{i,i}\rangle\rangle+\langle\langle\bar{\bs{P}}^\top_{i,\theta},\Sigma_{i,\theta}\rangle\rangle+\langle\langle\bar{\bs{P}}_{i,\theta},\Sigma^\top_{i,\theta}\rangle\rangle\\
		&+\langle\langle\bar{\bs{P}}_{i,i},\Sigma_{\theta,\theta}\rangle\rangle,
	\end{align*}	
	where $i\in\left\lbrace\bs{p},\bs{v} \right\rbrace $ and $\bar{\bs{P}} \triangleq\mathbb{E}{\left[\bar{\xi} \bar{\xi}^\top \right]}=\begin{bmatrix}
		\bar{\bs{P}}_{\theta,\theta}&\bar{\bs{P}}_{\theta,p}&\bar{\bar{\bs{P}}}_{\theta,v}\\
		\bar{\bs{P}}_{p,\theta}&\bar{\bs{P}}_{p,p}&\bar{\bs{P}}_{p,v}\\
		\bar{\bs{P}}_{v,\theta}&\bar{\bs{P}}_{v,p}&\bar{\bs{P}}_{v,v}
	\end{bmatrix}$,
	$\Sigma_{\bs{n}' }=\begin{bmatrix}
		\Sigma_{\theta,\theta}&\Sigma_{\theta,p}&\Sigma_{\theta,v}\\
		\Sigma_{p,\theta}&\Sigma_{p,p}&\Sigma_{p,v}\\
		\Sigma_{v,\theta}&\Sigma_{v,p}&\Sigma_{v,v}
	\end{bmatrix}$. The calculus techniques for $D$ and $B$ can also be found in~\cite{Barfoot2014AssociatingUW,Brossard2020AssociatingUT} and the values of $\bar{\bs{P}}$ are provided and can be approximated from the preceding time step in the discrete-time system.
	\begin{remark}[State estimation with estimated biases]\label{rm: estimated bais propagation}
		In our research, we view the bias error $\tilde{\bs{b}}_{\cdot} \triangleq [\tilde{\bs b}_g^\top~\tilde{\bs b}_a^\top]^\top$ as one component of the noise $\bs{n}$ that affects the dynamics in \eqref{eqn:IMU_estimator}. When we view the bias $\bar{\bs{b}}_{\cdot}$ as the state to be estimated in the filter, the kinematics of the estimated biases can be expressed as follows:
		\begin{equation*}
			\dot{\bar{\bs{b}}}_g=\bs{0}_{3\times 1},\quad\dot{\bar{\bs{b}}}_a=\bs{0}_{3\times1},
		\end{equation*}
		and the augmented state error can be expressed by adding the vector differences $\tilde{\bs{b}}_{\cdot}$:
		\begin{align}\label{eqn: IMU_augment_state_propagation}
			\begin{bmatrix}
				\dot{\bar{\xi}}\\
				\dot{\tilde{\bs{b}}}_{\cdot}
			\end{bmatrix}=\begin{bmatrix}
				\bs{A} 
				& \bs{0}_{9\times 6} \\
				\bs{0}_{6\times9}  &
				\bs{0}_{6}
			\end{bmatrix}
			\begin{bmatrix}
				\bar{\xi}\\
				\tilde{\bs{b}}_{\cdot}
			\end{bmatrix}+\begin{bmatrix}
				\bs{I}_9
				& \bs{0}_{9\times 6} \\
				\bs{0}_{6\times 9}  &
				\bs{I}_6
			\end{bmatrix}
			\bs{n}_{\rm imu},
		\end{align}	
		where $\bs{n}_{\rm imu}\triangleq\begin{bmatrix}
			\bs{n}_{\operatorname{4th}}^\top&\bs{n}^\top_{ bg}&\bs{n}^\top_{ ba}
		\end{bmatrix}^\top\sim\mathcal{N}(\bs{0}_{15\times 1},\bs{Q}),$ and $\bs{Q}\triangleq\operatorname{diag}\left\lbrace\Sigma_{4th},\sigma^2_{bg}\bs{I}_3,\sigma^2_{ba}\bs{I}_3 \right\rbrace $\footnote{
			The augmentation technique is also used for developing 
			``imperfect'' InEKF in~\cite{barrau2015non}. The introduction of additional vector difference
			coins the term of ``imperfect'' InEKF, since it sacrifices all the properties of the InEKF, see~\cite{barrau2015non} for details.}. The covariance $\bar{\bs{P}}_{\tilde{\bs{b}}_g}$ and $\bar{\bs{P}}_{\tilde{\bs{b}}_a}$ in~\eqref{eqn:n'_I} can also be estimated at the prediction step.
	\end{remark}
	
	\begin{remark}[Existing works in $SE_2(3)$]
		Brossard et al.\cite{Brossard2020AssociatingUT} provided IMU propagation in $SE_2(3)$  from the perspective of discrete-time kinematics, while our work approaches this from the perspective of continuous-time kinematics. Both of the algorithms excel in achieving more accurate estimates than the one obtained by the first-order approximation~\cite{Barrau}. 
	\end{remark}
	The analysis of probability $p(\bs{X}_{k+1}|\mathcal{T}_k)$ can leverage $\bar{\bs{X}}_{k+1}$. To this end,  we consider time $k$ 
	and let $\hat{\xi}_k$ be defined by $\hat{\xi}_k\triangleq\log (\bs{X}_k\hat{\bs{X}}^{-1}_k)$ which is the initial condition of~\eqref{eqn:linearity_of_xi}  and implement the following discrete-time dynamics: 
	\begin{align*}
		&\bar{\xi}_{k+1}=\Phi(kT,(k+1)T)\hat{\xi}_k+\bs{n}_{\operatorname{4th,d}}, \bs{n}_{\operatorname{4th,d}}\sim\mathcal{N}(\bs{0}_{9\times 1},\bs{Q}_d),\\
		&\bs{Q}_d\triangleq\int_{kT}^{(k+1)T} {\Phi}\left(\tau,(k+1)T\right) \Sigma_{\operatorname{4th}} \Phi^{T}\left(\tau,(k+1)T\right) \mathrm{d} \tau,\\
		&\Phi(t_1,t_2)\triangleq\exp{(\bs{A}(t_2-t_1))}.
	\end{align*}
	The mean of $p(\bar{\xi}_{k+1}|\mathcal{Z}_{k})$ then writes:
	\begin{align*}
		\mathbb E[\bar{\xi}_{k+1}|\mathcal{Z}_{k}]&=\int \bar{\xi}_{k+1} p(\bar{\xi}_{k+1}|\mathcal{T}_k)\d{\bar{\xi}_{k+1}}\\
		&\overset{\eqref{eqn:predict_imu}}{=}\int(\int \bar{\xi}_{k+1} p(\bar{\xi}_{k+1}|\hat{\xi}_k)\d{\bar{\xi}_{k+1}})p(\hat{\xi}_{k}|\mathcal{T}_k)\d{\hat{\xi}_{k}}\\
		&=\int \mathbb E[\bar{\xi}_{k+1}|\hat{\xi}_{k}]p(\hat{\xi}_{k}|\mathcal{T}_k)\d{\hat{\xi}_{k}}\\
		&\overset{(a)}{=}\Phi((k+1)T,kT)\int \hat{\xi}_{k} p(\hat{\xi}_{k}|\mathcal{T}_k)\d{\hat{\xi}_{k}}{=}\bs{0}_{9\times 1}.
	\end{align*}
	The equality (a) holds because $\bs{n}_{\operatorname{4th,d}}$ is a Gaussian noise independent of $\hat{\xi}_{k}$ and $\bar{\xi}_{k+1}|\hat{\xi}_{k}\sim \mathcal{N}(\Phi(\cdot,\cdot)\hat{\xi}_{k},\bs{Q}_d)$. The associated covariance $\bar{\bs{P}}_{k+1}$ is:
	\begin{align}
		\bar{\bs{P}}_{k+1}=\Phi(kT,(k+1)T)\hat{\bs{P}}_{k}\Phi(kT,(k+1)T)^\top+\bs{Q}_d.\label{eqn:cov_predict}
	\end{align}

	The end condition at time $k+1$ of~\eqref{eqn:IMU_estimator}, initialized with $\hat{\bs{X}}_{k}$, is denoted as $\bar{\bs{X}}_{k+1}$ and can be calculated using standard integration methods in the literature (such as (30) in~\cite{Forster2015OnManifoldPF}), which is skipped here. With the uncertainty~$\bar{\xi}_{k+1}|\mathcal{T}_{k}\sim\mathcal{N}(\bs{0}_{9\times 1},\bar{\bs{P}}_{k+1})$ and the estimates $\bar{\bs{X}}_{k+1}$, we can conclude that $\bs{X}_{k+1}|\mathcal{T}_{k}\sim\mathcal{N}_{RG}(\bar{\bs{X}}_{k+1},\bar{\bs{P}}_{k+1})$.

	\section{Update in the filter}\label{sec:update}
	In this section, we will explore the update step. We provide the GN method within the Lie groups framework and derive the formulation of the KF methodology in Section~\ref{sec:Iterative gaussian newton}. Secondly, we introduce a \textit{$\sqrt{n}$-consistent estimate} of each present state as the initial condition for the GN method, as described in Section~\ref{sec:initial condition of the GN}.
	
	For brevity in notation, we just discuss the update step and therefore omit the subscript $k+1$ in the section, and the predicted states have been assumed to follow $\mathcal{N}_{RG}(\bar{\bs{X}} , \bar{\bs{P}} )$ through the analysis in the last section. 
	\subsection{Iterated  update of filtering through the LGN method}~\label{sec:Iterative gaussian newton}
The section primarily introduces the iterated update through an LGN method under a filtering framework. The initialization of the update can be given flexibility. 
When the initialization is set to be the mean of the predicted distribution, the iterative update boils down to an extensively studied way in the literature \cite{Bell93,Bourmaud2016FromIO} of computing the mean and covariance of a concentrated Gaussian distribution to approximately solve Problem~\ref{prob:update problem}. The maximum a posteriori~(MAP) is a method commonly used for solving the mean.
When the iterated update is initialized with a statistically consistent pose, we will prove that
the distribution obtained from the iterated update converges asymptotically to the solution to Problem~\ref{prob:new update problem}. This is left to be discussed in the next section. The proof of their convergence is also provided in it.

With either the camera or LiDAR as an environmental sensor, given 
the concentrated Gaussian of the prior distribution and the Gaussian of the sampling distribution, first consider
the MAP problem of Problem~\ref{prob:update problem}, expressed as follows:
	\begin{align}\label{eqn:prob:update problem}
		\hat{\bs{X}}_{\MAP} =\mathop{\arg\min}_{\bs{X} \in SE_2(3)} 
		\left\|\bar{\xi}\right\|^2_{\bar{\bs{P}}} +\left\| r_{j}(\bs{T})\right\|^2_{\Sigma_j} ,
	\end{align}
 where $\bar{\xi}$ is defined in~\eqref{eqn:prior definition}, $j\in\left\lbrace{C},{L}\right\rbrace$, and  $r_j$ and $\Sigma_j$ are defined in~\eqref{eqn:residuals} and \eqref{eqn:covariance_residual}, respectively. The uncertainty, denoted as $\hat{\xi}$, of update estimates $\hat{\bs{X}}$ is defined as $\hat{\xi}\triangleq\log(\bs{X}\hat{\bs{X}}^{-1})$.
	
We now take a look the LGN approach for the aforementioned problem without placing a significant emphasis on the issue of convergence. 
Let $\mu^{(l)}$ be an estimate of $\bs{X}$ at the $l$th iteration and $\bs{T}^{(l)}$ be the pose of $\mu^{(l)}$. Define $\delta^{(l)}\triangleq\log{(\mu^{(l)} \bar{\bs{X}} ^{-1})}$ as the left increment to $\bar{\bs{X}} $ and $\tilde{\delta}_{l,l+1}\triangleq\log{(\mu^{(l+1)}\mu^{(l)-1})}$ for short. Utilizing the linearization in Section~\ref{sec:description of the LGN}, the next $l+1$th iteration is expressed as:
	\begin{align*}
		&\tilde{\delta}_{l,l+1}=\arg\min_{\delta\in\mathbb{R}^9}\left\|	\bs{J}_{\mu^{(l)}}\delta+\delta^{(l)} \right\|_{\bar{\bs{P}}}^2+ \left\|	r_j(\bs{T}^{(l)})-\bs{H}_{\mu^{(l)},j}\delta \right\|_{\Sigma_j}^2,\nonumber\\
		&\bs{J}_{\mu^{(l)}}\triangleq\frac{\partial \bar{\xi}}{\partial{\bs{X} }}|_{\mu^{(l)}}\overset{(b)}{=}\dexp_{\delta^{(l)}}^{-1}, \bs{H}_{\mu^{(l)},j}\triangleq\begin{bmatrix}-\frac{\partial r_{j}(\bs{T} )}{\partial{\bs{T} }}|_{\bs{T}^{(l)}}
		    &\bs{0}_{n_j\times 3}
		\end{bmatrix}\nonumber,
	\end{align*}
	where the Jacobians $\bs{J}_{\mu^{(l)}}$ and $\bs{H}_{\mu^{(l)},j}$ for $j\in\left\lbrace {C},{L} \right\rbrace $, evaluated at $\mu^{(l)}$, are defined in~\eqref{eqn:matrix_differential}. The value of $n_j$ is $2n$ for $j=C$ and $n$ for $j=L$. The equality $(b)$ is based on (i) in Lemma~\ref{lemma: compound of two matrix exponentials}. Further details of derivation on the sensor $j$ can be found in Appendix~\ref{apx:jacobians_camera_LIDAR}. The Jacobian of the negative residual in $\bs{H}_{\mu^{(l)}}$ is defined in accordance with the symbols commonly used in the majority of filtering literature.
	
The GN method then proceeds to compute $\mu^{(l+1)}$ as follows:
	\begin{align}~\label{eqn:LGN method of problem 1}
\mu^{(l+1)}&=\exp{(\tilde{\delta}_{l,l+1})}\mu^{(l)},\end{align}
  with
  \begin{align*}
		\tilde{\delta}_{l,l+1}&= \bs{F}_{\mu^{(l)}}\left(\bs{H}^\top_{\mu^{(l)},j}\Sigma^{-1}_jr_{j}(\bs{T}^{(l)})-\bs{J}_{\mu^{(l)}}^\top \bar{\bs{P}}^{-1}\delta^{(l)}\right),\nonumber\\
		\bs{F}_{\mu^{(l)}}&\triangleq\left(\bs{J}_{\mu^{(l)}}^\top \bar{\bs{P}} ^{-1}\bs{J}_{\mu^{(l)}}+\bs{H}^\top_{\mu^{(l)},j}\Sigma^{-1}_j\bs{H}_{\mu^{(l)},j} \right)^{-1}.
	\end{align*} 
To initialize $\mu^{(0)}$, we can use the predicted state $\bar{\bs{X}}$ or an estimate from environmental measurements. The choice can be based on which of the two efficiently gives the most accurate estimates in practice. More facts on the two different initialization options will be discussed in Remark~\ref{remark:refinement_of_InEKF} and in Sections~\ref{sec:initial condition of the GN} and \ref{sec:theoretical analysis}. 
	
	The KF methodology~\cite{Sorenson1970LeastsquaresEF} offers practical computational efficiency. We will convert the LGN iteration into the KF formulation. In particular, The KF gain $\bs{K}^{(l+1)}_j$ for the sensor $j$ is derived as follows: 
	\begin{align}
		\bs{K}^{(l+1)}_j&\triangleq \bs{J}_{\mu^{(l)}} \bs{F}_{\mu^{(l)}} \bs{H}^\top_{\mu^{(l)},j} \Sigma^{-1}_j = \bar{\bs{P}} \bs{J}_{\mu^{(l)}}^{-1\top} \bs{H}^\top_{\mu^{(l)},j}\bs{S}_j^{-1},\nonumber\\ \bs{S}_j&\triangleq \bs{H}_{\mu^{(l)},j} \bs{J}_{\mu^{(l)}}^{-1 } \bar{\bs{P}} \bs{J}_{\mu^{(l)}}^{-1\top} \bs{H}^\top_{\mu^{(l)},j}+ \Sigma_j\label{eqn:LGN kalman filter gain} . 
	\end{align} 
	The update in terms of KF formulation is then given by 	\begin{equation}\label{eqn: update form} \mu^{(l+1)}=\exp(\bs{K}^{(l+1)}_j( r_j(\bs{T}^{(l)})+ \bs{H}_{\mu^{(l)},j}\delta^{(l)}))\bar{\bs{X}}. \end{equation} 
More details on how to convert and obtain the KF gain are summarized in Appendix \ref{apx:Kalman Filter gain}.
	
When the iterated update~(see Algorithm~\ref{alg:Iterated_invariant_kalman_filter}) terminates at the $\bar{l}$th iteration, $\mu^{(\bar{l})}$ is considered to be the mean of $p(\bs{X}_{k+1}| \mathcal {T}_{k+1})$. The covariance matrix, which is represented as $\hat{\bs{P}}$, can be carried out by the Laplace approximation~\cite{Bell93} assuming that ${\mu^{(\bar{l})}}$ is sufficiently close to $\bs{X}$. The approximation is given by the following: \begin{align}
\hat{\bs{P}} &\triangleq \mathbb{E}(\hat{\xi}\hat{\xi}^\top) \approx \bs{F}^{-1}_{\mu^{(\bar{l})}} \nonumber \\
&= \bs{J}^{-1}_{\mu^{(\bar{l})}}(\bs{I} - \bs{K}^{(\bar{l}+1)}_{j}\bs{H}_{\mu^{(\bar{l})},j}\bs{J}^{-1}_{\mu^{(\bar{l})}})\bar{\bs{P}} \bs{J}^{-1\top}_{\mu^{(\bar{l})}}. \label{eqn:P update}
\end{align} The final equality of \eqref{eqn:P update} is obtained by the matrix inversion lemma.
	
	\begin{remark}[Refinement of InEKF]\label{remark:refinement_of_InEKF}
		When the initial value $\mu^{(0)}$ is set to $\bar{\bs{X}} $ and only one iteration is performed, the update formula~\eqref{eqn: update form} is identical to that used in InEKK. Since
		the iterative process and adaptive initialization allow for potential improvement in accuracy and robustness of the estimates, our proposed iterated update of filtering can be seen as a more refined version of InEKF~\cite{Barrau}. 
	\end{remark}

	\begin{algorithm}[h]
		\caption{Framework of iterated update of filtering}\label{alg:Iterated_invariant_kalman_filter}
		\KwIn{$\mu^{(0)}$, $\bar{\bs{P}}$, $\bs{z}_C$ or $\bs{z}_{L}$, $\Sigma_C$ or $\Sigma_L$, $l_{\rm max}$, $\tau$}
		\KwOut{$\hat{\bs{X}}$, $\hat{\bs{P}} $}		
            \textbf{Step 1}: set $l=0$\;
            \textbf{Step 2}: iterate until the end condition is satisfied\
		\begin{enumerate}
			\item compute $r_j(\mu^{(l)})$ via~\eqref{eqn:residuals},
			\item calculate $\bs{H}_{\mu^{(l)}}$ via~\eqref{eqn:camera jacobian matrix} and \eqref{eqn:LiDAR measurement equation},
			\item compute  $\bs{K}^{(l+1)}_j$ via~\eqref{eqn:LGN kalman filter gain} and  $\mu^{(l+1)}$ via~\eqref{eqn: update form},
   \item $l=l+1$\;
		\end{enumerate}	
		\textbf{Step 3}: calculate the covariance $\hat{\bs{P}}$ via~\eqref{eqn:P update}\; 
		\textbf{End Condition}: $|c^{({l})}-c^{({l}-1)}|\leq\tau$ or $l=l_{\rm max}$, where the cost $c^{({l})}$ is given as
		\begin{equation*}
			c^{(l)}=\delta^{(l)} (\bs{J}^\top_{\mu^{(l)}}{\bar{\bs{P}} }^{-1}\bs{J}_{\mu^{(l)}})^{-1}\delta^{(l)\top}+r_j(\mu^{(l)}){\Sigma}^{-1}_jr_j(\mu^{(l)})^{\top}.
		\end{equation*}
	\end{algorithm}
	
	\subsection{$\sqrt{n}$-consistent pose as the initial value for camera and LiDAR}\label{sec:initial condition of the GN}
	We will outline necessary steps to obtain an initial value from a statistical theory standpoint, as an alternative to $\bar{\bs{X}}$ in InEKF~(see Remark~\ref{remark:refinement_of_InEKF}). The initial value is selected as a $\sqrt{n}$-consistent pose derived from~\eqref{eqn:camera_model} or \eqref{eqn:LiDAR_model} where $n$ denotes the quantity of environmental measurements. A comprehensive theoretical analysis of this method, which includes its asymptotic efficiency and optimality, will be detailed in Section~\ref{sec:theoretical analysis}.
	
	\subsubsection{A $\sqrt{n}$-consistent pose estimator from camera measurements}
	We utilize the method proposed in \cite{zeng2022cpnp}. This method constructs linear equations from the original projection model~\eqref{eqn:camera_model} via modification and elimination of the original model, based on which a closed-form LS solution is obtained without considering the constraint $SE(3)$. The closed-form solution mitigates the impact of measurement noise and is projected onto $SE(3)$, yielding a $\sqrt{n}$-consistent pose estimate.
	
	In what follows, we will rephrase the method in our notation. Let $\bs{p}_{f_i} \in \mathbb R^3$ and ${}^{C}\bs{p}^\top_{f_i} \in \mathbb R^3$ denote the $i$th feature in $\{\mathcal{W}\}$ and $\{\mathcal{C}\}$, respectively. The projection of the $i$th feature in the image plane is denoted as $\bs{q}_{i}=[u_i~v_i]^\top \in \mathbb R^2$.
	We split $\bs{R}^\top_C$ and $-\bs{R}^\top_C\bs{p}_C$ as
	$\bs{R}^\top_C=[\bs{r}_1 ~\bs{r}_2~ \bs{r}_3]^\top$ and $-\bs{R}^\top_C\bs{p}_C=[p_{C,1}~p_{C,2}~p_{C,3}]^\top$.
	There exists an ambiguity of scale in the projected mapping $h_{C}$, that is, $\frac{^{C}p_{f_x}}{^{C}p_{f_z}}\equiv\frac{\alpha ^{C}p_{f_x}}{\alpha ^{C}p_{f_z}}$. Therefore, we can set the constraint $\frac{1}{\alpha} \triangleq\frac{1}{n}\sum_{i=1}^{n}{}^C{p}_{f_z}$ to eliminate one variable ($p_{C,3}$) and normalize the vector $\vec{\bs{x}}_C=\alpha\begin{bmatrix}
		\bs{r}^\top_3 &\bs{r}^\top_1&p_{C,1}&\bs{r}^\top_2&p_{C,2}
	\end{bmatrix}\in\mathbb{R}^{11}$. This allows us to optimize an 11-dimensional vector, and $\alpha$ can be determined by ensuring $\left|\bs{R}_C \right|=1$. We define $\bar{\bs{p}}_f=\frac{1}{n}\sum_{i=1}^{n}\bs{p}_{f_i}$. Now we have the following linear form of~\eqref{eqn:camera_model}:
	\begin{equation}\label{eqn:linear equation camera}
		\bs{b}=\bs{A}_C\vec{\bs{x}}_C+\vec{\bs{n}}_C,
	\end{equation}
	where 
	$$
	\bs{A}_C=\begin{bmatrix}
		-u_1(\bs{p}_{f_1}-\bar{\bs{p}}_f)^\top &f_x\bs{p}^\top_{f_1}&f_x&\bs{0}_{1\times4}\\
		-v_1(\bs{p}_{f_1}-\bar{\bs{p}}_f)^\top &\bs{0}_{1\times4}&f_y\bs{p}^\top_{f_1}&f_y\\
		\vdots\\
		-u_n(\bs{p}_{f_n}-\bar{\bs{p}}_f)^\top &f_x\bs{p}^\top_{f_n}&f_x&\bs{0}_{1\times4}\\
		-v_n(\bs{p}_{f_n}-\bar{\bs{p}}_f)^\top &\bs{0}_{1\times4}&f_y\bs{p}^\top_{f_n}&f_y
	\end{bmatrix},
	$$
	$f_x$ and $f_y$ are focal lengths of the camera, $\bs{b}=\begin{bmatrix}
		u_1&
		v_1
		&\cdots&{u}_n&v_n
	\end{bmatrix}^\top\in\mathbb{R}^{2n}$, and $\vec{\bs{n}}_C$ is the modified noise term. 
	
	The LS solution to ~\eqref{eqn:linear equation camera} is given by $({\bm A}_C^\top {\bm A}_C)^{-1} {\bm A}_C^\top {\bm b}$. Since the regressor ${\bm A}_C$ contains noisy measurements, it is correlated with the regressand $\bm b$, which in general may introduce bias into the LS solution and makes it inconsistent according to Definition~\ref{def:consistency}. To address this issue, it is essential to estimate the variance of measurement noise and  conduct bias elimination based on the estimated noise variance to achieve a consistent pose estimate.
	
	For more detailed information about obtaining a consistent estimate, denoted as $\hat{\sigma}_C$, of the noise variance, see \cite[Appendix A]{zeng2022cpnp}. Define the following matrix
	\begin{align}\label{eqn:coefficient_of_noise}
		\bs{G}_C=\begin{bmatrix}
			-(\bs{p}_{f_1}-\bar{\bs{p}}_f)^\top &\bs{0}_{1\times8}\\
			\vdots &\vdots
			\\
			-(\bs{p}_{f_n}-\bar{\bs{p}}_f)^\top &\bs{0}_{1\times8}
		\end{bmatrix}\in\mathbb{R}^{2n\times11}.
	\end{align}
	The bias-eliminated solution can be written in
	\begin{equation}\label{eqn: camera consistent solution vector}
		\vec{\bs{x}}_C=(\bs{A}^\top_C\bs{A}_C-\hat{\sigma}^2_{C}\bs{G}^\top_C\bs{G}_C)^{-1}(\bs{A}^\top_C\bs{b}-\hat{\sigma}^2_{C}\bs{G}^\top_C\bs{1}_{2n\times1}).
	\end{equation}
	As the scale $\alpha$ can be determined uniquely,  we can obtain an estimate of $\bs{x}_C \triangleq {\rm vec}([\bs{R}^\top_C~-\bs{R}^\top_C\bs{p}_C])\in\mathbb{R}^{12}$, denoted as $\hat{\bs{x}}_C$, and use the SVD method~\cite{Myronenko2009OnTC} to project $\hat{\bs{x}}_C\in\mathbb{R}^{12}$ onto $SE(3)$. The resulting projection $\hat{\bs{T}}_{C,\cons}$ of  $\hat{\bs{x}}_C$ onto $SE(3)$ exhibits the following property:
	\begin{theorem}\label{thm:n-consistent solution camera}
		The estimate $\hat{\bs{T}}_{C,\cons}\in SE(3)$  via solving $$
\hat{\bs{T}}_{C,\cons}=\arg\min_{\bs{T}_C\in SE(3)}\left\| \operatorname{vec}(\begin{bmatrix}
		    \bs{R}^\top_C&-\bs{R}^\top_C\bs{p}_C
		\end{bmatrix})-\hat{\bs{x}}_C\right\|^2, 	
		$$ has the property that $\hat{\bs{T}}_{C,\cons}$ is a $\sqrt{n}$-consistent pose for $\bs{T}_C$, i.e., 
		\begin{equation*}
			\hat{\bs{T}}_{C,\cons}=\bs{T}_C+O_p(\frac{1}{\sqrt{n}}).
		\end{equation*}
	\end{theorem}
		The proof of Theorem~\ref{thm:n-consistent solution camera} is given in Appendix~\ref{apx:consistent_solution_analysis}.
	
	\subsubsection{A $\sqrt{n}$-consistent estimate from LiDAR measurement}
	By substituting~\eqref{eqn:LiDAR_model} into~\eqref{eqn: Lidar_icp_problem}, we obtain
	\begin{equation}\label{linear_lidar_eqn}
		\bm {u}_j^\top q_j = \bm {u}_j^\top {\bm R}_{L} \bs{z}_{L_j} + \bm {u}_j^\top {\bm p}_{L} - \bm {u}_j^\top {\bm R}_{L} \bm{n}_{L_j},
	\end{equation}
	since the noise is assumed to be isotropic.
	Let ${\bm x}_{L}=[{\rm vec}({\bm R}_{L})^\top~~{\bm p}_{L}^\top]^\top$. Stacking~\eqref{linear_lidar_eqn} for all measurements yields the matrix form
	\begin{equation}\label{linear_Lidar_matrix_eqn}
		{\bm b}={\bm A}_{L} {\bm x}_{L} +\vec{{\bm n}}_{L},
	\end{equation}
	where
	\begin{equation*}
		{\bm b}=\begin{bmatrix}
			\bm {u}_1^\top q_1 \\
			\vdots \\
			\bm {u}_n^\top q_n
		\end{bmatrix},~
		{\bm A}_{L}=\begin{bmatrix}
			[\bs{z}_{L_1}^\top~1] \otimes \bm {u}_1^\top  \\
			~~~~~\vdots ~~~~~~ \\
			[\bs{z}_{L_n}^\top~1] \otimes \bm {u}_n^\top
		\end{bmatrix},~\vec{{\bm n}}_L=\begin{bmatrix}
		    - \bm {u}_1^\top {\bm R}_{L} \bm{n}_{L_1}\\
      \vdots\\
      - \bm {u}_n^\top {\bm R}_{L} \bm{n}_{L_n}
		\end{bmatrix}
	\end{equation*}
	and $\vec{{\bm n}}_{L}$ is the modified noise term. Similar to the case of the camera, the regressor ${\bm A}_{L}$ in the LiDAR system includes measurement noises, leading to correlation with the regressand $\bm b$. Consequently, the LS solution is biased and lacks consistency. 
	Define
	\begin{equation} \label{Q_bar}
		\bar {\bm Q}=\text{diag}{\left( 	\frac{{\bm U} {\bm U}^\top}{n},	\frac{{\bm U} {\bm U}^\top}{n} ,	\frac{{\bm U} {\bm U}^\top}{n} ,{\bs{0} }_{3}\right)  },
	\end{equation}
	where ${\bm U}=[{\bm u}_1~\ldots~{\bm u}_n]$.
	%The covariance $\hat{\sigma}$ is estimated based on samples.  Let $\bar {\bm A}^{(l)}=[{\bm A}^{(l)}~{\bm b}]$,  ${\bm S}=\bar {\bm A}^{(l)}^\top \bar {\bm A}^{(l)}/n$, and 
	%\begin{equation*}
	%	{\bm Q}=\text{diag}{\left( 	\frac{{\bm U} {\bm U}^\top}{n},	\frac{{\bm U} {\bm U}^\top}{n} ,	\frac{{\bm U} {\bm U}^\top}{n} ,{\bm O}_{4 \times 4}\right)  }.
	%\end{equation*}
	%The following lemma gives a consistent estimate of noise variance $\sigma^{(l)}^2$.
	%\begin{lemma} \label{consistent^{(l)}idar_noise_est}
	%	The estimate $\hat \sigma^{(l)}^2=1/\lambda_{\rm max}({\bm S}^{-1}{\bm Q})$ is a $\sqrt{n}$-consistent estimate of $\sigma^{(l)}^2$. 
	%\end{lemma}
	%\begin{proof}
	%	The proof is given in Appendix~\ref{apx:estimation_of_covariance_of_the^{(l)}idar}.
	%\end{proof}
	
	Therefore, the bias-eliminated solution is given by
	\begin{equation} \label{consistent_Lidar_pose_est}
		{\bs{x}}_{L}=\left(\frac{{\bm A}_{L}^\top {\bm A}_{L}}{n}- \hat \sigma_{L}^2 \bar {\bm Q}\right)^{-1} \frac{{\bm A}_{L}^\top {\bm b}}{n},
	\end{equation}
	where $\hat{\sigma}_L^2$ is a consistent estimate of the variance of $\bs{n}_{L}$. The estimation method closely resembles the one employed in the camera model, as detailed in \cite[Appendix A]{zeng2022cpnp}. Once $\bs{x}_{L}$ is obtained, a $\sqrt{n}$-consistent pose $\hat{\bs{T}}_{L,\cons}$ can also be obtained, exhibiting a resemblance to the result in Theorem~\ref{thm:n-consistent solution camera}. For consistency analysis in the LiDAR case, please refer to Appendix~\ref{apx:consistent_solution_analysis}.

To enhance sample efficiency in the LEM cases, we propose the EIKF algorithm by initiating $\mu^{(0)}$ with an $\sqrt{n}$-consistent pose estimate at each sampling instant in Algorithm~\ref{alg:Iterated_invariant_kalman_filter}. Our EIKF is summarized in Algorithm~\ref{alg:EIEKF}. In addition, it is worth highlighting that in the EIKF, a single LGN step is sufficient to significantly improve the estimation accuracy when $n$ is large. This eliminates the need for additional LGN iterations, thereby maintaining a low computational complexity of $O(n)$. Such benefits stem from the fact
that the EIKF, as a whole, solves the mean and covariance of 
$p(\bs{X}_{k+1} |\mathcal{T}_{k+1})$
raised in Problem~\ref{prob:new update problem} as $n$ increases. This is the theoretical foundation of the EIKF in offering superior accuracy when dealing with LEMs and will be further explained in Section~\ref{sec:equivalence}.

	\begin{algorithm}[h]
		\caption{Update of EIKF}\label{alg:EIEKF}
   \KwIn{ $\bar{\bs{X}}$,$\bar{\bs{P}}$, $\bs{z}_j$~($j$ is $C$ or $L$)}
		\KwOut{$\hat{\bs{X}}$, $\hat{\bs{P}} $}	
		\textbf{Step 1:} compute $\sqrt{n}$-consistent pose $\hat{\bs{T}}_{j,\cons}$ and estimated covariance $\hat{\Sigma}_j$ via~\eqref{eqn: camera consistent solution vector} or~\eqref{consistent_Lidar_pose_est}\;
        \textbf{Step 2:} set $\hat{\bs{T}}=
					\hat{\bs{T}}_{j,\cons}{}^{I}\bs{T}^{-1}_j$,
            $\mu=\begin{bmatrix}
                \hat{\bs{T}}&\bar{\bs{v}}_I\\
                \bs{0}_{1\times 4}&1
            \end{bmatrix}$\footnotemark and compute $\bs{H}_{\mu,j}$\;
        \textbf{Step 3:}
            $\bs{K}_j=\bar{\bs{P}} \bs{J}_{\mu}^{-1\top} \bs{H}^\top_{\mu,j}\bs{S}_j^{-1}$, and $\bs{S}_j\triangleq \bs{H}_{\mu,j} \bs{J}_{\mu}^{-1 } \bar{\bs{P}} \bs{J}_{\mu}^{-1\top} \bs{H}^\top_{\mu,j}+ \hat{\Sigma}_j$\;
        \textbf{Step 4:}
        $\hat{\bs{X}} = \exp(\bs{K}_j( r_j(\hat{\bs{T}})+ \bs{H}_{\mu,j}\log{(\mu \bar{\bs{X}}^{-1})}))\bar{\bs{X}}$,
$\hat{\bs{P}} = \bs{J}^{-1}_{\mu}(\bs{I} - \bs{K}_{j}\bs{H}_{\mu,j}\bs{J}^{-1}_{\mu})\bar{\bs{P}} \bs{J}^{-1\top}_{\mu}$
	\end{algorithm}
 \footnotetext{The variable ${}^{I}\bs{T}_j$ is defined in~\eqref{eqn:relative transformation} and $\bar{\bs{v}}_I$ is the velocity component of $\bar{\bs{X}}$. 
Though the LGN iteration is initialized partly using $\bar{\bs{v}}_I$ of the a priori 
$\bar{\bs X}$, the update is optimal in the sense of MMSE to Problem~\ref{prob:new update problem}, see Theorem~\ref{thm:equivalent relation} and its proof. 
}

Lastly, we conclude this section with a discussion of the update of the augmented IMU state, which completes the ``predict-update'' Kalman filtering cycle for the IMU's states.

 \begin{remark}[Update of the augmented states of the IMU]
		In our approach, the estimation to the biases $\bs{b}_{g}$ and $\bs{b}_{a}$  cannot be accommodated with the IMU's kinematics. However, a technique outlined in~\cite{xu2022fast}
enables us to handle the estimation of $\bs{b}_{g}$ and $\bs{b}_{a}$ in conjunction with the line of $SO(3)$ kinematics under the iterated EKF framework.
It is straightforward to see that 
this technique continues to succeed within the 
InEKF and EIKF frameworks on $SE_2(3)$.
Due to space constraints, we omit implementation details here. For those interested in the finer points, we kindly refer them to~\cite{xu2022fast}. 
	\end{remark}
	
	\section{Asymptotical Property Analysis of the EIKF }~\label{sec:theoretical analysis}	
	We will investigate the EIKF with an initial $\sqrt{n}$-consistent pose estimates from LEM. This section begins by discussing the optimal fusion of the ML estimate and the a priori. We will demonstrate that the EIKF with the initial $\sqrt{n}$-consistent pose exhibits identical asymptotic properties as the optimal fusion method.
	
	We will concentrate on one particular kind of environmental measurement sensor and leave out the sensor index, as the process for the other case is similar.	
	\subsection{Fusion of the a priori with an MLE of pose}\label{sec:Optimality_fusion}
	When the pose $\bs{T}\in SE(3)$ to be estimated can be uniquely identified from our $r(\cdot)$, the MLE of $\bs{T}$ can be obtained and expressed as:
	\begin{align}\label{eqn:measurement optimization problem}
		\bs{T}_{\MLE}=\mathop{\arg\min}_{\bs{T}\in SE(3) }\left\|r(\bs{T}) \right\|^2_{\Sigma}.
	\end{align}
	 The following theorem presents a formal statement on the asymptotic properties of the ML estimate for $\bs{T}$.
	\begin{theorem}[Consistency and asymptotic normality of the ML estimate]\label{thm:property_of_MLE}
Under certain regularity conditions\footnote{
In \cite{Jennrich1969AsymptoticPO}, the discussed conditions are expressed with rigorous mathematical forms. These conditions include $\bs{T}$ being an interior element of a compact set, the measurement noises being independent and identically distributed (i.i.d.) with a zero mean and finite covariance, landmarks adhering to a specific distribution, the identifiability of $r()$, the continuity of $r'()$ and $r''()$. }, as $n$ approaches infinity, the ML estimate $\bs{T}_{\MLE}$ converges to the true state $\bs{T}$ in probability. Additionally, there exists a random variable $\xi_{\MLE}$ such that $\bs{T}_{\MLE}=\exp{(\xi_{\MLE})}\bs{T}$ and it has an asymptotic normal distribution:
		\begin{equation*}
		\sqrt{n}\xi_{\MLE}\overset{d}{\longrightarrow}\mathcal{N}(0,\bs{M}^{-1}) ,
		\end{equation*}
		where $\overset{d}{\longrightarrow}$ means convergence in distribution. The matrix $\bs{M}$ is determined by the limit of the fisher information matrix $\bs{F}_{\MLE}$ divided by $n$, i.e.,  
		\begin{equation*}
			\bs{M}=\lim_{n\to \infty}\frac{1}{n}\bs{F}_{\MLE} =\lim_{n\to \infty}\frac{1}{n}\frac{\partial r(\bs{T})}{\partial \bs{T}}^\top\Sigma^{-1}  \frac{\partial r(\bs{T})}{\partial \bs{T}}.
		\end{equation*}
	\end{theorem}
	The proof is given in the Appendix~\ref{apx:property_of_MLE}.
 
 %%%The properties of the MLE on the Lie group
	With a substantial quantity of environmental measurements, the ML estimate on the Lie group $\bs{T}_{\MLE}$ converges to the (right)concentrated Gaussian distribution with the mean of $\bs{T}$ at a rate of $1/\sqrt{n}$ and displays a minimal covariance being the inverse of the Fisher information matrix. The distribution of $\bs{T}_{\MLE}$ can be represented as $\mathcal{N}_{RG}(\bs{T},\bs{F}_{\MLE}^{-1})$. Note also that the ML estimate is asymptotically efficient according to Definition~\ref{def:efficiency}.

 %%We will study the MLE in Problem 2.2

In revisiting Problem~\ref{prob:new update problem}, our objective is to derive the posterior distribution. The necessity of a virtual measurement on the pose is crucial for the intended outcome. We are now in a position to treat the ML estimate as the virtual measurement. The following theorem scrutinizes the posterior distribution with the virtual measurement.

 %%Some settings of our theorem
 Before doing this, we rewrite the negative log-posterior of Problem~\ref{prob:new update problem} with the virtual measurement $\bs{T}_{\MLE}$ and define:
	\begin{equation}\label{eqn:linear MLE}
		\hat{\bs{X}}_{\op}= \mathop{\arg\min}_{\bs{X}\in SE_2(3)}\left\|\bar{\xi} \right\|^2_{\bar{\bs{P}}}+\left\|\log(\bs{T}_{\MLE}\bs{T}^{-1}) \right\|^2_{\bs{F}_{\MLE}^{-1}}.
	\end{equation} 
	
	\begin{theorem}[Update with ML pose]\label{thm:the_optimum_of_MAP}
With a sufficiently large $n$, the distribution of the a posteriori of Problem~\ref{prob:new update problem} is right concentrated Gaussian. The mean of the distribution is $\hat{\bs{X}}_{\op}$ and
		\begin{align}\label{eqn:linear Lie group optimization}	\hat{\bs{X}}_{\op}=\exp{(\hat{\delta})}\bar{\bs{X}},
		\end{align}
  where $\hat{\delta}\triangleq(\bar{\bs{P}}^{-1}+\tilde{\bs{I}}^\top \tilde{\bs{J}}^{-1 \top}_{\tilde{\delta}_{\MLE}} \bs{F}_{\MLE} \tilde{\bs{J}}^{-1}_{\tilde{\delta}_{\MLE}}\tilde{\bs{I}})^{-1}\tilde{\bs{I}}^\top\bs{F}_{\MLE}{\tilde{\delta}_{\MLE}}$, $\tilde{\delta}_{\MLE}\triangleq\log{(\bs{T}_{\MLE}\bar{\bs{T}}^{-1})}$, $\tilde{\bs{J}}_{\tilde{\delta}_{\MLE}}\triangleq
      \dexp^{-1}_{\tilde{\delta}_{\MLE}}\in\mathbb{R}^{6\times 6}$ and $\tilde{\bs{I}}\triangleq\begin{bmatrix}
			\bs{I}_{6}&
			\bs{0}_{6\times3}
		\end{bmatrix}$. The covariance is $(\bs{J}_{\hat{\bs{X}}_{\op}}\bar{\bs{P}}^{-1}\bs{J}^\top_{\hat{\bs{X}}_{\op}}+\tilde{\bs{I}}\bs{F}_{\MLE}\tilde{\bs{I}}^\top)^{-1}$. 
	\end{theorem}
 The proof is given in Appendix~\ref{prof:proof_of_MAPwithMLE}. 

%%Problem 2.2 can be solved by the MLE 
The non-linearity of $r(\cdot)$ in~\eqref{eqn:camera_model} or \eqref{eqn:LiDAR_model}
hinders us from solving
Problem~\ref{prob:update problem} easily.    
In contrast, the introduction of $\bs{T}_{\MLE}$ enable us to solve 
Problem~\ref{prob:new update problem} analytically with minimal effort when $n$ is sufficiently large.
%%It is noted that the MAP problem in problem 2.2 can be reformed into the least squres problem
Furthermore, the solution $\hat{\bs{X}}_{\op}$ to~\eqref{eqn:linear MLE} can be interpreted as a perturbed
$\bar{\bs{X}}$ from the left 
and has a closed-form expression. The estimator $\hat{\bs{X}}_{\op}$ is optimal in the sense of MMSE. In particular, this arises from the fact that the ML estimate $\bs{T}_{\MLE}$ is close to the true pose, thereby Lemma~\ref{lemma: compound of two matrix exponentials} applied to
asymptotically approximate the BCH term in the Lie algebra~(see~\eqref{eqn:LLS_of_MAP_using_MLE} in the proof of Theorem~\ref{thm:the_optimum_of_MAP}). 

%%The left problem is how to get the MLE, we need the algorithm that is accurate,fast and theoretical garauntee.

The ML pose $\bs{T}_{\MLE}$ is available from solving an ICP problem~\cite{Segal2009GeneralizedICP} with Lidar points.
The challenge remains in computing the ML pose from visual features due to the high nonlinearity of $r(\cdot)$ in~\eqref{eqn:camera_model}. 
While some works proposed numerical methods~\cite{olsson2008branch} to compute it, the computational cost associated with them is significantly high. 
We will 
get around the difficulty by initializing the LGN step~\eqref{eqn:LGN method of problem 1} with a $\sqrt{n}$-consistent pose. 
We will show that the resulting outcome converges to $\hat{\bs{X}}_{\op}$ in probability.

	\subsection{Fusion with $\sqrt{n}$-consistent pose estimation alternatively for EIKF}\label{sec:equivalence}
%%Setting of the EIKF method   
Recall that we have obtained the $\sqrt{n}$-consistent pose of \eqref{eqn:camera_model} or \eqref{eqn:LiDAR_model}. In Algorithm~\ref{alg:EIEKF}, the initial pose of the LGN of \eqref{eqn:prob:update problem} is set to be $\bs{T}_{\cons}$ for a single iteration. Let us denote the resulting estimate as $\hat{\bs{X}}_{\EIKF}$. We discuss the asymptotic property of $\hat{\bs{X}}_{\EIKF}$ in the next theorem. 

 %%the advantage of the EIKF
	\begin{theorem}[Asymptotic convergence of $\hat{\bs{X}}_{\EIKF}$]\label{thm:equivalent relation}
		 Let the estimated value of our Algorithm~\ref{alg:EIEKF} be denoted as $\hat{\bs{X}}_{\EIKF}$. The estimated value $\hat{\bs{X}}_{\EIKF}$ has the 
  %  can be expressed as
		% \begin{align*}
		% 	\hat{\bs{X}}_{\mf}&=\exp{(\bs K}\left(r(\bs{T}_{\cons})+\bs{H}_{\bs{X}_{\cons}}r_{\bar{\bs{X}}}(\bs{X}_{\cons})\right))\bar{\bs{X}},\\
		% 	\bs{K}&\triangleq\bs{J}_{\bs{X}_{\cons}}\bs M\bs{H}^\top_{\bs{X}_{\cons}}\Sigma_j^{-1},\\
		% 	\bs M&\triangleq\left(\bs{J}_{\bs{X}_{\cons}}^\top \bar{\bs{P}}^{-1}\bs{J}_{\bs{X}_{\cons}}+\bs{H}^\top_{\bs{X}_{\cons}}\Sigma_j^{-1}\bs{H}_{\bs{X}_{\cons}} \right)^{-1}.
		% \end{align*}
		% The covariance writes 
		% \begin{equation*}
		% 	\mathbb{E}\left[\delta_{\hat{\bs{X}}_{\mf}}\delta_{\hat{\bs{X}}_{\mf}}^\top\right]=\bs{J}^{-1}_{\bs{X}_{\cons}}(\bs{I}-{K}\bs{H}_{\bs{X}_{\cons}}\bs{J}^{-1}_{\bs{X}_{\cons}})\bar{\bs{P}}\bs{J}^{-1\top}_{\bs{X}_{\cons}}.
		% \end{equation*}
		 following asymptotic property:\begin{equation}
\hat{\bs{X}}_{\EIKF}-\hat{\bs{X}}_{\op}=o_p(\frac{1}{\sqrt{n}}).
		\end{equation}
	\end{theorem} 
    The proof can be found in Appendix~\ref{prof:proposition_of_mf}. 
    The roadmap of it is first
    to utilize LGN refinement to a $\sqrt{n}$-consistent pose 
    and then to prove 
that both $\hat{\bs{X}}_{\EIKF}$ and  $\hat{\bs{X}}_{\op}$ converge to this refined pose. 

%%%the advantage of the EIKF 
The EIKF stands out as a specialized filter that efficiently leverages the LEM case, offering a reliable initial value-$\sqrt{n}$-consistent pose. As theoretically demonstrated, the estimator shares the same asymptotic properties as $\hat{\bs{X}}_{\op}$ and is asymptotically optimal in terms of MMSE. 
In scenarios featuring reliable and dense measurements, the EIKF produces accurate and computationally effective estimates. Moreover, accuracy can be further improved with increasing data collection.

%%% %%%%Discussion with the other related works
Our filtering method features a single LGN iteration in the update of Algorithm~\ref{alg:EIEKF}. The overall computational complexity is $O(n)$, including the computation of the $\sqrt{n}$-consistent pose and the one single LGN iteration. Therefore, Algorithm~\ref{alg:EIEKF} demonstrates scalability in processing dense tracked feature positions and point clouds.

%%%The backup of the method
Considering generalizability across various scenarios, it is crucial to enable an adaptable initial pose to deal with sparse features. When the practical $n$ falls below a certain threshold, we resort to multiple LGN iterations with $\bar{\bs{X}}$ as an alternative in initialization. This is motivated by noting that the a priori constructed from the IMU may offer a more reliable initial pose for LGN. Though there is no convergence guarantee for the iterated method in this case, our experiments demonstrate its applicability in the majority of cases. In addition, we find through trial and error an empirical threshold to switch the two initialization methods in VIO and LIO scenarios. The practical EIKF is summarized in Algorithm~\ref{alg:practical EIEKF}.

\begin{algorithm}[h]
		\caption{Practical EIKF}\label{alg:practical EIEKF}
   \textbf{Parameters:}~$N$, $l_{\rm EIKF}$, $\tau_{\rm EIKF}$,$\Sigma_C$,$\Sigma_L$\\
		\KwIn{$\hat{\bs{X}}_{k}$,  $\hat{\bs{P}}_{k}$, $\bs{\omega}_{m,k+1}$, $\bs{a}_{m,k+1}$, $\bs{z}_{C,k+1}$, $\bs{z}_{L,k+1}$}
		\KwOut{$\hat{\bs{X}}_{k+1}$, $\hat{\bs{P}}_{k+1}$}
		\textbf{Predict:} compute the predicted state $\bar{\bs{X}}_{k+1}$ via~\eqref{eqn:IMU_estimator} and the covariance $\bar{\bs{P}}_{k+1}$ via~\eqref{eqn:cov_predict}\;
		\textbf{Update:}
			\uIf{$n>N$}{
				execute Algorithm~\ref{alg:EIEKF}\;
			}
			\Else{
				~set $\mu^{(0)}=\bar{\bs{X}}_{k+1}$,
				$l_{\rm max}=l_{\rm EIKF}$,
                    $\tau=\tau_{\rm EIKF}$, $\Sigma_C$, $\Sigma_L$\;
                    execute Algorithm~\ref{alg:Iterated_invariant_kalman_filter}\;
			}
	\end{algorithm}

	\section{Experiments}\label{sec:experiment}
	We present simulations and experimental results of real datasets to validate the performance of our proposed algorithm for VIO, LIO, and LVIO. Our evaluations encompass assessments of accuracy, robustness and computational efficiency, which collectively demonstrate the EIKF method compared to the state-of-the-art filtering method IEKF and InEKF~\cite{Barrau}. Additionally, we explore the impact of different initializers for the EIKF, i.e., initialization via the predicted state (EIKF-I) or via the $\sqrt{n}$-consistent pose (EIKF-C), to provide a comprehensive comparison across various scenarios and sensor configurations. 
	\subsection{Simulation}
	We presented the simulation results of different scenarios including VIO and LIO. Specifically, we focused on the PnP method in VIO and the ICP method in LIO. Different samplings of the measurements, initial disturbances and measurement noise levels were designed to showcase the accuracy. The RMSE was selected as the evaluation criterion.
	
	The IMU measurements were taken at a high rate of 100Hz, while the camera measurements were taken at a relatively low rate of 20Hz. The ground truth trajectory was generated using the position function $[70\sin(0.15t),80\cos(0.15t),7\sin(0.75t)]$ and the angular velocity function $[0.2\text{rad/s},0.3\text{rad/s},0.1\text{rad/s}]$.  The specific experimental settings that include the noise of the IMU measurements used for the evaluation were summarized in Table~\ref{tab:experiment_setting}.  
	\begin{table}[htbp]
		\centering
		\caption{Noise Statistics in Simulation}
		\begin{tabular}{lll}
			
			\toprule
			\textbf{Type} & \textbf{Measurement} & \textbf{S.D.}(Units) \\
			\midrule
			\multirow{4}{*}{\textbf{Measurement}} & Gyroscope Noises& $0.01 \, ({\text{rad}}/{\text{s}\sqrt{\text{HZ}}}) $ \\
			& Accelerometer Noises& $0.01 \, (\text{m} / \text{s}^2\sqrt{\text{HZ}})$ \\
			& Gyroscope Random Walk & $0.001 \, (\text{rad} / \text{s}^2\sqrt{\text{HZ}})$ \\
			& Accelerometer Random Walk &$0.001 \, (\text{m} / \text{s}^3\sqrt{\text{HZ}})$ \\
			% \midrule
			% \multirow{5}{*}{\textbf{State}} & Orientation of IMU & {\color{red}$0.0001 \, \text{rad}^2$} \\
			% & Position of IMU & {\color{red}$0.0001 \, \text{m}^2$} \\
			% & Velocity of IMU & {\color{red}$0.0001 \, \text{m}^2 / \text{s}^2$} \\
			% & Gyroscope Bias & {\color{red}$0.0001 \, \text{rad}^2 / \text{s}^2$} \\
			% & Accelerometer Bias & {\color{red}$0.0001 \, \text{m}^2 / \text{s}^4$} \\
			
			\bottomrule
		\end{tabular}
		\label{tab:experiment_setting}
	\end{table}
	\subsubsection{VIO}
	The IEKF,  the InEKF, the EIKF-C and the EIKF-I  were executed 100 times in four different scenarios to assess and compare their performance. These scenarios included a single setting, varying landmarks, different levels of noise, and initial disturbance.

	%------------------------------------------------------------------------------
	% The EKF,  the InEKF, {\color{red}the iterated InEKF} and the EIKF were run 100 times in two cases of various scale of noise and various random initial orientations and positions. In the first case, the scales of noise were set into $\sigma_{\text{cam}}=0.02$ and $\sigma_{\text{cam}}=2$. In the second one, the orientation initial errors were set to lie within the interval $(0, \pi/12)$ for pitch, roll, and yaw, while the position initial errors were set to lie within $(0,10)$ for the x, y, and z directions. 
	%%%1. Comparison of Different Initial error
	
	\begin{figure}[htbp]
		\centering
		\subfigure[]{
			\includegraphics[width=0.25\textwidth]{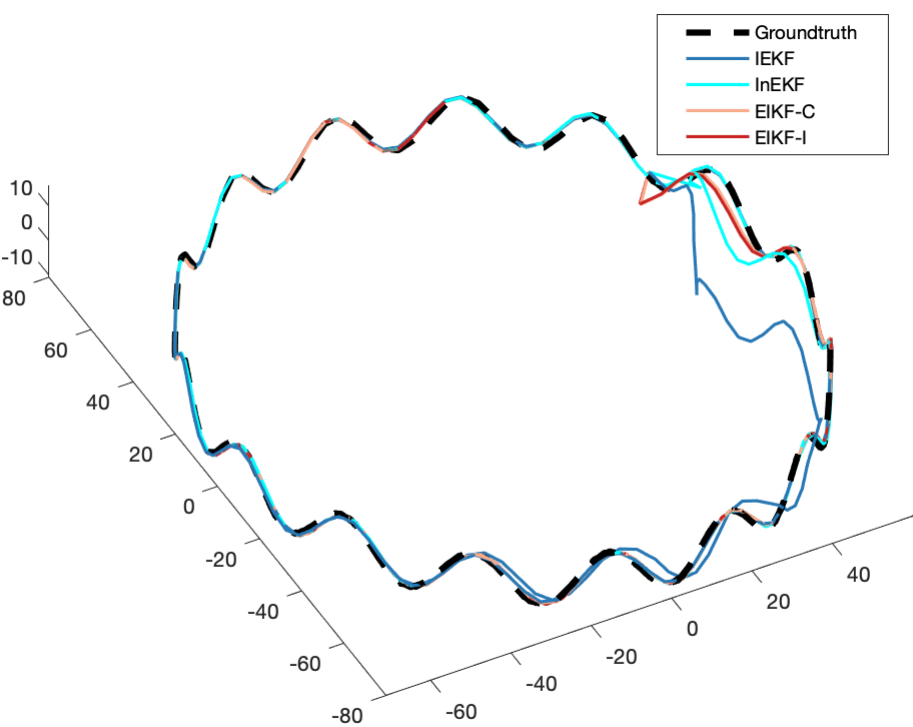}
		}
  \subfigure[]{
			\includegraphics[width=0.4\textwidth]{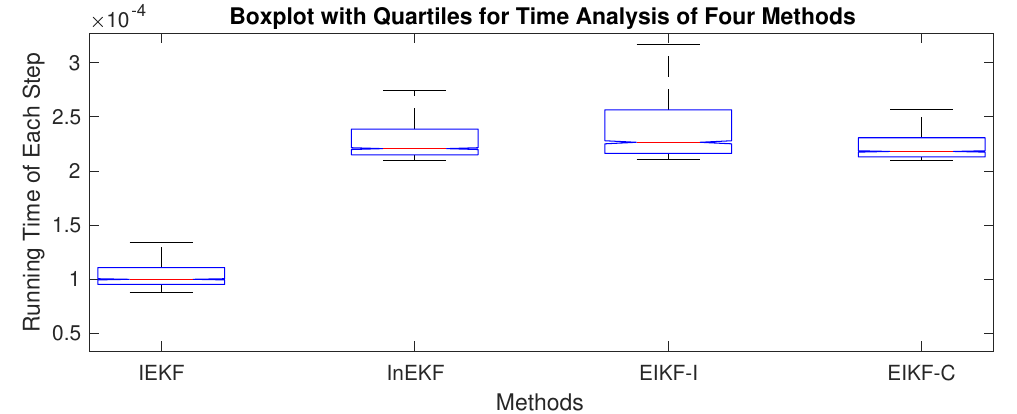}
		}
		\subfigure[]{
			\includegraphics[width=0.4\textwidth]{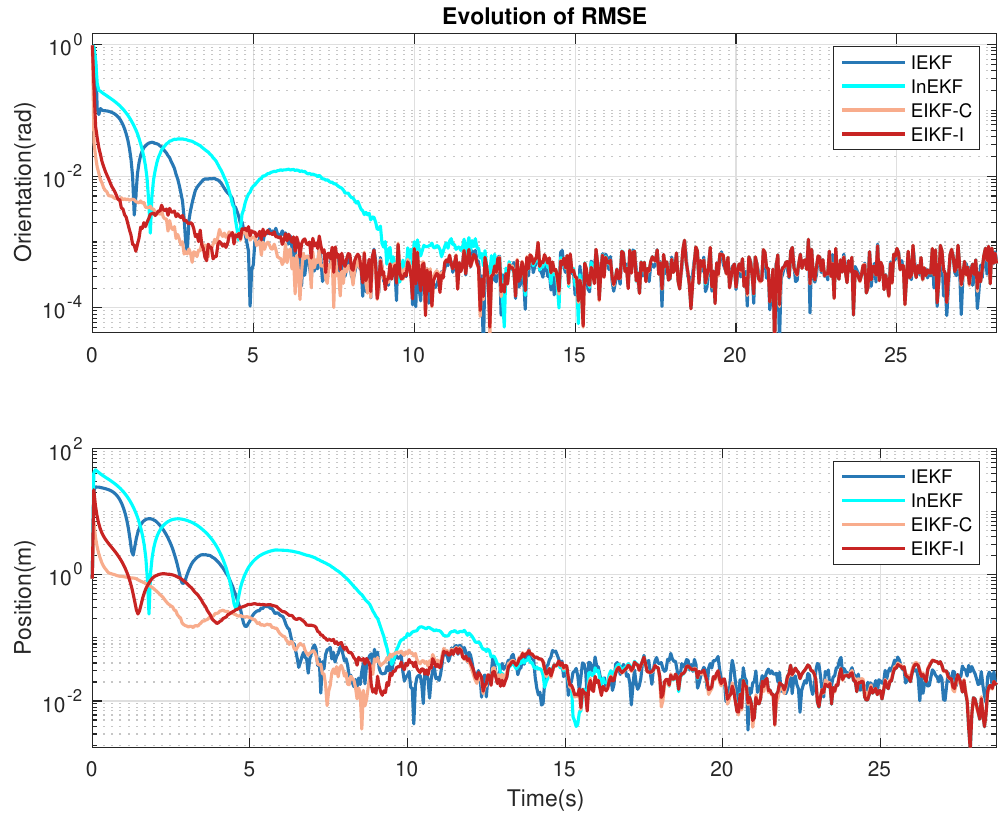}
		}
		\caption{The estimated trajectories, time analysis and RMSEs of orientation and position through IEKF, InEKF, EIKF-C and EIKF-I for VIO. The initial position deviation is set to $[0.5m, 0.5m, 0.5m]$, and the initial attitude deviation is set to $[\frac{\pi}{6}, \frac{\pi}{6}, -\frac{\pi}{6}]$~($[\text{roll}, \text{pitch}, \text{yaw}]$). The entire trajectory is generated using $\begin{bmatrix}
				70\sin{(0.15t)},80\sin{(0.15t)},7\sin{(0.75t)}.
			\end{bmatrix}$. We allow a maximum of three iterations for IEKF and EIKF-I. All methods successfully converge, but exhibit different average RMSE values over the entire evolution. Indeed, EIKF-C demonstrates superior performance in low cost as other filterings do, small RMSE and a fast convergence rate, with just a single iteration in update of the time period from $t=0$ to $30$. }
		\label{vio_rmse}
	\end{figure}
	In the first case, the entire trajectory in our experiment spanned a duration of 30 seconds with 100 camera landmarks and the standard derivation of noises set into $\sigma_{\text{cam}}=1$ pixel.  The trajectory, the time analysis and the evolution of the RMSE of orientation and position were shown in Figure~\ref{vio_rmse}.  The experimental results demonstrated that all filters exhibit convergence, but with a different convergence rate and average RMSE. EIKF-C demonstrated superior performance in RMSE and convergence rate while maintaining low computational cost.  To simplify the comparison, we used the average RMSE to evaluate all the methods in the following three experiments. 
	
	In the second case, we performed a comparison of all methods under  \textit{ varying numbers of the landmark} settings. The experimental results were depicted in Figure~\ref{fig:landmarks_num_vio}. The figure illustrated that as the quantity of landmarks increased, the average RMSE of EIKF-C noticeably decreased, whereas the other methods did not exhibit this behavior. EIKF-C consistently outperformed the other methods when the quantity of landmark was larger than $50$. Furthermore, EIKF-C exhibited a faster asymptotic rate with respect to the quantity of landmarks, confirming the asymptotic efficiency as analyzed in Section~\ref{sec:theoretical analysis}.
	
	The third and fourth cases involved comparisons of all methods under different \textit{noises} and \textit{initial} \textit{deviation} settings to demonstrate robustness. In terms of noise setting, we considered standard deviations ranging from $0.1$ to $2$ pixels as the noises levels. The results are displayed in Figure~\ref{fig:noise_scale_vio}. EIKF-I and EIKF-C achieved lower RMSEs compared to those of other methods. However, when the measurement noise was set to values greater than $1.5$, EIKF-I outperformed EIKF-C, indicating that initialization using the predicted state performed better in these scenarios. Regarding the initial disturbances, we applied the deviation coefficients ranging from $0$ to $1.0$ to the values $[0.5m, 0.5m, 0.5m]$ and $[\frac{\pi}{6}, \frac{\pi}{6}, -\frac{\pi}{6}]$. The results, shown in Figure~\ref{fig:init_scale_vio}, further highlighted the lower RMSEs achieved by EIKF-I and EIKF-C. In summary, these two experiments showcased the accuracy and robustness of EIKF under different noise and disturbance conditions.
	
	\begin{figure}[htbp]
		\centering
		\includegraphics[width=0.47\textwidth]{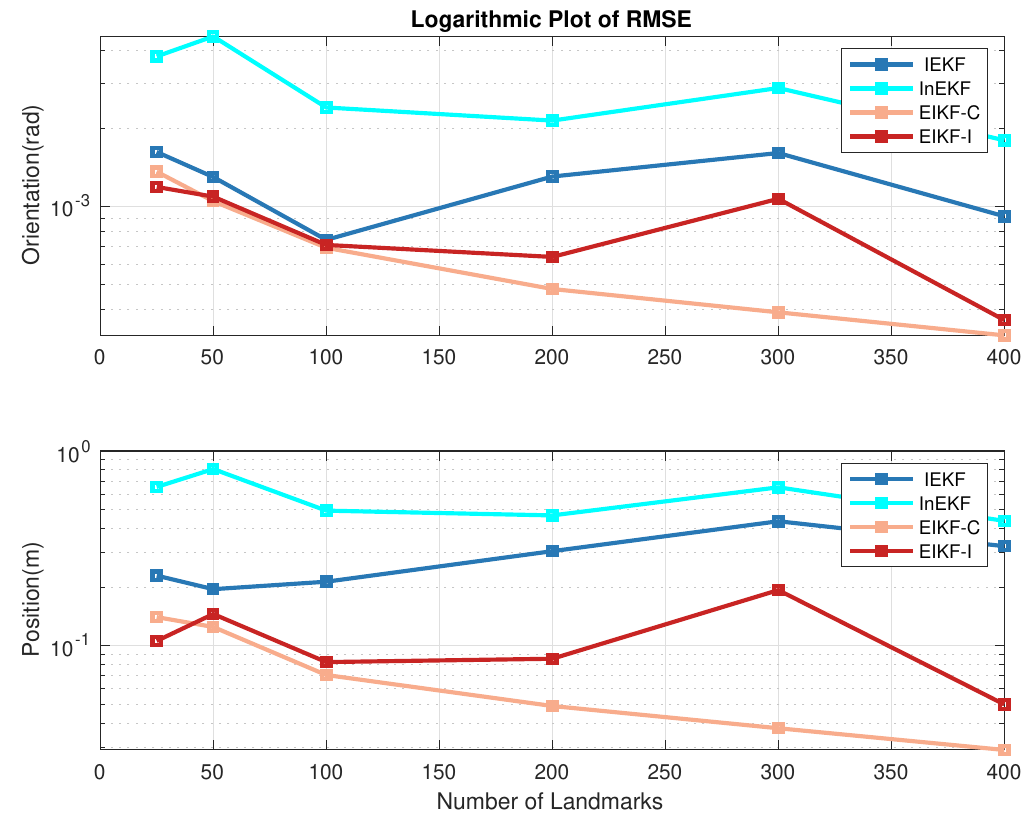}
		
		\caption{RMSEs for orientation and position, versus different numbers of landmarks, are compared among IEKF, InEKF, EIKF-C, and EIKF-I. EIKFs outperform others in terms of both asymptotic convergence rate and the average RMSE. }\label{fig:landmarks_num_vio}
	\end{figure}
	
	\begin{figure}[htbp]
		\centering
		\includegraphics[width=0.47\textwidth]{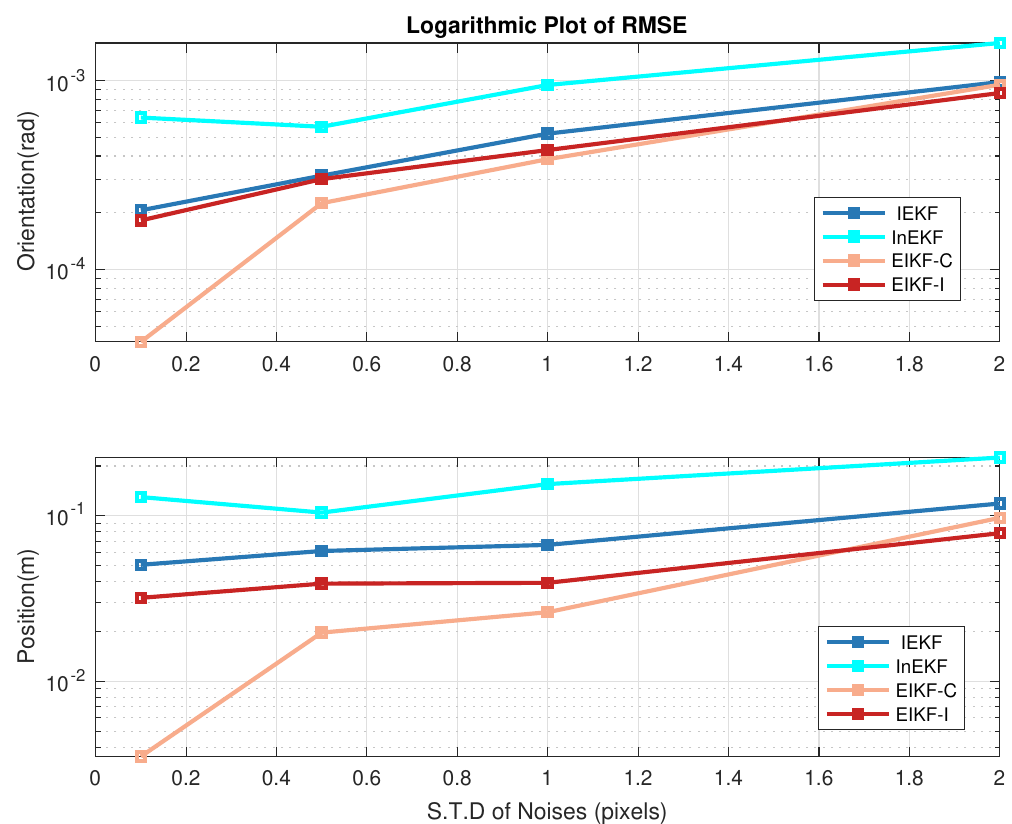}
		
		\caption{RMSEs for orientation and position, versus different levels of noises, are compared among IEKF, InEKF, EIKF-C, and EIKF-I.}\label{fig:noise_scale_vio}
	\end{figure}
	
	\begin{figure}[htbp]
		\centering
		\includegraphics[width=0.47\textwidth]{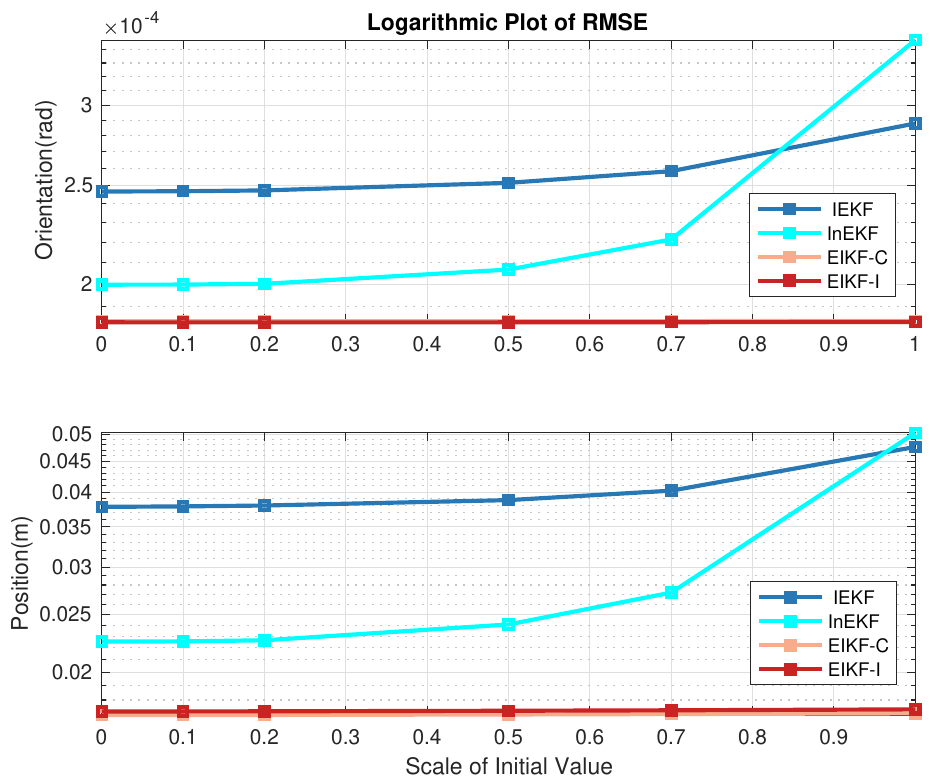}
		\caption{RMSEs for orientation and position, versus different scales of initial deviations, are compared among IEKF, InEKF, EIKF-C, and EIKF-I.}\label{fig:init_scale_vio}
	\end{figure}
	% 2. convergence property

	% 3. iterated inekf v.s. eikf
	% In addition to the main experiments, we also conducted an additional experiment to demonstrate the difference between EIKF and the iterated InEKF, with the same experimental settings including initial conditions and scale of noise. In this experiment, we observed the convergence behavior of both methods after two iterations. The results showed that the convergence gap of the iterated InEKF became more similar to that of EIKF after two iterations. However, even with the same number of iterations, EIKF still exhibited better performance in terms of accuracy and efficiency, as shown in Figure \ref{}. This highlights the advantages of the EIKF method in providing efficient and optimal filtering results compared to the iterated InEKF approach.

	%%%%%%%%%%%%%%%%%%%%%%%%%%%%%%%%%%%%%%%%%%%%%%%%%%%%%%%%%%%%%%%%%%%%%%%%%%%%%%%%%%%%%%%%%%%%%%%    
	\subsubsection{LIO}
	The IEKF, the InEKF and the EIKF were also conducted on LIO system in the same four distinct scenarios as the VIO system. The measurements of LiDAR and IMU were obtained at 50Hz and 100Hz, respectively. 
	
	In the first case, using the same trajectory as the VIO with 400 landmarks and a noises standard deviation of 0.2, we presented the trajectories of all filters in Figure~\ref{lio_path_noise}. The experimental results revealed the fast convergence rate and lower average RMSE of the EIKFs compared to IEKF and InEKF. In the other three cases, involving comparisons under different numbers of landmarks, noises levels, and initial deviations, the simulations supported similar conclusions as in the VIO experiments. Moreover, the differences between IEKF and EIKF in LIO were less pronounced than in VIO, indicating that EIKF excelled in non-linear models compared to linear ones.
	% \begin{figure}[htbp]
		% 	\centering
		% 	\subfigure[]{
			% 		\includegraphics[width=\mysize]{Mlio_eps/liosmalltraj.eps}
			% 	}
		% 	\subfigure[]{
			% 		\includegraphics[width=\mysize]{Mlio_eps/liobigtraj.eps}
			% 	}
		% 	\subfigure[]{
			% 		\includegraphics[width=\mysize]{Mlio_eps/smalle.pdf}
			% 	}
		% 	\subfigure[]{
			% 		\includegraphics[width=\mysize]{Mlio_eps/liobigERR.eps}
			% 	}
		
		% 	\caption{The estimated trajectory and RMSE of position and orientation through EKF, InEKF and EIKF for LIO with different LiDAR measurement noises. The initial position $[x,y,z]$ deviation is set to  $[-2\text{m}, -3 \text{m}, 1 \text{m}]$ and the attitude $[roll,pitch,yaw]$ deviation is $[18^\circ,12^\circ, 12^\circ]$. From left to right: \textbf{The first column}: small LiDAR measurement noise ($\sigma_{lid}=0.20$m) condition. \textbf{The second column}: large LiDAR measurement noise ($\sigma_{lid}=0.80$m) condition.}
		%  	\label{lio_path_noise}
		% \end{figure}
	
	\begin{figure}[htbp]
		\centering
		\includegraphics[width=0.45\textwidth]{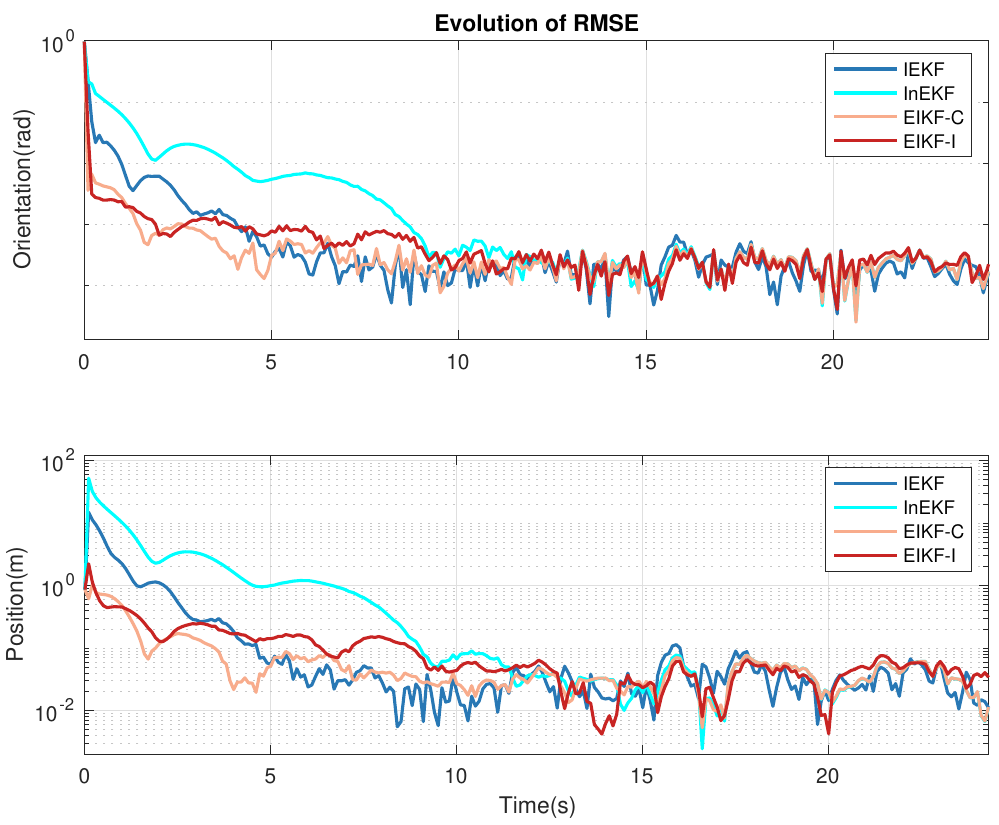}
		\caption{The estimated trajectory and RMSEs of position and orientation by IEKF, InEKF, EIKF-C and EIKF-I for LIO.}
		\label{lio_path_noise}
	\end{figure}
	
	\begin{figure}[htbp]
		\centering
		\includegraphics[width=0.45\textwidth]{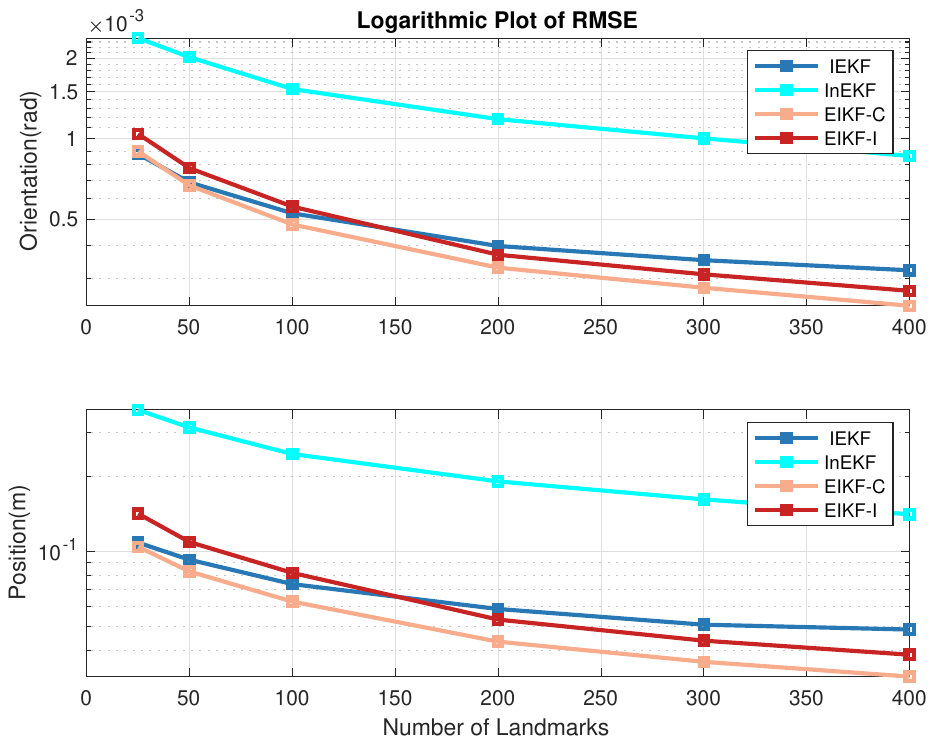}
		
		\caption{RMSEs for orientation and position versus different numbers of landmarks, are compared among IEKF, InEKF, EIKF-C, and EIKF-I.}
	\end{figure}
	
	\begin{figure}[htbp]
		\centering
		\includegraphics[width=0.45\textwidth]{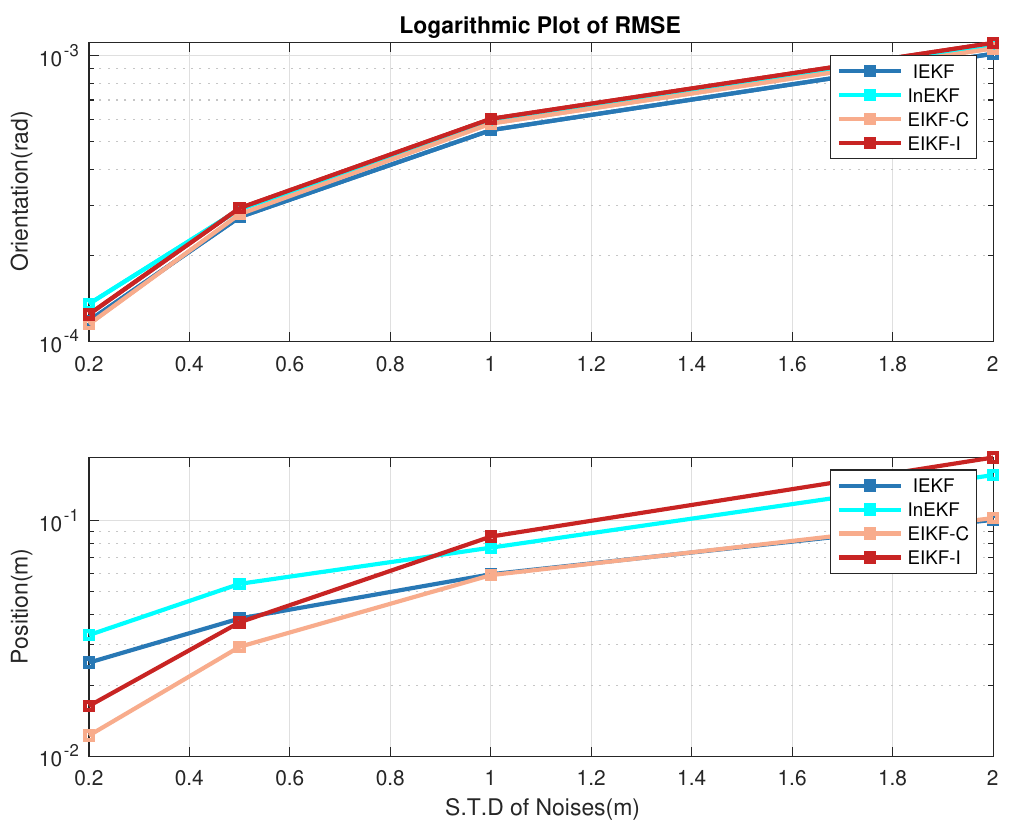}
		
		\caption{RMSEs for orientation and position versus different levels of noises are compared among IEKF, InEKF, EIKF-C, and EIKF-I.}
	\end{figure}
	
	\begin{figure}[htbp]
		\centering
		\includegraphics[width=0.45\textwidth]{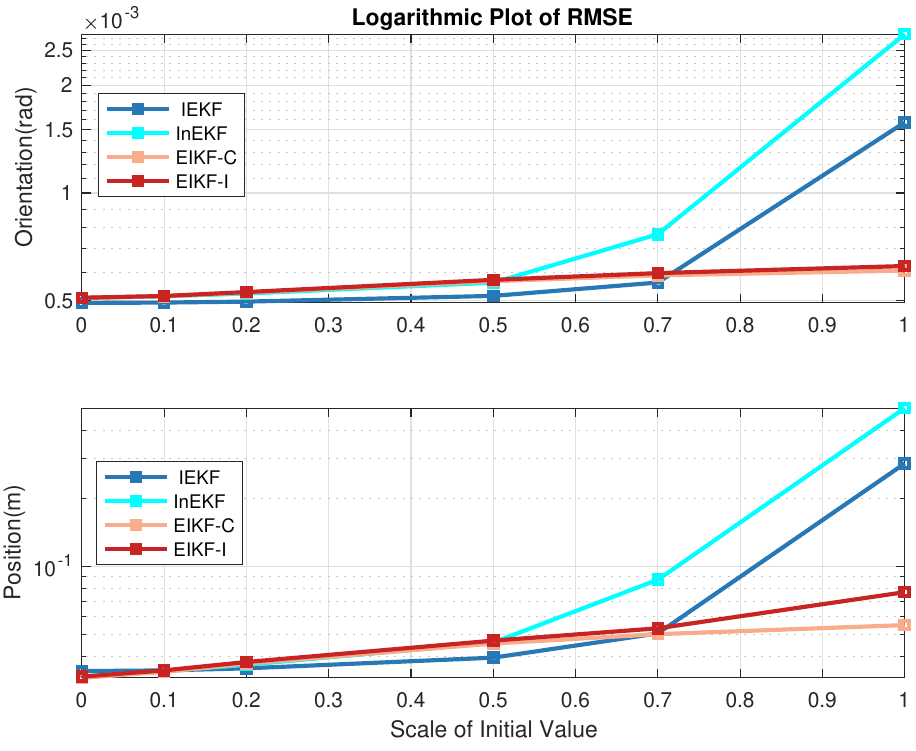}
		\caption{RMSEs for orientation and position versus different scales of initial deviations are compared among IEKF, InEKF, EIKF-C, and EIKF-I.}
	\end{figure}

	In summary of the simulation part, these results clearly demonstrated the superiority of the EIKFs method, particularly in terms of the RMSE evaluation and robustness in the non-linear measurement problem.  Specifically, \begin{inparaenum}
		\item from the initial deviation experiments, the class of InEKF, including InEKF, EIKF-I and EIKF-C, worked better than the one of EKF and the iterated scheme improved the accuracy, proving the accuracy and robustness of EIKF.
		\item when meeting the smaller measurement noises and LEM, EIKF-C worked better than EIKF-I, proving the asymptotic efficiency of EIKF.
		\item  EIKF-C had stable performance without considering iteration due to theoretical guarantee as we analyzed in Section~\ref{sec:theoretical analysis}.
	\end{inparaenum}    Hence, EIKF-C leveraged its asymptotic optimum to handle the non-linearities effectively without relying on linearization, confirming its superior performance in pose estimation tasks.
	
	% The estimated trajectory and RMSE of position and orientation through EKF, InEKF and EIKF for LIO with small LiDAR measurement noise $\sigma_{lid}=0.2$. From left to right: \textbf{The first column}: no initial position and attitude deviations. Three filters perform similarly. \textbf{The second column}: large initial position deviations $[15 \text{m}, 10 \text{m}, 5 \text{m}]$ and attitude deviations (roll $15^\circ $, pitch $15^\circ $, yaw $15^\circ $) with LiDAR measurement noise $\sigma_{lid}=0.2$. From the enlarged RMSE curve, it can be observed that as time increases, both InEKF and EIKF outperform EKF.}

\subsection{Experiment: Extensive evaluation in public datasets}
This experiment was designed to evaluate the performance of EIKF in some selected filtering algorithms in different inertial-based navigation system, as depicted in Figure~\ref{fig:exp}. Our main goals consisted of assessing the precision and speed of execution. For accuracy evaluation, we employed the following metric to compare estimated poses with ground truth poses: the RMSE of the Absolute Transition Error (ATE) \footnote{The ATE is determined by calculating the difference between the estimated trajectory and ground truth after alignment. More details are available in~\cite{Zhang2018ATO}}. We compared our method with other state-of-the-art techniques such as EKF, IEKF, and InEKF. To test the filterings, we utilized the OpenVINS~\cite{Geneva2020OpenVINSAR} framework for the VIO module, Fast-LIO~\cite{xu2022fast} for the LIO module, and R3LIVE~\cite{lin2021r} for the LVIO module.

\begin{figure}[htbp]
	\centering
	\includegraphics[width=0.5\textwidth]{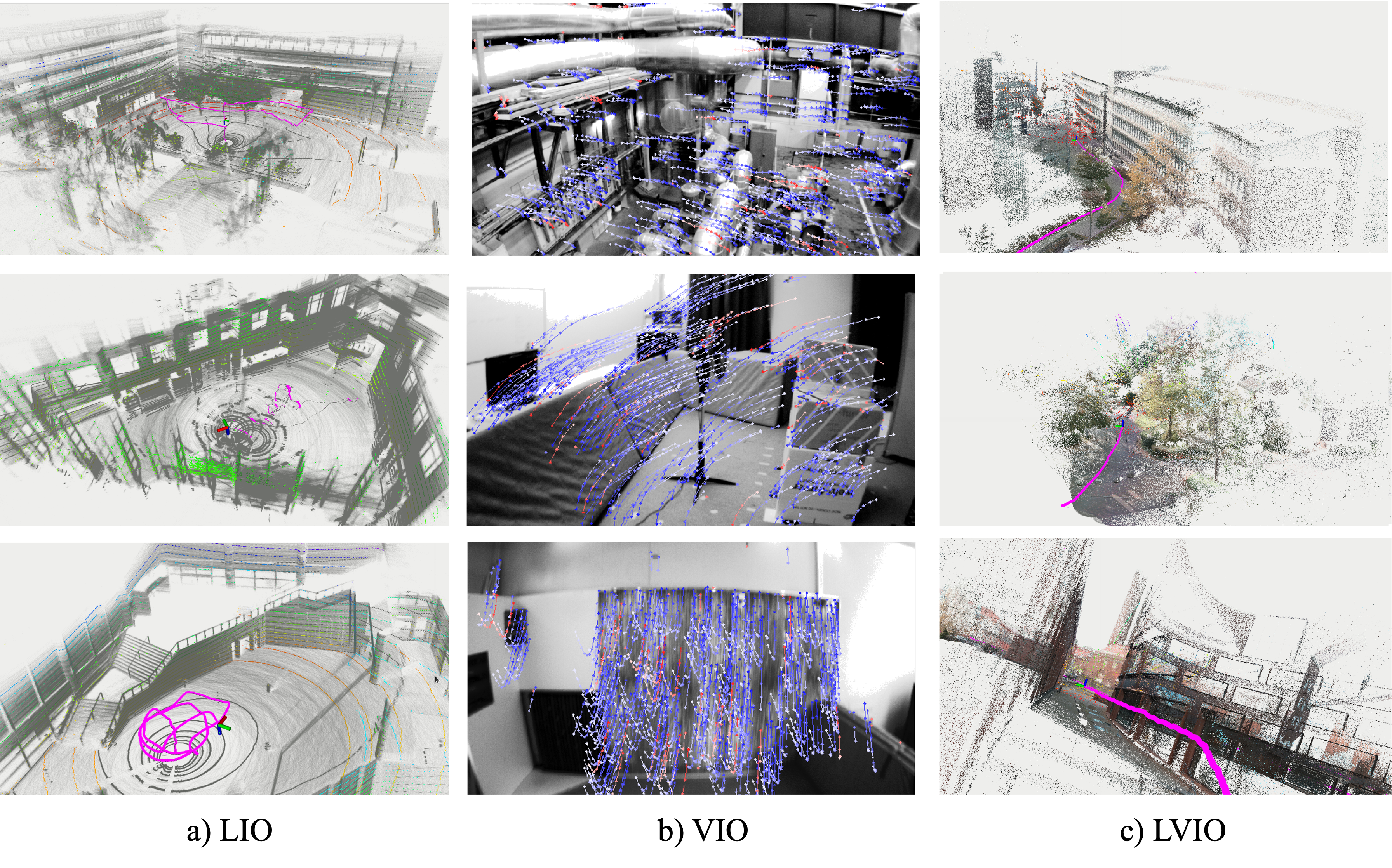}
	\centering
	\caption{Selected Inertial-based odometry. The left column represents LIO with the EIKF implementation, the middle column is VIO, and the right column represents LVIO. In these experiments, we use a $\sqrt{n}$-consistent pose estimation as the initial value for the update step of the EIKF in LEM scenarios.  }\label{fig:exp}
\end{figure}

\begin{table} 
	\centering
	\caption{ATE Analysis in VIO Evaluation}
	\begin{threeparttable}
		\begin{tabular}{c c c c c }
			\hline
			\hline
			&  ($^\circ$) / ($m$) &   EKF           & InEKF              & EIKF           \\
			\hline
			& \textbf{MH\_01\_easy} & 2.498 / 0.168 &  2.049 / 0.134& \textbf{1.785} / \textbf{0.116}\\
			& \textbf{MH\_02\_easy} & 0.749 / \textbf{0.104}  & 0.717 / 0.110 & \textbf{0.606} / 0.118\\
			& \textbf{MH\_03\_medium} &1.498 / 0.154  & 1.327 / 0.162 & \textbf{1.199} / \textbf{0.095}\\
			& \textbf{MH\_04\_difficult}&  0.961 / 0.193 & \textbf{0.957} / 0.175  &1.128 / \textbf{0.160}\\
			& \textbf{MH\_05\_difficult} & 1.380 / 0.294 &0.976 / 0.297  &   \textbf{0.770} / \textbf{0.235}  \\
			& \textbf{V1\_01\_easy}  & 0.829 / 0.060 & \textbf{0.814} / 0.066 & 0.832 / \textbf{0.048}\\
			& \textbf{V1\_02\_medium} & 1.982 / 0.079& \textbf{1.928} / 0.075 & 1.996 / \textbf{0.063}\\
			& \textbf{V1\_03\_difficult}&  2.978 / 0.066& 2.564 / 0.061 & \textbf{2.304} / \textbf{0.051}\\
			& \textbf{V2\_01\_easy} &\textbf{1.416} / 0.076& 1.529 / 0.077 & 1.465 / \textbf{0.057} \\
			& \textbf{V2\_02\_medium}& 1.338 / 0.087 &1.186 / 0.080 & \textbf{0.946} / \textbf{0.072}\\
			& \textbf{V2\_03\_difficult}& 1.621 / 0.100 & 1.526 / 0.101 & \textbf{0.939} / \textbf{0.092}\\ 
			\hline\hline
		\end{tabular}\label{table:exp_VIO}
		
	\end{threeparttable}
\end{table}

\begin{table}[]
	\centering 
	\caption{Selected Datasets in LIO and LVIO esperiment}
	\begin{tabular}{ c c c c}
		\hline
		
		Name & Dataset\tablefootnote{We performed the evaluation on the NTU VIRAL dataset using the specialized MATLAB script provided by the dataset authors~\cite{nguyen2022ntu}. For assessing other datasets, we utilized the EVO toolkit, accessible at \href{https://github.com/MichaelGrupp/evo}{https://github.com/MichaelGrupp/evo}.} & \makecell{duration\\($sec$)} &\makecell{Distance\\($m$)} \\ 
		\hline
		ulhk1 & HK-2019-04-26-1 & 108.08 & 396.28  \\  
		ulhk2 & HK-2019-04-26-2  & 87.14 & 248.00 \\
		ulhk3 & HK-2019-03-16-1  & 86.44&118.90   \\
		ulca1 &  \makecell{CAGoldenBridge\\20190828191451} & 240  &4452.31  
		\\
		ulca2 & \makecell{CALombardStreet\\20190828190411 }&  253 &1499.69 
		\\
		eee 01 & NTU VIRAL &398.7&237.07\\
		eee 02 & NTU VIRAL&321.1&171.07  \\
		eee 03 & NTU VIRAL&181.4&127.83  \\
		nya 01 & NTU VIRAL&396.3&160.30  \\
		nya 02 & NTU VIRAL&428.7&249.10  \\
		nya 03 & NTU VIRAL&411.2&315.46 \\
		sbs 01 & NTU VIRAL&354.2&202.30  \\
		sbs 02 & NTU VIRAL&373.3&183.57 \\
		sbs 03 & NTU VIRAL&389.3&198.54 \\
	       tuhh\_02 &   MCD zzz\_day\_02 & 500   & 749  \\
		tuhh\_04 & MCD zzz\_day\_02&   188   &297\\
            kth\_10 & MCD yyy\_day\_10&  615 &918.5\\
		% V1\_01\_easy & Euroc Mav &&\\
		% V1\_02\_medium & Euroc Mav &&\\
		% V1\_03\_difficult & Euroc Mav &&\\
		% V2\_01\_easy & Euroc Mav &&\\
		% V2\_02\_medium & Euroc Mav &&\\
		% V2\_03\_difficult & Euroc Mav &&\\
		% MH\_01\_easy & Euroc Mav &&\\
		% MH\_02\_easy & Euroc Mav &&\\
		% MH\_03\_medium & Euroc Mav &&\\
		% MH\_04\_difficult & Euroc Mav &&\\
		% MH\_05\_difficult & Euroc Mav &&\\
		\hline
	\end{tabular}
\end{table}

\begin{table} 
	\centering
	\caption{ATE Analysis in LIO Evaluation}
	\begin{threeparttable}
		\begin{tabular}{c c c c c}
			\hline \hline
			% &  & &LIO &&   \\
			% \hline
			
			Dataset&   EKF &Iterated  EKF&InEKF & EIKF(Ours)\\
			% &\vline&EKF &Iterated  EKF&InEKF & EIKF(Ours)     \\
			\hline
			% Dataset&   EKF &Iterated  EKF&InEKF & Iterated InEKF & EIKF
			eee01        &     0.40      &  0.21    & 0.55  &\textbf{0.21}  \\
			eee02         &      0.42       &  0.27     &   0.13  & \textbf{0.12}  \\
			eee03         &      0.30      & 0.21       &  \textbf{0.14} &  0.16   \\
			nya01        &         0.16    & 0.14           & 0.32  &  \textbf{0.11}    \\
			nya02        &       0.26      &\textbf{0.21}      &0.24&\textbf{0.21}\\
			nya03       &0.13&0.13&0.32&\textbf{0.12}                                                  \\
			sbs01   &    0.22 &0.18         &  \textbf{0.09}        & 0.11               \\
			sbs02&0.19&0.16&\textbf{0.10}&0.14\\
			sbs03&0.15&0.14&\textbf{0.10}&\textbf{0.10}\\
			ulhk1&1.89&0.92&1.45&\textbf{0.60}\\
			ulhk2&0.51 & 0.50& 0.50 & \textbf{0.49}\\
			ulhk3&1.37&1.22&0.93&\textbf{0.76}\\
			ulca1& diverge& 0.26 & 0.20 & \textbf{0.15}  \\
			ulca2\tablefootnote{To assess the performance of ULCA2, we acknowledge that the ground truth for the $z$-axis is not accurate. Therefore, we also provide an analysis of the errors in the $xy$-axis. }&  diverge& 1.17 (0.69)&1.64 (0.72)&\textbf{1.12} (\textbf{0.67} ) \\
			
			% 1.51&2.65&1.11&\textbf{0.93} & 2.31   \\
			\hline\hline
		\end{tabular}\label{table:exp_LIO_LVIO}
		%The evaluation results are presented in terms of accuracy assessment using the Absolute Pose Error (APE) per 100 meters metric, aligning with the evaluation methodology outlined in the work of Zheng et al.~\cite{zheng2022fast}.
	\end{threeparttable}
\end{table}

% \begin{table*} \label{ATE_ov}
	% \centering
	% \caption{ATE openvins}
	% \begin{threeparttable}
		% \begin{tabular}{c c c c c c c c c c c c c }
			% \hline
			%     \hline
			% & \textbf{MH\_01\_easy} & \textbf{MH\_02\_easy} & \textbf{MH\_03\_medium} & \textbf{MH\_04\_difficult} & \textbf{MH\_05\_difficult} & \textbf{V1\_01\_easy}  & \textbf{V1\_02\_medium} & \textbf{V1\_03\_difficult} & \textbf{V2\_01\_easy} & \textbf{V2\_02\_medium} & \textbf{V2\_03\_difficult} & \textbf{Average} \\\hline
			% EIKF & 104.887 / 22514.159 & 139.708 / 2386.924 & 5.805 / 76.429 & 1.215 / 0.208 & 1.101 / 0.285 & 0.885 / 0.068  & 1.928 / 0.075 & 2.285 / 0.064 & 104.950 / 374.764 & 1.186 / 0.080 & 1.526 / 0.101 & 33.225 / 2304.832 \\
			% EKF & 159.399 / 70.565 & 117.674 / 53.484 & 86.472 / 42.557 & 1.086 / 0.212 & 1.401 / 0.298 & 0.829 / 0.060 & 1.982 / 0.079 & 2.978 / 0.066 & 1.416 / 0.076 & 1.338 / 0.087 & 1.621 / 0.100 & 34.200 / 15.235 \\
			% InEKF & 119.111 / 9397.770 & 91.666 / 2546.365 & 5.972 / 76.405 & 1.215 / 0.208 & 1.101 / 0.285 & 0.814 / 0.066 & 1.928 / 0.075 & 2.564 / 0.061 & 1.529 / 0.077 & 1.186 / 0.080 & 1.526 / 0.101 & 20.783 / 1092.863 \\
			%  \hline\hline
			% \end{tabular}
		
		% \end{threeparttable}
	% \end{table*}

\subsubsection{VIO}
We conducted a comparative evaluation of the EIKF algorithm within the OpenVINS framework, comparing it to the EKF and InEKF methods. OpenVINS is based on the Multi-State Constraint Kalman Filter, an EKF approach designed for VIO pose estimates. We customized the framework to integrate the EIKF algorithm and improve feature point tracking. In particular, InEKF performance was improved to the extent that the iterated scheme was no longer required. We used the EuRoC MAV dataset as our evaluation dataset.

The experimental results, presented in Table~\ref{table:exp_VIO}, clearly showed that EIKF outperformed the other methods in terms of accuracy, when tracking a large number of feature points.

\subsubsection{LIO}
We used two diverse datasets, namely the NTU VIRAL dataset~\cite{nguyen2022ntu} and the Urbanloco dataset~\cite{Wen2019UrbanLocoAF}, to assess the performance of the LIO module with the EIKF method. These datasets allowed us to evaluate the accuracy and robustness of EIKF in a variety of real-world scenarios, spanning indoor and outdoor environments, dynamic and static environments, and presenting a comprehensive set of challenges. This facilitated a thorough evaluation of the LIO module's performance and enabled a comparison with other state-of-the-art methods.

The experimental results, as summarized in Table~\ref{table:exp_LIO_LVIO}, highlighted the superior accuracy of the EIKF method compared to other filtering approaches, such as IEKF in Fast-LIO. 
\subsubsection{LVIO}
We evaluated the filter performance in datasets\cite{mcdviral2023}. The LVIO architecture involved LIO and VIO. The LIO utilized LiDAR measurements to reconstruct geometric structures and estimated poses using the point-to-plane ICP method. On the other hand, VIO improved pose estimates by incorporating the PnP method, which relied on point clouds to provide features.

The results presented in Table~\ref{table:exp_LVIO} provided valuable insights into the accuracy of the EIKF method within the VIO module, particularly when handling LEM scenarios. These findings demonstrated the potential of the EIKF method to improve the accuracy of pose estimates in LVIO scenarios.
\begin{table}[]
	\centering
	\caption{Absolute Pose Error w.r.t. Translation Part($m$) Analysis in LVIO Evaluation}
	\begin{tabular}{ccccc}
		\hline \hline
		&   \multicolumn{1}{l}{EKF} &   \multicolumn{1}{l}{IEKF}        & \multicolumn{1}{l}{InEKF}       & \multicolumn{1}{l}{EIKF} \\ \hline
		
		tuhh\_02           &2.1  &0.29                                 &0.34                                      &\textbf{0.23 }           \\ 
		tuhh\_04          &   0.67                            & 0.48                                     & 38.5 & \textbf{0.41 }         \\ 
		kth\_10       &1.27             &  1.41               & 1.22                                      & \textbf{1.02 }              \\ \hline\hline
	\end{tabular}\label{table:exp_LVIO}
\end{table}

\section{Conclusion}\label{sec:conclusion}
EIKF presents an appealing approach to achieve optimality in the sense of MMSE in the LEM scenarios. We propose three key contributions. Firstly, we introduce an iterative framework for filtering with adaptive initialization. Secondly, we address a $\sqrt{n}$-consistent pose of the PnP and general ICP problem, utilizing it as the initial value in the GN update step. The solution is closed-form and has $O(n)$ computational complexity. Finally, we demonstrate that the filter can attain optimality in the estimation from the a priori and MLE in LEM. This superior performance holds significant practical implications, potentially benefiting VIO, LIO and LVIO with dense samples of the measurements in the field. 
Furthermore, our findings open the door to future research avenues, where the theory-guaranteed number of numerous environmental measurements and a better initial solution can be further explored and optimized for various domains.

\appendices 

			{
				\section{The Gauss-Newton method as a update step of iterated InEKF}\label{apx:Kalman Filter gain}
				This section introduces how to incorporate the LGN method into Kalman filtering. 

				The increment $\tilde{\delta}_{l,l+1}$, defined as $\mu^{(l+1)}=\exp(\tilde{\delta}_{l,l+1})\mu^{(l)}$, from the $l$th iterated value to the $l+1$ iterated value is obtained in the following form by the LGN method:
				\begin{align*}
					\tilde{\delta}_{l,l+1}&=\bs{F}_{\mu^{(l)}}\left(\bs{H}_{\mu^{(l)}}^\top \Sigma^{-1}r(\bs{T}^{(l)})-\bs{J}_{\mu^{(l)}}^\top \bs{P}^{-1}\log{(\mu^{(l)}\bar{\bs{X}}^{-1})}\right),
				\end{align*}  
                where $				\bs{F}_{\mu^{(l)}}\triangleq\left(\bs{J}_{\mu^{(l)}}^\top \bs{P}^{-1}\bs{J}_{\mu^{(l)}}+\bs{H}_{\mu^{(l)}}^\top \Sigma^{-1}\bs{H}_{\mu^{(l)}} \right)^{-1}$.
				
				Recall that $\delta^{(l)}\triangleq\log{(\mu^{(l)}\bar{\bs{X}}^{-1})}$.
				Under the assumption that $\tilde{\delta}_{l,l+1}$ is small in Lemma~\ref{lemma: compound of two matrix exponentials}, we have that
				\begin{equation*}						\delta^{(l+1)}=\bs{J}_{\mu^{(l)}}\tilde{\delta}_{l,l+1}+\delta^{(l)}.				\end{equation*}
				Using the identity $ \delta^{(l)}\equiv\bs{J}_{\mu^{(l)}}\delta^{(l)}\equiv\bs{J}^{-1}_{\mu^{(l)}}\delta^{(l)}$, we have the following 
				\begin{align*}
						\delta^{(l+1)}
						=\bs{J}_{\mu^{(l)}} \bs{F}_{\mu^{(l)}}\left(\bs{H}_{\bs{T}^{(l)}}^\top \Sigma^{-1}\left( r(\mu^{(l)})+\bs{H}_{\mu^{(l)}} \delta^{(l)}\right) \right). 
				\end{align*}
				
				Regarding the form of the Kalman filter, the estimate $\mu^{(l+1)}$ can be written in the following form:
				\begin{align*}
	\mu^{(l+1)}=\exp{(\bs{K}^{(l+1)}( r(\bs{T}^{(l)})+ \bs{H}_{\mu^{(l)}}\delta^{(l)}))}\bar{\bs{X}},
				\end{align*}
				where
$\bs{K}^{(l+1)}=\bs{J}_{\mu^{(l)}} \bs{F}_{\mu^{(l)}}\bs{H}_{\mu^{(l)}}^\top \Sigma^{-1}$.
				We also derive the gain $\bs{K}$ in the following form using the matrix inversion lemma:
				\begin{equation}
\bs{K}^{(l+1)}=\bs{P}\bs{J}^{-1\top}_{\mu^{(l)}}\bs{H}_{\mu^{(l)}}^\top \left(\bs{H}_{\mu^{(l)}}\bs{J}^{-1}_{\mu^{(l)}}\bs{P}\bs{J}^{-1\top}_{\mu^{(l)}}\bs{H}_{\mu^{(l)}}^\top +\Sigma\right)^{-1}. 				
      \end{equation}
			}
			{\section{Derivation of the Jacobians for the camera measurement and the LiDAR measurement}~\label{apx:jacobians_camera_LIDAR}
				
				The Jacobian matrix $\bs{H}_{\hat{\bs{X}},C}$ of one feature in camera~\eqref{eqn:camera_model} is:
				\begin{align}\label{eqn:camera jacobian matrix}
					\bs{H}_{\hat{\bs{X}},C}&=\frac{\partial h_p}{\partial ^C\bs{p}_f}\begin{bmatrix}
						\frac{\partial h_t}{\partial \hat{\bs{\theta}}} &  \frac{\partial h_t}{\partial \hat{\bs{p}}}&\bs{0}_{2 \times 3}
					\end{bmatrix},\nonumber\\
					\frac{\partial h_p}{\partial ^C\bs{p}_f}&=\frac{1}{^{C}{p}_{z}}\left[\begin{array}{ccc}
						f_x & 0 & -f_x\frac{^{C}p_{x}}{^{C}p_{z}},\nonumber \\
						0 & f_y & -f_y \frac{^{C}p_{y}}{^{C}p_{z}}
					\end{array}\right],\\
					\frac{\partial h_t}{\partial \hat{\bs{\theta}}} &= ^I{\bs{R}}^\top_C   \hat{\bs{R}}^{\top}  ({^{G}\bs{p}_{f}})^{\wedge},\quad \frac{\partial h_t}{\partial \hat{\bs{p}}}=-^I{\bs{R}}^\top_C   \hat{\bs{R}}^{\top}.
				\end{align}
				
				The LiDAR measurement model writes
				\begin{align*}
					\bs{z}_{L}=h_{L}(^G\bs{p}_f)+\bs{n}_{L}
					=u^{\top} \left( h_{t,L}(^L\bs{p}_f) -\bs{q}_j     \right)+\bs{n}_{L}
				\end{align*}
				where $h_{t,L}$ is the mapping that transforms the position of a measured point in $\left\lbrace L \right\rbrace $ into $\left\lbrace W \right\rbrace $. 
				The Jacobian of $\bs{H}_{\hat{\bs{X}},L}$ can be rewritten as:
				\begin{align}\label{eqn:LiDAR measurement equation}
					\bs{H}_{\hat{\bs{X}},L}&=u^{\top}\begin{bmatrix}
						\frac{\partial h_{t,L}}{\partial \hat{\bs{\theta}}}& \frac{\partial h_{t,L}}{\partial \hat{\bs{p}}}&\bs{0}_{3\times3}
					\end{bmatrix},\nonumber\\
					\frac{\partial h_{t,L}}{\partial \hat{\bs{\theta}}}&= -\left(\hat{\bs{R}} ^I\bs{R}_{L} ^L \bs{p}_{f} \right)^{\wedge}-  \left(\hat{\bs{R}} ^I \bs{p}_{L}\right)^{\wedge} -\bs{p}^{\wedge},\nonumber\\
					\frac{\partial h_{t,L}}{\partial \hat{\bs{p}}}&= \bs{I}_3. 
				\end{align}
			}
			{\section{$\sqrt{n}$-Consistency analysis of the solutions in camera and LiDAR}~\label{apx:consistent_solution_analysis}
   \subsection{Consistency analysis}
				Before presenting the detailed explanation of $\sqrt{n}$-consistency, we give the following lemma.
 				\begin{lemma}[{\cite[Lemma 4]{zeng2022global}}]
					\label{lemma_noise_aver} Let $\left\lbrace {X}_k \right\rbrace $ be a sequence of independent random variables with $\mathbb{E}\left[{X}_k\right]=0$ and $\mathbb{E}\left[{X}_k^2\right]\leq\phi<\infty$ for any $k$. Then, there is $\sum_{k=1}^{n}{X}_k/\sqrt{n}=O_p(1)$.
				\end{lemma} 
				
				For the camera case, note that $\bs{A}_{C}$ is defined in~\eqref{eqn:linear equation camera}. Here we define the noise-free counterpart of $\bs{A}_{C}$ as
				$$\bs{A}_{C,o}=\begin{bmatrix}
					-u_{1,o}(\bs{p}_{f_1}-\bar{\bs{p}}_f)^\top &f_x\bs{p}^\top_{f_1}&f_x&\bs{0}_{1\times4}\\
					-v_{1,o}(\bs{p}_{f_1}-\bar{\bs{p}}_f)^\top &\bs{0}_{1\times4}&f_y\bs{p}^\top_{f_1}&f_y\\
					\vdots\\
					-u_{n,o}(\bs{p}_{f_n}-\bar{\bs{p}}_f)^\top &f_x\bs{p}^\top_{f_n}&f_x&\bs{0}_{1\times4}\\
					-v_{n,o}(\bs{p}_{f_n}-\bar{\bs{p}}_f)^\top &\bs{0}_{1\times4}&f_y\bs{p}^\top_{f_n}&f_y
				\end{bmatrix}, $$
				where $u_{i,o},v_{i,o}$ are the projections of features in the image without noise. It has been proven in~\cite{zeng2022cpnp} that the LS solution $(\bs{A}_{C,o}^\top \bs{A}_{C,o})^{-1}\bs{A}_{C,o}^\top \bs{b}$ is $\sqrt{n}$-consistent. However, $\bs{A}_{C,o}$ is the noise-free counterpart and is not available in practice.
				Therefore, the main idea is to analyze the gap between $\bs{A}_{C}^\top\bs{A}_{C}/n$ and $\bs{A}_{C,o}^\top\bs{A}_{C,o}/n$ and between $\bs{A}_{C}^\top\bs{b}/n$ and $\bs{A}_{C,o}^\top\bs{b}/n$. By eliminating the gaps, we can obtain a $\sqrt{n}$-consistent estimate.
				
				To eliminate these gaps, let $\Delta_{\bs{A}_C}\triangleq\bs{A}_{C}-\bs{A}_{C,o}$. From Lemma~\ref{lemma_noise_aver}, we have 
				\begin{align*}
					\frac{\bs{A}^\top_C\bs{A}_C}{n}&=\frac{\bs{A}^\top_{C,o}\bs{A}_{C,o}}{n}+\frac{\bs{A}^\top_{C,o}\Delta_{\bs{A}_C}+\Delta_{\bs{A}_C}^\top\bs{A}_{C,o}+\Delta_{\bs{A}_C}^\top\Delta_{\bs{A}_C}}{n}\\
					&=\frac{\bs{A}^\top_{C,o}\bs{A}_{C,o}}{n}+\sigma^2_{C}\frac{\bs{G}_C^\top\bs{G}_C}{n}+O_p(\frac{1}{\sqrt{n}}),
				\end{align*}
				and 
				\begin{align*}
					\frac{\bs{A}_C^\top\bs{b}}{n}=\frac{\bs{A}_{C,o}^\top\bs{b}}{n}+\sigma^2_C\frac{\bs{G}_C^\top}{n}\bs{1}_{2n\times1}+O_p(\frac{1}{\sqrt{n}}),
				\end{align*}
				where $\bs{G}_C$ is defined in~\eqref{eqn:coefficient_of_noise}. Therefore, we obtain
				\begin{align*}
					\vec{\bs{x}}_C & =(\bs{A}^\top_C\bs{A}_C-\hat{\sigma}^2_{C}\bs{G}^\top_C\bs{G}_C)^{-1}(\bs{A}^\top_C\bs{b}-\hat{\sigma}^2_{C}\bs{G}^\top_C\bs{1}_{2n\times1}) \\
					& = \left(\frac{\bs{A}_{C,o}^\top \bs{A}_{C,o}}{n}+O_p(\frac{1}{\sqrt{n}})\right)^{-1} \left( \frac{\bs{A}_{C,o}^\top \bs{b}}{n}+O_p(\frac{1}{\sqrt{n}})\right) \\
					& = (\bs{A}_{C,o}^\top \bs{A}_{C,o})^{-1}\bs{A}_{C,o}^\top \bs{b} + O_p(\frac{1}{\sqrt{n}}) \\
					& = \vec{\bs{x}}_{C,o}+ O_p(\frac{1}{\sqrt{n}}),
				\end{align*}
				where $\vec{\bs{x}}_{C,o}$ is the true value of $\vec{\bs{x}}_{C}$.
				
				Now we consider the LiDAR case. First, we write the noise-free counterpart of ${\bm A}_{L}$ as
				\begin{align*}
					{\bm A}_{L,o}=\begin{bmatrix}
						\bs{z}_{L_1}^{o \top} \otimes \bm {u}_1^\top ~~ \bm {u}_1^\top \\
						~~~~~\vdots ~~~~~~~~ \vdots \\
						\bs{z}_{L_n}^{o \top} \otimes \bm {u}_n^\top ~~ \bm {u}_n^\top
					\end{bmatrix}.
				\end{align*}
				Similar to the camera case, we are standing on the point to eliminate the bias between $({\bm A}_{L}^\top {\bm A}_{L})^{-1} {\bm A}_{L}^\top {\bm b}$ and $({\bm A}_{L,o}^\top {\bm A}_{L,o})^{-1} {\bm A}_{L,o}^{\top} {\bm b}$. Recall that $\bar {\bm Q}$ is defined in~\eqref{Q_bar}.
				From Lemma~\ref{lemma_noise_aver}, we have
				\begin{align*}
					\frac{{\bm A}_{L}^\top {\bm A}_{L}}{n} & = \frac{{\bm A}_{L,o}^{ \top} {\bm A}_{L,o}}{n} +  \sigma_L^2 \bar {\bm Q} + O_p (\frac{1}{\sqrt{n}} ) \\
					\frac{{\bm A}_{L}^\top {\bm b}}{n} & = \frac{{\bm A}_{L,o}^{ \top} {\bm b}}{n} + O_p (\frac{1}{\sqrt{n}} ).
				\end{align*}
				Therefore, we obtain
				\begin{align*}
					{\bs{x}}_{L} & =\left(\frac{{\bm A}_{L}^\top {\bm A}_{L}}{n}- \hat \sigma_{L}^2 \bar {\bm Q}\right)^{-1} \frac{{\bm A}_{L}^\top {\bm b}}{n} \\
					& = \left(\frac{{\bm A}_{L,o}^\top {\bm A}_{L,o}}{n}+O_p (\frac{1}{\sqrt{n}} )\right)^{-1} \left( \frac{{\bm A}_{L,o}^\top {\bm b}}{n} +O_p (\frac{1}{\sqrt{n}} )\right) \\
					& = (\bs{A}_{L,o}^\top \bs{A}_{L,o})^{-1}\bs{A}_{L,o}^\top \bs{b} + O_p(\frac{1}{\sqrt{n}}) \\
					& = {\bs{x}}_{L,o}+ O_p(\frac{1}{\sqrt{n}}), 
				\end{align*}
				where ${\bs{x}}_{L,o}$ is the true value of ${\bs{x}}_{L}$.

		\subsection{Proof of Theorem \ref{thm:n-consistent solution camera}	}
         Denote $\bs{T}_{\cons}\in SE(3)$ be the matrix form of a $\sqrt{n}$ consistent pose, projected from the vector form $\hat{\bs{x}}$. We can obtain by the following optimization problem:
					$$
					\bs{T}_{\cons} = \mathop{\arg\min}_{\bs{T}_{\cons}\in S E(3)}\left\| \operatorname{vec}(\begin{bmatrix}
		    \bs{R}^\top&-\bs{R}^\top\bs{p}
		\end{bmatrix})-\hat{\bs{x}}\right\|. 
					$$
					
					For any $n$ and the true state $\bs{x}_o\in\mathbb{R}^{12}$, we can verify that 
					\begin{align*}
						&\left\| \operatorname{vec}(\begin{bmatrix}
		    \bs{R}_{\cons}^\top&-\bs{R}_{\cons}^\top\bs{p}_{\cons}
		\end{bmatrix})-\hat{\bs{x}}\right\|\leq \|\bs{x}_o-\hat{\bs{x}}\|+\\&\left\| \operatorname{vec}(\begin{bmatrix}
		    \bs{R}_{\cons}^\top&-\bs{R}_{\cons}^\top\bs{p}_{\cons}
		\end{bmatrix})-{\bs{x}_o}\right\|\leq 2\|\bs{x}_o-\hat{\bs{x}}\|
					\end{align*}
					where the first inequality holds because $\bs{T}_{\cons}$ is the optimum of the above optimization problem. As $\hat{\bs{x}}$ converges to 
					$\bs{x}_o$ at the rate of 
					$O_p(1/\sqrt{n})$, the proof is complete.
				
				%	\subsection{Estimation of the covariance of the LiDAR}~\label{apx:estimation_of_covariance_of_the^{(l)}idar}
				%The proof is mainly based on the following lemma: 
				%\begin{lemma}[{\cite[Lemma 6]{zeng2023consistent}}] \label{lemma^{(l)}argest_eig}
				%	Let ${\bf Q}$ and ${\bf S}^o$ be two real symmetric matrices and ${\bf S}={\bf S}^o+{\bf Q}$. If ${\bf S}$ is positive-definite and ${\bf S}^o$ is positive-semidefinite, then $\lambda_{\rm max}( {\bf S}^{-1} {\bf Q})=1$.
				%\end{lemma}
				%On the one hand, given that ${\bm u}_j$'s are not coplanar, the matrix ${\bm S}=\bar {\bm A}^{(l)}^\top \bar {\bm A}^{(l)}/n$ are positive-definite with probability one. On the other hand, let $\bar {\bm A}^{(l)}^{o}=[{\bm A}^{(l)}^o~{\bm b}]$ be the noise-free counterpart of $\bar {\bm A}^{(l)}$ and ${\bm S}^o={\bm A}^{(l)}^{o \top} {\bm A}^{(l)}^o/n$. From~\eqref{noisefree^{(l)}idar_matrix_eqn} we know that ${\bm S}^o$ is positive-semidefinite. Moreover, it can be verified that
				%\begin{equation*}
				%	{\bm S} = {\bm S}^o+\sigma^2{\bm Q}+O_p \left(\frac{1}{\sqrt{n}} \right).
				%\end{equation*}
				%By Lemma~\ref{lemma^{(l)}argest_eig}, we have
				%$\lambda_{\rm max}\left( {\bm S}^{-1} \left( \sigma^2{\bm Q}+O_p \left(\frac{1}{\sqrt{n}} \right)\right) \right)=1$, i.e., $\hat \sigma^2=1/\lambda_{\rm max}({\bm S}^{-1}{\bm Q})=\sigma^2+O_p \left(\frac{1}{\sqrt{n}} \right)$.
				
			}
			{\section{Analysis of asymptotic property}\label{apx:theoretical analysis of asymptotic property}
\subsection{Proof of Theorem~\ref{thm:property_of_MLE}}\label{apx:property_of_MLE}
Given any $\bs{Y}$ and such $\tilde{\delta}\in\mathbb{R}^6$ that $\bs{T} =\exp{(\tilde{\delta})}\bs{Y}$, we have the following MLE 
		\begin{align*}
			\tilde{\delta}_{\MLE}=\mathop{\arg\min}_{\tilde{\delta}}\left\|  r_{\bs{Y}}(\tilde{\delta})\right\|_{\Sigma}^2,\quad r_{\bs{Y}}(\tilde{\delta}) \triangleq r(\exp{(\tilde{\delta})}\bs{Y}).
		\end{align*}
		It is well-known fact in \cite{Jennrich1969AsymptoticPO} that under certain regularity conditions, for the true parameter $\tilde{\delta}\in\mathbb{R}^6$ and optimal estimator $\tilde{\delta}_{\MLE}\in\mathbb{R}^6$ w.r.t. the quantity of samples $n$, it holds that $$\epsilon\triangleq \tilde{\delta}_{\MLE}-\tilde{\delta},\quad \epsilon\overset{d}{\longrightarrow}\mathcal{N}(0,\bs{F}_{\tilde{\delta}}^{-1}).$$ 
  The estimated fisher information matrix denoted by $\bs{F}_{\MLE,\tilde{\delta}}$ can be obtained by $$
		\bs{F}_{\MLE,\tilde{\delta}} = \frac{\partial r_{\bs{Y}}(\tilde{\delta})}{\partial \tilde{\delta}}^\top\Sigma^{-1}\frac{\partial r_{\bs{Y}}(\tilde{\delta})}{\partial \tilde{\delta}}.
		$$

		Now, let $\bs{Y}$ be the true parameter, i.e. $\bs{Y}=\bs{T}$ and we have $\tilde{\delta}=0$. Therefore, we can conclude that $\tilde{\delta}_{\MLE}\overset{d}{\longrightarrow}\mathcal{N}(0,\bs{F}_{\MLE}^{-1})$, which completes the proof. 
   
            \subsection{Proof of Theorem~\ref{thm:the_optimum_of_MAP}}\label{prof:proof_of_MAPwithMLE}
The work \cite{Gribonval2011ShouldPL} theoretically has proven that an estimator of MAP is also an MMSE estimator with the assumption of measurement being the white Gaussian and the pdf of a priori being the exponential form. We now try to solve a MAP.
 Using the LGN reparameterization method, we get an equivalent form of MAP~\eqref{eqn:linear MLE}:
	\begin{align}\label{eqn:LLS_of_MAP_using_MLE}	\hat{\delta}&=\mathop{\arg\min}_{\delta}\left\|\delta\right\|^2_{\bar{\bs{P}}}+\left\|\tilde{\delta}_{\MLE}-\dexp_{\tilde{\delta}_{\MLE}} \tilde{\bs{I}}\delta\right\|^2_{\bs{F}_{\MLE}^{-1}},
	\end{align}
where $\delta\triangleq\log({\bs{X}}\bar{\bs{X}}^{-1})$ and the resulting equality $\log(\bs{T}_{\MLE}\bs{T}^{-1})=\tilde{\delta}_{\MLE}-\dexp_{\tilde{\delta}_{\MLE}} \tilde{\bs{I}}\delta$ follows (i) of Lemma~\ref{lemma: compound of two matrix exponentials} as $\log(\bs{T}_{\MLE}\bs{T}^{-1})$ is small.

By the GN method for the above optimization problem, we obtain
\begin{align*}
    \hat{\delta} &=(\bar{\bs{P}}^{-1}+\tilde{\bs{I}}^\top\dexp^\top_{\tilde{\delta}_{\MLE}}\bs{F}_{\MLE}\dexp_{\tilde{\delta}_{\MLE}}\tilde{\bs{I}} )^{-1}\cdot\\
&\tilde{\bs{I}}^\top\dexp^\top_{\tilde{\delta}_{\MLE}}\bs{F}_{\MLE}\tilde{\delta}_{\MLE},
\end{align*}
 and the estimates can be $\hat{\bs{X}}_{\op}=\exp{(\hat{\delta})}\bar{\bs{X}}$.

We are now in a position to consider the posterior distribution and compute the covariance. Let $\hat{\xi}\triangleq\log{(\hat{\bs{X}}_{\op}\bs{X}^{-1})}$. The posterior distribution can be written into the following:
\begin{align}	
&p(\bs{X}_{k+1} |\mathcal{T}_{k+1})\propto p(\bs{X}_{k+1}|\mathcal{T}_{k})p(\bs{T}_{\MLE}|\bs{X}_{k+1})\nonumber\\
&\propto \exp{(\frac{1}{2}(\left\|\bs{J}_{\hat{\bs{X}}_{\op}}\hat{\xi}+\hat{\delta}\right\|^2_{\bar{\bs{P}}}+\left\|\log{(\bs{T}_{\MLE}\hat{\bs{T}}^{-1})}+\tilde{\bs{I}}\hat{\xi}\right\|^2_{\bs{F}_{\MLE}^{-1}}))}\label{eqn:concise_form_of_distribution},
\end{align}
which is the form of the multivariate Gaussian distributions with the random variables being $\hat{\xi}$. Therefore, based on the definition of $\mathcal{N}_{RG}$, we consider $\bs{X}$ to be $\mathcal{N}_{RG}$.  The covariance can be derived by arranging the coefficients of the quadratic terms, i.e., $\mathbb{E}[\hat{\xi}\hat{\xi}^\top]=(\bs{J}^{-1\top}_{\hat{\bs{X}}_{\op}}\bar{\bs{P}}^{-1}\bs{J}^{-1}_{\hat{\bs{X}}_{\op}}+\tilde{\bs{I}}^\top\bs{F}_{\MLE}\tilde{\bs{I}})^{-1}$.
That completes the proof.

            \subsection{Supporting results for the proof of Theorem~\ref{thm:equivalent relation}}\label{prof:proposition_of_mf}
  Before proving the main result, we employ the LGN method with a $\sqrt{n}$ consistent pose as initialization to approximate the ML estimate. Let a single LGN iteration with the $\sqrt{n}$-consistent pose be denoted as $\bs{T}_{\GN}$, which writes:
 \begin{align}
          \bs{T}_{\GN} &= \exp{(\Delta)}\bs{T}_{\cons},\quad \bs{H}_{\bs{T}_{\cons}}\triangleq -\frac{\partial r(\bs{T})}{\partial \bs{T}}|_{\bs{T}_{\cons}}, \nonumber\\ \Delta&\triangleq (\bs{H}_{\bs{T}_{\cons}}^{\top}\Sigma^{-1}\bs{H}_{\bs{T}_{\cons}})^{-1}\bs{H}_{\bs{T}_{\cons}}^{\top}\Sigma^{-1}r(\bs{T}_{\cons}).
 \end{align}
Particularly, when examining the asymptotic property, we can observe that $\Delta$ is sufficiently small to fulfill Lemma~\ref{lemma: compound of two matrix exponentials}.
 
In the subsequent discussion, we will elaborate on why we use $\bs{T}_{\GN}$ as a substitute for $\bs{T}_{\MLE}$.
 
	\begin{lemma}[Efficiency of Refinement({\cite[Ch.6.4, Corollary 4.4]{Lehmann1950TheoryOP}})]\label{lemma:refinement_of_the_GN}
		Under certain regularity conditions, we have:
		$$
	\bs{T}_{\GN}-\bs{T}_{\MLE}=o_p(\frac{1}{\sqrt{n}}).
		$$
	\end{lemma}

	 When the value of $n$ is large, $\bs{T}_{\GN}$ is a reasonable substitute for the ML estimate. 
 % The procedure can be concluded as a two-step approach: \begin{enumerate}
	% 	\item determine a $\sqrt{n}$-consistent estimate,
	% 	\item and implement the one single LGN iteration with the $\sqrt{n}$-consistent estimate as the initial value.
	% \end{enumerate} 
 % In addition, 
 The estimation of the covariance of $\bs{T}_{\GN}$ is done using the Fisher information, denoted as $\bs{F}_{\GN}$. The Fisher information is calculated as $\bs{F}_{\GN}= \mathbb E\left[\frac{\partial r(\bs{T}_{\GN})}{\partial \bs{T}_{\GN}}^\top\Sigma^{-1}\frac{\partial r(\bs{T}_{\GN})}{\partial \bs{T}_{\GN}}\right]$. 
 
When the estimator $\bs{T}_{\GN}$ replaces $\bs{T}_{\MLE}$ in Theorem~\ref{thm:the_optimum_of_MAP}, the resulting fused estimator is denoted as $\hat{\bs{X}}_{\Gf}$, which writes
 \begin{align}\label{eqn:Gf_form}	
 \hat{\bs{X}}_{\Gf}=\exp{(\delta_{\Gf})}\bar{\bs{X}}
		\end{align}
  where $\delta_{\Gf}\triangleq\bs{M}_{\Gf}\tilde{\bs{I}}^\top\tilde{\bs{J}}^{-1 \top}_{\delta_{\GN}} \bs{F}_{\GN}{\tilde{\delta}_{\GN}}$,
			$\bs{M}_{\Gf} \triangleq(\bar{\bs{P}}^{-1}+\tilde{\bs{I}}^\top \tilde{\bs{J}}^{-1 \top}_{\delta_{\GN}} \bs{F}_{\GN} \tilde{\bs{J}}^{-1}_{\delta_{\GN}}\tilde{\bs{I}})^{-1}$, and $\tilde{\delta}_{\GN}\triangleq\log{(\bs{T}_{\GN}\bar{\bs{T}}^{-1})}$.

We state that $\hat{\bs{X}}_{\Gf}$ also demonstrates the same asymptotic properties as $\hat{\bs{X}}_{\op}$.
 
	\begin{proposition}[Asymptotic property of $\hat{\bs{X}}_{\Gf}$]\label{prop:Asymptotic property of Gf}
	   Under certain regularity conditions, we have:\begin{equation}
\hat{\bs{X}}_{\Gf}-\hat{\bs{X}}_{\op}=o_p(\frac{1}{\sqrt{n}}).
		\end{equation}
	\end{proposition}
\begin{proof}
We briefly talk about the proof because a more rigorous mathematical proof involves too many details. The solution $\hat{\bs{X}}_{\op}$ can be seen as the continuous mapping of $\bs{T}_{\MLE}$. When we enter $\bs{T}_{\GN}$ into the continuous mapping, the output is $\hat{\bs{X}}_{\Gf}$. Hence, we use (v) in Lemma~\ref{lemma:calculus_pf_op}, we find that $\hat{\bs{X}}_{\Gf}-\hat{\bs{X}}_{\op}=o_p({1}/{\sqrt{n}}).
$	
\end{proof}

In the end, we present the asymptotic properties of some matrix formulas, which can be utilized in the proof of Theorem~\ref{thm:equivalent relation}.
% \begin{lemma}~\label{lemma:r_H_consistent}
% 					Consider the residuals $r\triangleq\begin{bmatrix}
% 						r^\top_1, \cdots, r^\top_i,\cdots, r^\top_n
% 					\end{bmatrix}^\top$ with the variables $\bs{T}$. Let $\bs{T}_{\cons}$ denote the $\sqrt{n}$-consistent solution to $\bs{T}$. Let $\bs{H}_{\bs{T}_{\cons}}$ and $\bs{H}_{\bs{T}}$ denote the Jacobians of $r$ at the value of $\bs{T}_{\cons}$ and $\bs{T}$. We have the following relationship:
% 					\begin{align*}
% 						\bs{H}_{\bs{T}_{\cons}}^\top\Sigma^{-1}\bs{H}_{\bs{T}_{\cons}}=\bs{H}_{\bs{T}}^\top\Sigma^{-1}\bs{H}_{\bs{T}}+nO_p(\frac{1}{\sqrt{n}}).
% 					\end{align*}
% 				\end{lemma}
% 				\begin{proof}
% 					Let $\bs{H}_{i,\bs{T}_{\cons}}$ and $\bs{H}_{i,\bs{T}}$ denote the Jacobians of $r_i$ at the value of $\bs{T}_{\cons}$ and $\bs{T}$. We have 
% 					\begin{align*}
% 						&\frac{1}{\sigma^2}\bs{H}_{\bs{T}_{\cons}}^\top\bs{H}_{\bs{T}_{\cons}} = \sum^{n}_{i=1} \frac{1}{\sigma^2}\bs{H}_{i,\bs{T}_{\cons}}^\top\bs{H}_{i,\bs{T}_{\cons}}\\
% 						&\overset{(c)}{=}\sum^{n}_{i=1} \frac{1}{\sigma^2}(\bs{H}_{i,\bs{T}}^\top\bs{H}_{i,\bs{T}}+O_p(\frac{1}{\sqrt{n}}))=\bs{H}_{\bs{T}}^\top\Sigma^{-1}\bs{H}_{\bs{T}}+nO_p(\frac{1}{\sqrt{n}}). 
% 					\end{align*}
% 					Equality $(c)$ holds because of $(iv)$ in Lemma~\ref{lemma:calculus_pf_op}.
% 				\end{proof}
				\begin{lemma}~\label{lemma:approximation_X_a}
					If $\bs{T}_{\GN}=\bs{T}_{\cons}+O_p({1}/{\sqrt{n}})$, the following two formulas
					\begin{align*}
		\bs{M}_{\cons}\triangleq&\left( \bar{\bs{P}}^{-1}+\bs{J}_{\bs{X}_{\cons}}^{\top-1}\bs{H}^\top_{\bs{X}_{\cons}}\Sigma^{-1}\bs{H}_{\bs{X}_{\cons}}\bs{J}^{-1}_{\bs{X}_{\cons}} \right)^{-1},\\
						\bs{M}_{\GN}\triangleq&\left( \bar{\bs{P}}^{-1}+\bs{J}_{\bs{X}_{\GN}}^{\top-1}\bs{H}^\top_{\bs{X}_{\GN}}\Sigma^{-1}\bs{H}_{\bs{X}_{\GN}}\bs{J}^{-1}_{\bs{X}_{\GN}} \right)^{-1},
					\end{align*}
					have $\bs{I}-\bs{M}_{\GN}\bar{\bs{P}}^{-1}=\bs{I}-\bs{M}_{\cons}\bar{\bs{P}}^{-1}+o_p(\frac{1}{\sqrt{n}})$. Recall that $\bs{J}_{\cdot}$ and $\bs{H}_{\cdot}$ are defined in Section~\ref{sec:Iterative gaussian newton}.
				\end{lemma}
				\begin{proof}
                    We first prove that $\sqrt{n}(\bs{M}_{\GN}-\bs{M}_{\cons}) $ converges to $0$ in probability. We analyze $\sqrt{n}\bs{M}_{\cons}$ writes:
                    \begin{align*}
                        &\sqrt{n}\bs{M}_{\cons} =\sqrt{n}\left( \bar{\bs{P}}^{-1}+\bs{J}_{\bs{X}_{\cons}}^{\top-1}\bs{H}^\top_{\bs{X}_{\cons}}\Sigma^{-1}\bs{H}_{\bs{X}_{\cons}}\bs{J}^{-1}_{\bs{X}_{\cons}} \right)^{-1}\\
                        &=\sqrt{n}\left( \bar{\bs{P}}^{-1}+\bs{J}_{\bs{X}_{\GN}}^{\top-1}\bs{H}^\top_{\bs{X}_{\GN}}\Sigma^{-1}\bs{H}_{\bs{X}_{\GN}}\bs{J}^{-1}_{\bs{X}_{\GN}}+O_p(\frac{1}{\sqrt{n}}) \right)^{-1}.
                    \end{align*}
                    The last equality is led by the fact $\bs{T}_{\GN}=\bs{T}_{\cons}+O_p(\frac{1}{\sqrt{n}})$ and (iv) in Lemma~\ref{lemma:calculus_pf_op}. We observe that ${\bs{H}^\top_{\bs{X}_{\GN}}\Sigma^{-1}\bs{H}_{\bs{X}_{\GN}}}/{\sqrt{n}}$ converges to infinity as ${\bs{H}^\top_{\bs{X}_{\GN}}\Sigma^{-1}\bs{H}_{\bs{X}_{\GN}}}/{n}$ converges to certain matrix. Hence, it is easy to conclude that $\sqrt{n}\bs{M}_{\cons}$ converges to $0$ in probability, and so is $\sqrt{n}(\bs{M}_{\GN}-\bs{M}_{\cons}) $. This means $\bs{M}_{\GN} = \bs{M}_{\cons}+o_p({1}/{\sqrt{n}})$.

                    Finally, using (v) in Lemma~\ref{lemma:calculus_pf_op}, we conclude that $\bs{M}_{\GN}\bar{\bs{P}}^{-1}=\bs{M}_{\cons}\bar{\bs{P}}^{-1}+o_p({1}/{\sqrt{n}})$ and $\bs{I}-\bs{M}_{\GN}\bar{\bs{P}}^{-1}=\bs{I}-\bs{M}_{\cons}\bar{\bs{P}}^{-1}+o_p({1}/{\sqrt{n}})$ .           
					% Let $\bar{I}\triangleq\bs{H}^\top_{\bs{X}}\Sigma^{-1}\bs{H}_{\bs{X}}$,$J_{\cons}\triangleq\bs{J}_{\bs{X}_{\cons}}^\top \bar{\bs{P}}^{-1}\bs{J}_{\bs{X}_{\cons}}$, and $J_{o}\triangleq\bs{J}_{\bs{X}_{o}}^\top \bar{\bs{P}}^{-1}\bs{J}_{\bs{X}_{o}}$ for convenience. Due to Lemma~\ref{lemma:r_H_consistent}, we can conclude that $\bs{H}^\top_{\bs{X}_{\cons}, }\Sigma^{-1}\bs{H}_{\bs{X}_{\cons}}$ is equal to $\bar{I}$ plus a term that is of order $O_p(1/\sqrt{n})$ as $n$ increases.
					% By using (iv) in Lemma~\ref{lemma:calculus_pf_op}, we can calculate that $J_{\cons}=J_{o}+O_p(1/\sqrt{n})$. Inserting this into $\bs{F}_{\cons}$ yields
					% \begin{align*}
					% 	\bs{F}_{\cons}&=\left(J_o+n\bar{I}+(n+1)O_p(\frac{1}{\sqrt{n}})\right)^{-1}(J_o+O_p(\frac{1}{\sqrt{n}}))\\
					% 	&=(J_o+n\bar{I})^{-1}(J_o+O_p(\frac{1}{\sqrt{n}}))\\
					% 	&- (J_o+n\bar{I})^{-1}(n+1)O_p(\frac{1}{\sqrt{n}})\bs{F}_{\cons}.
					% \end{align*}
					% Hence, the asympototical property of $\bs{F}_{\cons}$ and $(J_o+n\bar{I})^{-1}J_o$ writes
					% \begin{align*}
					% 	&\sqrt{n}(\bs{F}_{\cons}-(J_o+n\bar{I})^{-1}J_o)=(J_o+n\bar{I})^{-1}O_p(1)\\
					% 	&-(J_o+n\bar{I})^{-1}(n+1)O_p(1)\bs{F}_{\cons}.
					% \end{align*}
					% We observe that R.H.S of the equation converge to $0$ in probability, which implies that $\bs{F}_{\cons}=(J_o+n\bar{I})^{-1}J_o+o_p(1/\sqrt{n})$ according to the definition of $o_p$. In the similar deviation, we have $\bs{F}_{\GN}=(J_o+n\bar{I})^{-1}J_o+o_p(1/\sqrt{n})$, which completes the proof. 
				\end{proof}

\subsection{Proof of Theorem~\ref{thm:equivalent relation}}
    \begin{proof}
    We will analyze the asymptotic property of $\log{(\hat{\bs{X}}_{\EIKF}\bar{\bs{X}}^{-1})}$ and $\log{(\hat{\bs{X}}_{\Gf}\bar{\bs{X}}^{-1})}$, where $\hat{\bs{X}}_{\Gf}$ has been shown to have the same asymptotic property as $\hat{\bs{X}}_{\op}$ in Proposition~\ref{prop:Asymptotic property of Gf}.

    From \eqref{eqn: update form}, $\delta_{\EIKF}\triangleq\log{(\hat{\bs{X}}_{\EIKF}\bar{\bs{X}}^{-1})}$ can be simplified as follows:
    \begin{equation*}
\delta_{\EIKF}=\delta_{\cons}+\bs{M}_{\cons}\left(\bs{J}^{\top-1}_{\bs{X}_{\cons}}\bs{H}^{\top}_{\bs{X}_{\cons}}\Sigma^{-1}r(\bs{T}_{\cons})-\bar{\bs{P}}^{-1}\delta_{\cons}\right).
    \end{equation*}
    From \eqref{eqn:Gf_form} in Proposition~\ref{prop:Asymptotic property of Gf} and the fact $\delta_{\GN}=\bs{J}_{\bs{X}_{\cons}}\Delta+\delta_{\cons}$, we can simply $\delta_{\Gf}$ into
    \begin{align*}
        \delta_{\Gf}&=\delta_{\cons}-\bs{M}_{\GN}\bar{\bs{P}}^{-1}\delta_{\cons}+(\bs{I}-\bs{M}_{\GN}\bar{\bs{P}}^{-1})\bs{J}_{\bs{X}_{\cons}}\\
&\cdot\left(\bs{H}^\top_{\bs{X}_{\cons}}\Sigma^{-1}\bs{H}_{\bs{X}_{\cons}}\right)^{-1}\bs{H}^{\top}_{\bs{X}_{\cons}}\Sigma^{-1}r(\bs{T}_{\cons}).
    \end{align*}
    By Lemma~\ref{lemma:approximation_X_a}, we have $\delta_{\EIKF}=\delta_{\Gf}+o_p({1}/{\sqrt{n}}) $. The conclusion regarding covariance can also be derived in a similar manner using Lemma~\ref{lemma:approximation_X_a}. 
	\end{proof}
			}
			
			\bibliographystyle{unsrt}
			% argument is your BibTeX string definitions and bibliography database(s)
			\bibliography{ref}
			
		\end{document}